\documentclass{article}

\usepackage[preprint]{neurips_2026}
\usepackage{amsthm}
\usepackage{amsmath}
\usepackage{graphicx}
\usepackage{multirow}
\usepackage{float}
\usepackage{booktabs}
\usepackage{tabularx}
\usepackage{array}
\usepackage{placeins}
\newcommand{\condvar}{C}              % condition random variable
\newcommand{\condset}{\mathcal{C}}    % condition set
\newcommand{\ncond}{K}                % number of conditions
\newcommand{\condval}{c}              % a realization of C
\newcommand{\partmap}{\pi}           % label-to-condition partition map
\newcommand{\featenc}{f_{\mathrm{enc}}} % feature extractor / backbone map

\newcommand{\embedmat}{Q}             % simplex embedding matrix
\newcommand{\Id}{\mathbf{I}}          % identity matrix
\newcommand{\infoscale}{\gamma}       % scalar MI level in product-space theorem
\newcommand{\priorbias}{b_0}          % global shift for conditional prior

\newcommand{\cleanrule}{y^\star}      % deterministic clean-label rule
\newcommand{\nuis}{N}                 % nuisance variable
\newcommand{\aggq}{\bar q_\theta}     % aggregated posterior   

\newcommand{\noreg}{\textsc{No Reg}}
\newcommand{\stdkl}{\textsc{StdKL}}
\newcommand{\condkl}{\textsc{CondKL}}
\newcommand{\dualkl}{\textsc{DualKL}}

\newtheorem{theorem}{Theorem}
\newtheorem{lemma}{Lemma}
\newtheorem{proposition}{Proposition}
\newtheorem{corollary}{Corollary}

\theoremstyle{remark}

\usepackage[utf8]{inputenc} % allow utf-8 input
\usepackage[T1]{fontenc}    % use 8-bit T1 fonts
\usepackage{hyperref}       % hyperlinks
\usepackage{url}            % simple URL typesetting
\usepackage{booktabs}       % professional-quality tables
\usepackage{amsfonts}       % blackboard math symbols
\usepackage{nicefrac}       % compact symbols for 1/2, etc.
\usepackage{microtype}      % microtypography
\usepackage{xcolor}         % colors
\usepackage{float}          % [H] float placement in appendix figures
\usepackage{wrapfig}        % text wrapping around figures

\AtBeginEnvironment{thebibliography}{\small}

\title{Rethinking the Information Bottleneck: Structured Decomposition under Label-Induced Partitions}

\author{%
  Jingyao Zhang \quad
  Yuxuan Li \quad
  Lu Han \quad
  Ali Anaissi \quad
  Nguyen H. Tran \\
  School of Computer Science, The University of Sydney \\
  Sydney, NSW 2006, Australia
}

\begin{document}

\maketitle

\begin{abstract}
Standard information bottleneck (IB) regularization constrains representations via a single scalar $I(Z;X)$, implicitly treating all information as homogeneous. However, a single global compression control couples label-relevant structure with residual within-condition variation, rather than regulating their allocation independently, allowing nuisance information to persist in learned representations. For example, in medical imaging applications, residual variation often stems from acquisition conditions, background factors, or subject-specific appearance. This issue becomes particularly pronounced in data-limited settings, where models tend to overfit such variation, hindering generalization. While existing regularization methods can stabilize training, control capacity, or shape representation geometry, they do not explicitly separate nuisance-like variation from task-supporting structure. To address this limitation, we revisit IB from a structured perspective based on a label-induced partition, where condition-level structure and within-condition information play distinct roles. This leads to a dual-bottleneck formulation: a standard KL term controls global information capacity, while a conditional KL term targets within-condition information. We show that the conditional KL admits an exact decomposition into a within-condition information term and a prior-mismatch term, explaining its alignment with the design objective. With a simplex-structured conditional prior, the method provides controllable latent geometry and integrates seamlessly into existing pipelines. Experiments on classification and segmentation show the clearest gains in low-data classification and consistent improvements across dense prediction benchmarks.

\end{abstract}

\section{Introduction}
\label{sec:intro}

Learning latent representations with strong generalization is a central goal in supervised learning. Not all encoded information should be treated equally: some components support label-induced separability, while others merely reflect residual within-condition variation that is easy to memorize but unreliable for prediction. This is especially salient in medical imaging, where acquisition conditions, background factors, and subject-specific appearance contribute substantial residual variation. In this work, we formalize this distinction under a label-induced condition variable \(\condvar=\partmap(Y)\), where \(Y\) denotes the class label and \(\partmap\) maps labels to condition groups, and study how to preserve condition-level structure while more directly constraining within-condition information.

Among principled approaches, the Information Bottleneck (IB) principle and its variational formulation, the Variational Information Bottleneck (VIB), view regularization as balancing predictive relevance against compression \citep{tishby1999information,tishby2015deep,alemi2017deep}. In standard VIB, this is implemented by constraining \(I(Z;X)\). However, under the label-induced partition, the total information admits
\[
I(Z;X)=I(Z;C)+I(Z;X\mid C).
\]
This reveals a regularizer-level structural coupling: a single global compression control jointly penalizes condition-level structure and within-condition variation, leaving their allocation without an independent control axis. Encoders with the same global budget may therefore retain different amounts of residual information, a limitation formalized in Sec.~\ref{subsec:theory_single_axis}. This issue is especially relevant in medical imaging, where acquisition and subject-specific variation can be substantial, but it also applies more broadly when residual variation is easy to memorize yet weakly aligned with label-induced separability \citep{arpit2017closer}; see Sec.~\ref{subsec:residual_variation_concern}. We therefore introduce a dual-bottleneck objective, where a standard KL term controls global capacity and a conditional KL term targets within-condition information under \(\condvar=\partmap(Y)\).

To instantiate this realization in a structured way, we further introduce a class-conditional prior constrained by regular-simplex geometry. This design provides a symmetric and controllable geometric scaffold for the conditional prior, injects a separability-oriented inductive bias into the latent space, and reduces direct reliance on empirical prototype construction at initialization.

The main contributions of this paper are summarized as follows:
\begin{itemize}
    \item We present a structured information-control view under a label-induced partition, showing that global compression controls the aggregate information budget but not its allocation between condition-level and within-condition information.
    \item We propose a dual-bottleneck objective as a direct operational realization of this view, and derive an exact decomposition of the conditional KL term into a within-condition information component and a condition-wise prior-mismatch component---a closure property that makes the conditional KL a principled surrogate for constraining $I(Z;X\mid C)$.
    \item We establish a task-constrained two-axis upper-bound result: the dual regularizer simultaneously upper-bounds visible within-condition nuisance dependence (through $J_{\mathrm{cond}}$) and conditional-invisible condition-level encoding (through $J_{\mathrm{std}}$), neither of which is controlled by the other surrogate alone; the corresponding family-inclusion property follows as an immediate corollary.
    \item We introduce a simplex-constrained conditional prior with fixed shared variance for practical instantiation, and experiments on supervised classification and segmentation, including medical-image benchmarks, demonstrate consistent improvements over the unregularized baseline, with the clearest gains in data-limited regimes.
\end{itemize}

\section{Related work}
\label{sec:related_work}
\paragraph{Information bottleneck and variational formulations.}
The Information Bottleneck (IB) principle formalizes supervised representation
learning as a trade-off between predictive relevance and compression
\citep{tishby1999information}. This view was later extended to deep models
\citep{tishby2015deep} and made practical through variational objectives such as
the Variational Information Bottleneck (VIB) \citep{alemi2017deep}. However,
a single global bottleneck on $I(Z;X)$ can be too coarse for deep classification,
especially in deterministic or continuous-input regimes
\citep{amjad2020learning}. Related work also suggests that a single principle may
be insufficient in complex representation-learning settings
\citep{ahuja2021invariance}. Our work addresses a different limitation: under a label-induced partition, the global budget conflates condition-level structure with within-condition variation, motivating a dual bottleneck.

\paragraph{Conditional bottlenecks and conditional priors.}
Our conditional term is related to prior conditional information-control objectives, including conditional VIB-style formulations for multi-view representation learning \citep{hwang2021multiview}, conditional compression objectives for self-supervised visual representations \citep{lee2021compressive}, and recent conditional information-bottleneck formulations for sequence modeling and time-series imputation \citep{kalajdzievski2023transfer,choi2024conditional}. These works show that conditioning the bottleneck can be useful, but their conditions arise from views or augmentations rather than from a supervised label-induced partition. In contrast, our condition variable is \(\condvar=\partmap(Y)\), which is used to separate condition-level structure
$I(Z;C)$ from within-condition information $I(Z;X\mid C)$. Moreover, we retain
both $J_{\mathrm{std}}$ and $J_{\mathrm{cond}}$, so global capacity control is
not replaced by conditional control. Condition-dependent priors also appear in
CVAEs and semi-supervised generative models
\citep{sohn2015learning,kingma2014semi}, where they serve generative likelihoods
or ELBOs. Here, the conditional prior instead acts as a supervised
information-control prior whose KL term decomposes into $I(Z;X\mid C)$ plus
condition-wise prior mismatch, with simplex geometry providing a symmetric
instantiation.

\paragraph{Representation geometry and class structure.}
A complementary line of work improves supervised representations by shaping
their geometry. Center loss, prototype learning, supervised contrastive learning,
and MCR$^2$ promote intra-class compactness, inter-class separation, or
class-wise compressibility through empirical centers, prototypes, pairwise
objectives, or coding-rate criteria
\citep{wen2016discriminative,yang2018robust,khosla2020supervised,yu2020learning}. In contrast, we impose geometric structure directly at the conditional prior, and ground it in an information-theoretic decomposition that enables targeted control of within-condition information without empirical centers. %To our knowledge, this is the first formulation that links representation geometry to an explicit decomposition of mutual information, replacing surrogate geometric objectives with a principled information-level control mechanism.

\begin{table}[h]
\centering
\caption{Objective-level comparison. $J_{\mathrm{std}}$ and $J_{\mathrm{cond}}$
denote global and label-conditioned KL controls.}
\label{tab:objective_comparison}
\footnotesize
\setlength{\tabcolsep}{2.6pt}
\renewcommand{\arraystretch}{0.92}
\begin{tabularx}{\linewidth}{
@{}
>{\raggedright\arraybackslash}p{0.20\linewidth}
>{\centering\arraybackslash}p{0.12\linewidth}
>{\centering\arraybackslash}p{0.12\linewidth}
>{\raggedright\arraybackslash}p{0.22\linewidth}
>{\raggedright\arraybackslash}X
@{}
}
\toprule
Method & Global & Cond. & Geometry / prior & Main distinction \\
\midrule
VIB / \textsc{StdKL}
& $J_{\mathrm{std}}$
& --
& global prior
& only global bottleneck \\

\textsc{CondKL} (cond-only)
& --
& $J_{\mathrm{cond}}$
& conditional prior
& no label-induced dual control \\

Center / prototype
& --
& implicit
& empirical centers
& feature-space geometry \\

SupCon
& --
& implicit
& contrastive pairs
& pairwise geometry \\

Generic regularizers
& generic
& --
& none
& orthogonal heuristics \\

\textsc{DualKL} (ours)
& $J_{\mathrm{std}}$
& $J_{\mathrm{cond}}$
& simplex conditional prior
& joint two-axis control \\
\bottomrule
\end{tabularx}
\end{table}

\paragraph{Simplex geometry and neural collapse.}
Neural Collapse studies show that deep classifiers often approach highly
symmetric Simplex ETF configurations, and that such geometry can be optimal
under simplified feature models
\citep{papyan2020prevalence,zhu2021geometric}. Inspired by this line of work, we adopt simplex geometry as a structured prior template for class-conditioned regularization. In our setting, this choice provides a symmetric and analytically transparent starting geometry while reducing reliance on empirical center construction at initialization.

\section{Problem description, basic assumptions, and design goal}
\label{sec:problem_assumptions_design}

\subsection{Residual variation as a representation concern}
\label{subsec:residual_variation_concern}

In this work, we treat regularization not as generic capacity control alone, but as representation-level selection of what to preserve and what to suppress. Classical regularizers mainly control capacity or stabilize optimization \citep{krogh1991simple,srivastava2014dropout,prechelt1998early,cubuk2019autoaugment}, while information-bottleneck methods make this selection explicit through compression-oriented objectives \citep{tishby1999information,alemi2017deep}. Our specific concern is residual variation that is easy to memorize yet weakly aligned with label-induced separability \citep{arpit2017closer}.

\begin{wrapfigure}{r}{0.4\columnwidth}
    \vspace{-0.5em}
    \centering
    \includegraphics[width=\linewidth]{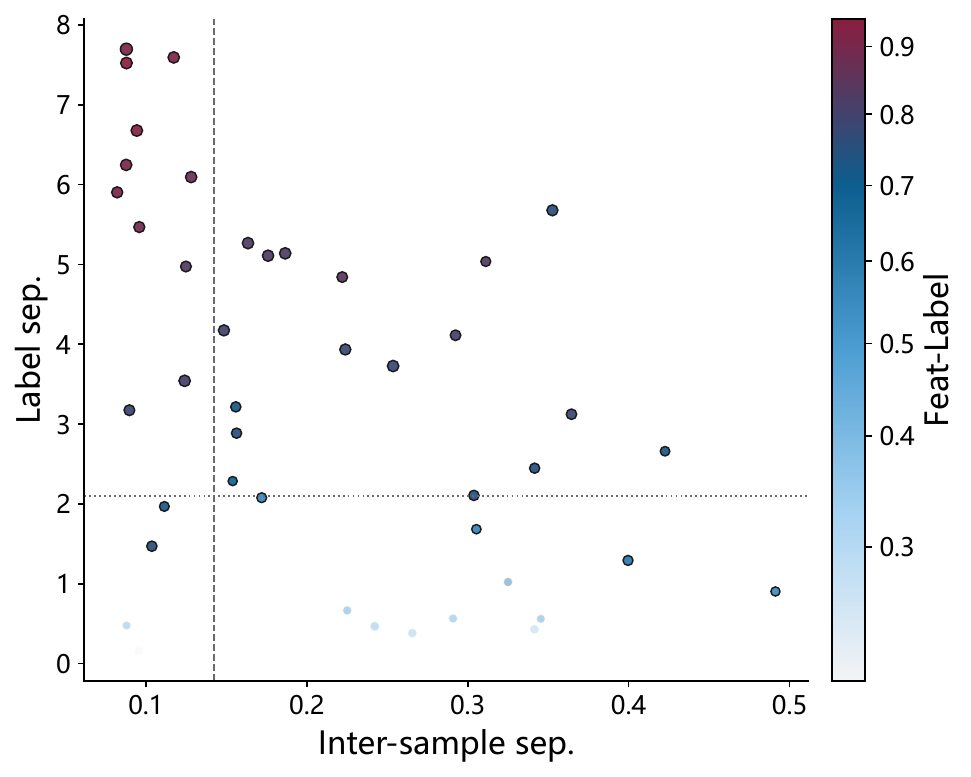}
    \vspace{-1.5em} 
    \caption{\textbf{BloodMNIST feature heterogeneity.} Points denote low-level statistics with sample separation on the x-axis and label separation on the y-axis. Lower-right points correspond to high variation but weak label relevance.}
    \label{fig:bloodmnist_feature_heterogeneity}
    \vspace{2pt}
\end{wrapfigure}

This concern is especially salient in medical-image settings, where acquisition conditions, background factors, and subject-specific appearance can vary substantially. At the same time, the underlying issue is not domain-specific: in supervised representation learning more broadly, residual within-condition variation can coexist with task-supporting structure and therefore requires explicit control.

This view is consistent with prior information-theoretic analysis. 
Under the assumptions that \(N\) is a nuisance for \(Y\) (i.e., \(I(N;Y)=0\)),
\(Z\) is sufficient for \(Y\) so that \(I(Z;Y)=I(X;Y)\), and
\(N\to X\to Z\), Achille and Soatto~\citep{achille2018emergence} derive
\[
I(Z;\nuis) \le I(Z;X)-I(X;Y).
\]
Equivalently, under sufficiency this can be read as \(I(Z;\nuis)\le I(Z;X)-I(Z;Y)\). We use this result only as motivation: information retained beyond what is needed for the task can carry nuisance dependence, so residual variation is a meaningful object of control.
Thus, retaining more input information than the task strictly requires does not guarantee that nuisance-related dependence vanishes. Fig.~\ref{fig:bloodmnist_feature_heterogeneity} provides dataset-level evidence on BloodMNIST: some low-level statistics vary strongly across samples but separate labels weakly. Although this analysis is not an information-plane estimate, it supports treating residual variation as a distinct object of control.

\subsection{Information categories under a label-induced partition}
\label{subsec:information_categories}

We now formalize the above concern in a standard supervised representation learning setting. Let \(X \in \mathcal{X}\) denote the observed input, \(Y \in \mathcal{Y}\) the supervised target, and \(Z \in \mathbb{R}^d\) the latent representation produced by a stochastic encoder \(q_\theta(z\mid x)\). We further introduce a finite condition variable induced by the supervised target, \(\condvar=\partmap(Y)\), where \(\partmap\) maps labels to condition groups. This condition variable is not intended to exhaust all task semantics; rather, it serves as an analytical coordinate for distinguishing different informational roles in the learned representation.

Under this coordinate system, we distinguish two categories of information. The first is condition-level structure \(I(Z;\condvar)\), namely the component aligned with the label-induced partition and supportive of separability across conditions. The second is within-condition information \(I(Z;X\mid \condvar)\), namely the dependence that remains between the representation and the input after conditioning on the partition. The former is the structure we aim to preserve relatively, while the latter is the component we aim to constrain more directly.
Accordingly, the design goal is to realize a more selective form of representation control:
\[
I(Z;X)=I(Z;\condvar)+I(Z;X\mid \condvar).
\]
This decomposition follows from the chain rule and the encoder-induced independence \(Z\perp \condvar\mid X\); see Appendix~\ref{app:notation}--\ref{app:thm1} (Theorem~\ref{thm:decomp}). We use it to separate condition-level from within-condition information.

\subsection{Assumptions and interpretation scope}
\label{subsec:assumptions_scope}

We conduct the analysis in a clean supervised learning setting. For interpretation, we assume that the observed input can be written as
\begin{equation}
X=T(S,\nuis), \qquad Y=f(S),
\end{equation}
where \(S\) denotes the semantic pattern supporting the supervised target, \(\nuis\) denotes additional perturbation or background variation, and \(T\) combines them into the observed input. This setup allows us to focus on input complexity and residual variation, rather than confounding effects from noisy or corrupted labels.

Let \(\condvar=\partmap(Y)\) denote the label-induced condition variable introduced above. We treat labels as reliable so that the decomposition and conditional modelling are well defined, and use \(\condvar\) as the coordinate for separating condition-level structure from within-condition information. While medical imaging provides a motivating regime with substantial residual variation, the formulation is domain-agnostic and applies broadly to supervised representation learning settings where residual variation coexists with task-relevant structure.

These assumptions are introduced for interpretability; the proposed method does not depend on them in practice. Minimally, \(I(Z;X\mid \condvar)\) is the residual dependence between \(Z\) and \(X\) after conditioning on \(\condvar\). A stronger nuisance-based interpretation requires additional conditional-independence structure and is introduced later as a complementary view rather than as a prerequisite for the method; see Sec.~\ref{subsec:nuisance} and Appendix~\ref{app:nuisance} (Proposition~\ref{prop:nuisance}).

\section{Theory: information-control view}
\label{sec:theory_control}

\subsection{Label-induced information coordinates}
\label{subsec:theory_decomposition}

We analyze representation control under \(\condvar=\partmap(Y)\), introduced in Sec.~\ref{subsec:information_categories}, as a coordinate for two informational roles. Since the posterior encoder is \(q_\theta(z\mid x)\), \(Z\perp \condvar\mid X\); Theorem~\ref{thm:decomp} gives
\[
I(Z;X)
=
I(Z;\condvar)
+
I(Z;X\mid \condvar).
\]
The first term captures condition-level structure aligned with the label-induced partition, while the second term captures residual dependence on the input after conditioning on that partition.
Thus, the central control question is not merely how much information \(Z\) retains from \(X\), but how the total information budget is allocated between condition-level structure and within-condition residual variation.

A formal proof of the decomposition is given in Appendix~\ref{app:thm1} (Theorem~\ref{thm:decomp}); see Fig.~\ref{fig:info_plane} for a schematic illustration.

\begin{figure}[H]
    \centering
    \begin{minipage}[t]{0.24\textwidth}
        \centering
        \includegraphics[width=\linewidth]{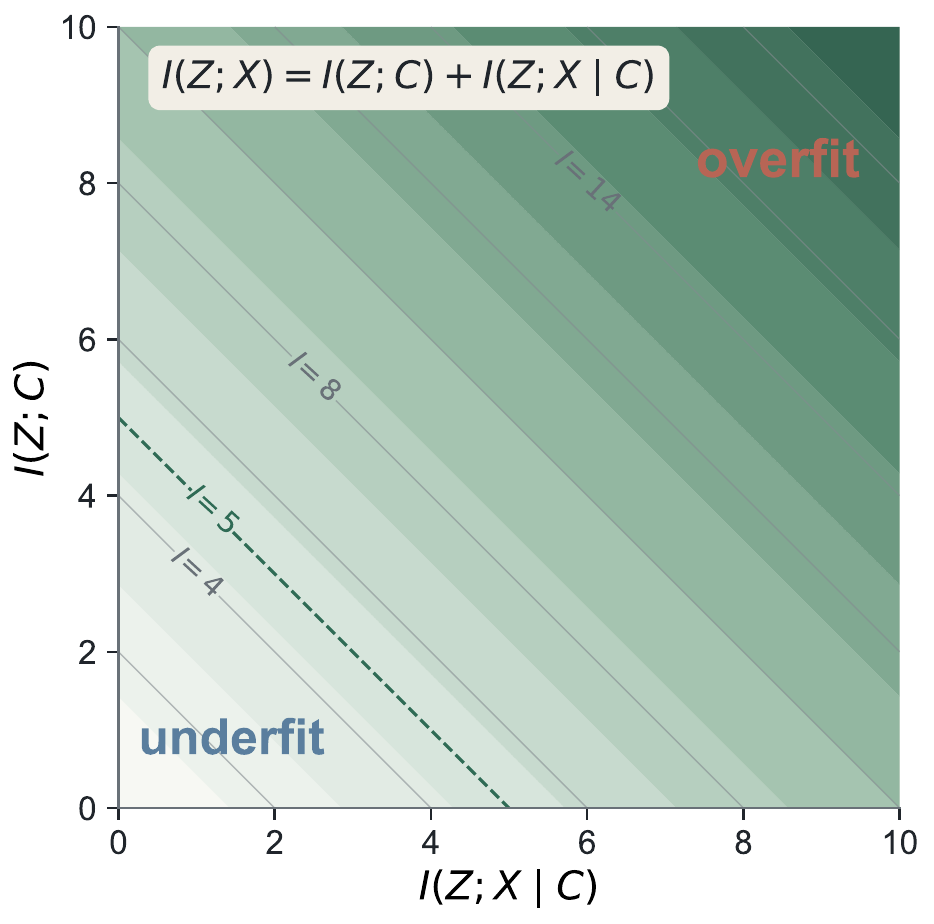}
        \vspace{-12pt}
        
        {\footnotesize \textbf{(a)} Information decomposition view}
    \end{minipage}\hfill
    \begin{minipage}[t]{0.24\textwidth}
        \centering
        \includegraphics[width=\linewidth]{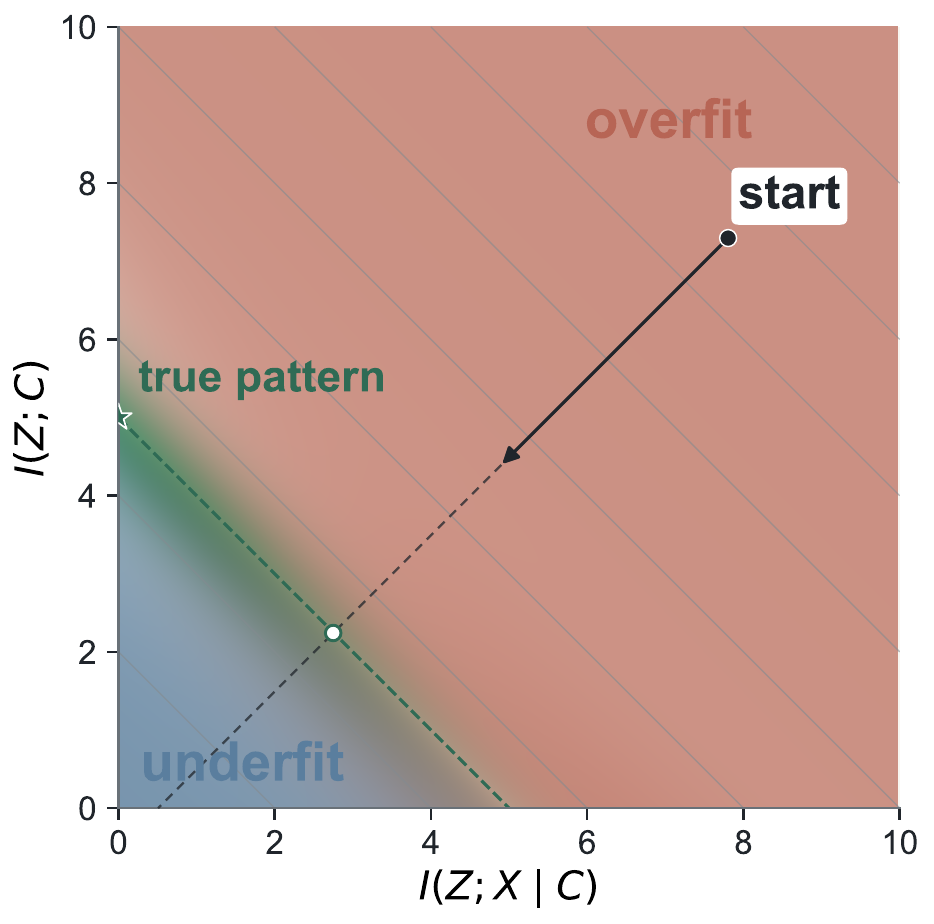}
        \vspace{-12pt}
        
        {\footnotesize \textbf{(b)} Effect of global control}
    \end{minipage}\hfill
    \begin{minipage}[t]{0.24\textwidth}
        \centering
        \includegraphics[width=\linewidth]{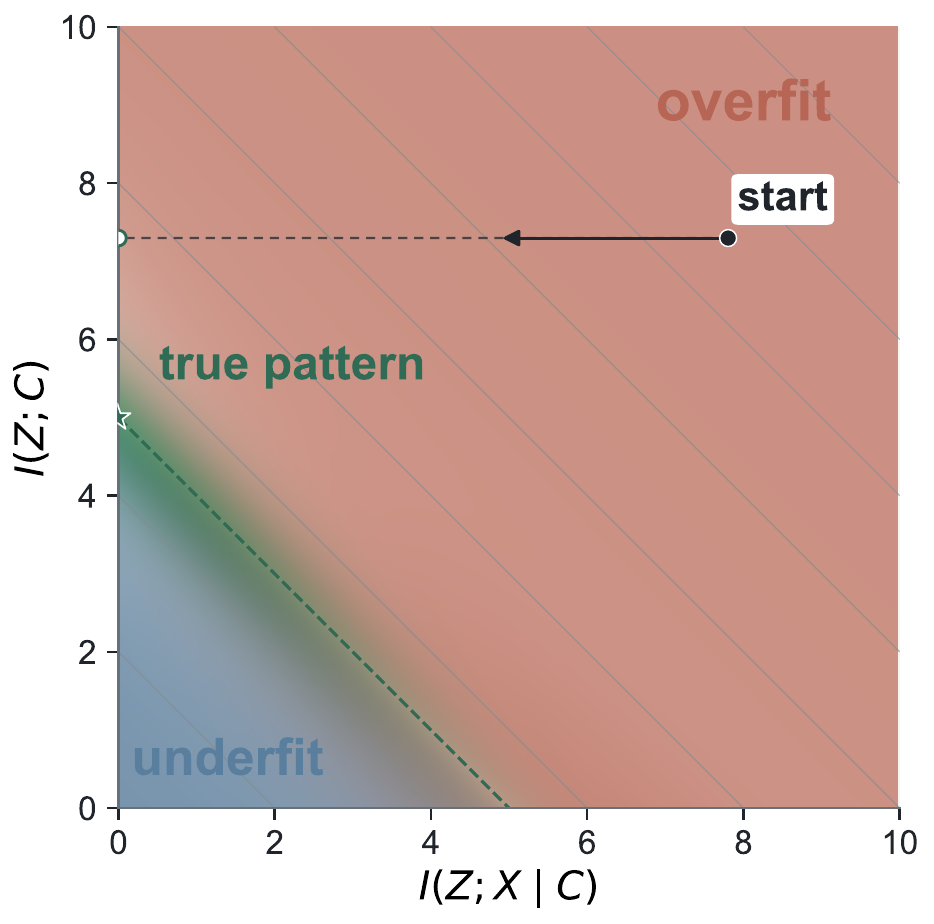}
        \vspace{-12pt}
        
        {\footnotesize \textbf{(c)} Effect of conditional control}
    \end{minipage}\hfill
    \begin{minipage}[t]{0.24\textwidth}
        \centering
        \includegraphics[width=\linewidth]{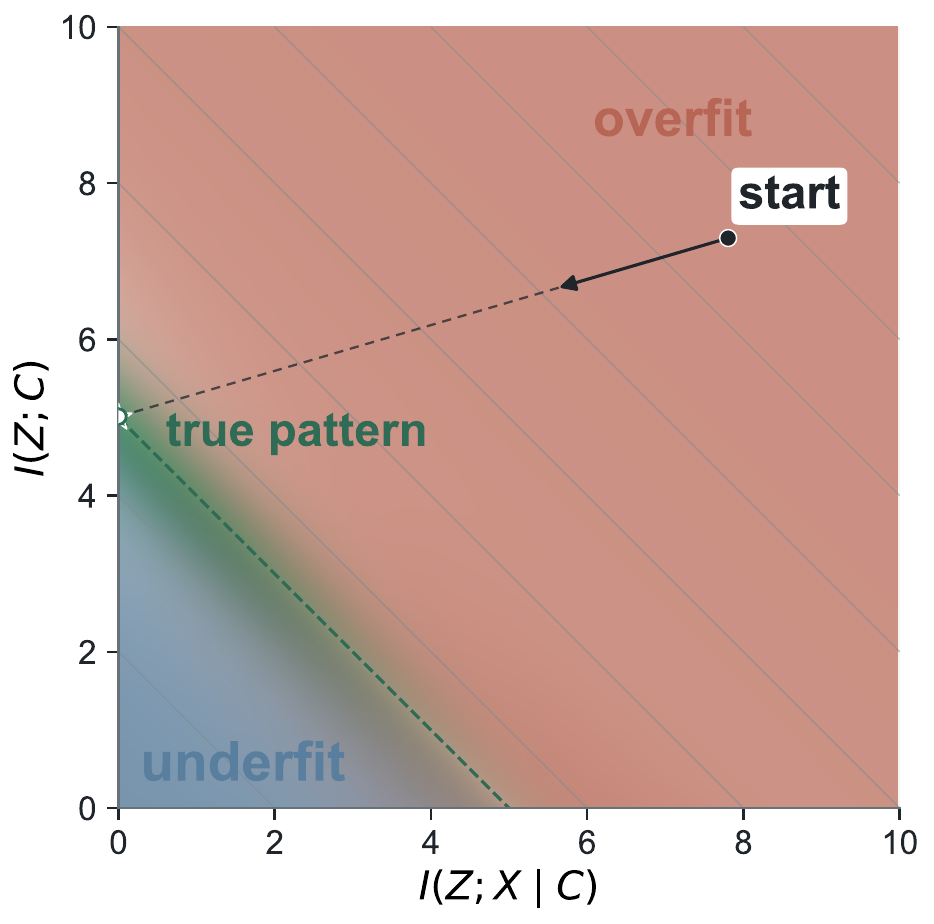}
        \vspace{-12pt}
        
        {\footnotesize \textbf{(d)} Effect of dual control}
    \end{minipage}

    \caption{\textbf{Information-plane view of representation control.}
Panel (a) illustrates the decomposition
\(I(Z;X)=I(Z;\condvar)+I(Z;X\mid \condvar)\).
Panels (b)--(d) schematically compare global-only, conditional-only, and dual control directions.}
    \label{fig:info_plane}
\end{figure}

\subsection{Why single-axis control is insufficient}
\label{subsec:theory_single_axis}
\label{subsec:insufficient}
\label{subsec:why_both}

Even when \(C=Y\), where the exact IB objective can be rewritten in these coordinates, its single compression coefficient couples the reward on \(I(Z;Y)\) with the penalty on \(I(Z;X\mid Y)\). More generally, a global-only bottleneck penalizes the aggregate \(I(Z;X)\), but does not provide an independent axis for regulating its split between \(I(Z;C)\) and \(I(Z;X\mid C)\).
Indeed, Appendix~\ref{app:thm2} constructs encoders with identical \(I(Z;X)\) but different \(I(Z;X\mid \condvar)\)
(Theorems~\ref{thm:insufficient}--\ref{thm:productspace}).
Therefore, global information control alone is insufficient for selectively suppressing within-condition residual variation.

Conversely, conditional control alone is also incomplete.
It can target within-condition dependence, but it does not by itself constrain global or condition-level encoding.
Appendix~\ref{subsec:invisible_condition_level_encoding} gives a structural regime in which a condition-level component \(B\) is invisible to \(I(Z;X\mid \condvar)\), while still occupying positive global information capacity
(Proposition~\ref{prop:invisible_condition_level_encoding}).
This shows that a purely conditional bottleneck can miss information aligned with the condition variable, making a global control axis still necessary.

\label{subsec:nuisance}
The within-condition term also admits an optional nuisance interpretation.
Suppose that
\[
X=T(S,\nuis),
\qquad
Y=f(S),
\qquad
\condvar=\partmap(Y),
\]
where \(S\) denotes task-relevant signal and \(\nuis\) denotes within-condition nuisance variation.
If, conditioned on \(\condvar\), the variables satisfy the Markov chain
\[
\nuis \to X \to Z,
\]
then the conditional data-processing inequality gives
\[
I(Z;\nuis\mid \condvar)
\le
I(Z;X\mid \condvar).
\]
Thus, when residual variation can be modeled as \(\nuis\), controlling \(I(Z;X\mid \condvar)\) also controls a visible nuisance-dependence axis.
This nuisance reading is optional: \(I(Z;X\mid \condvar)\) is not nuisance information in general, but it upper-bounds nuisance dependence under the stated conditional Markov assumption.
A formal statement is provided in Appendix~\ref{app:nuisance} (Proposition~\ref{prop:nuisance}).

Taken together, the theory identifies two complementary residual-information axes:
a visible within-condition nuisance axis, controlled through \(I(Z;X\mid \condvar)\), and a conditional-invisible condition-level axis, controlled through global information capacity.
This motivates a dual control view rather than a single-axis bottleneck, as illustrated schematically in Fig.~\ref{fig:info_plane}(b)--(d) and visualized in latent space in Fig.~\ref{fig:latent_compare}.

\begin{figure}[H]
    \centering
    \begin{minipage}[t]{0.24\textwidth}
        \centering
        \includegraphics[width=\linewidth]{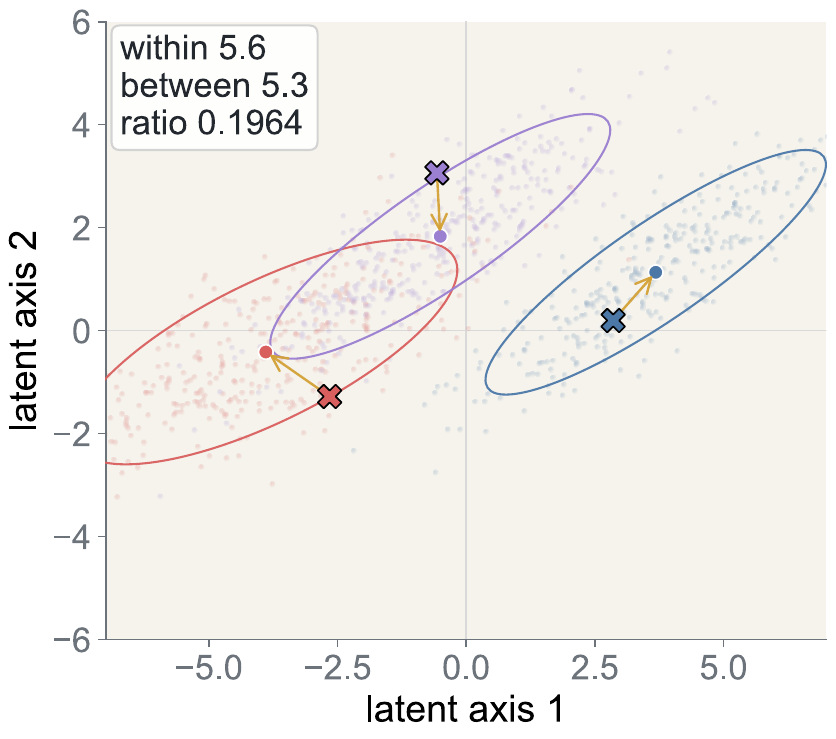}
        \vspace{-12pt}
        
        {\footnotesize \textbf{(a)} \noreg{}}
    \end{minipage}\hfill
    \begin{minipage}[t]{0.24\textwidth}
        \centering
        \includegraphics[width=\linewidth]{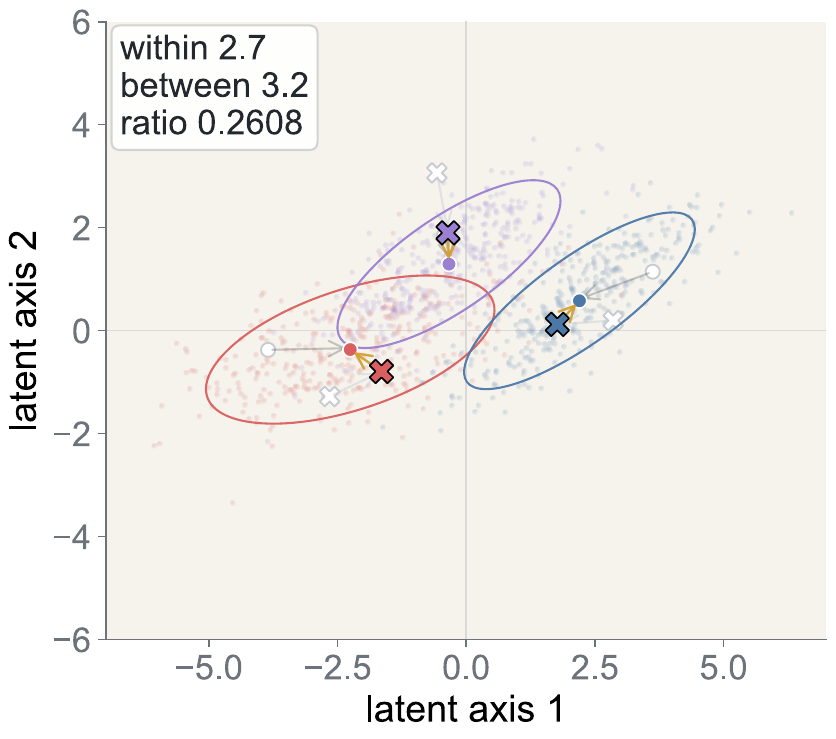}
        \vspace{-12pt}
        
        {\footnotesize \textbf{(b)} Global control}
    \end{minipage}\hfill
    \begin{minipage}[t]{0.24\textwidth}
        \centering
        \includegraphics[width=\linewidth]{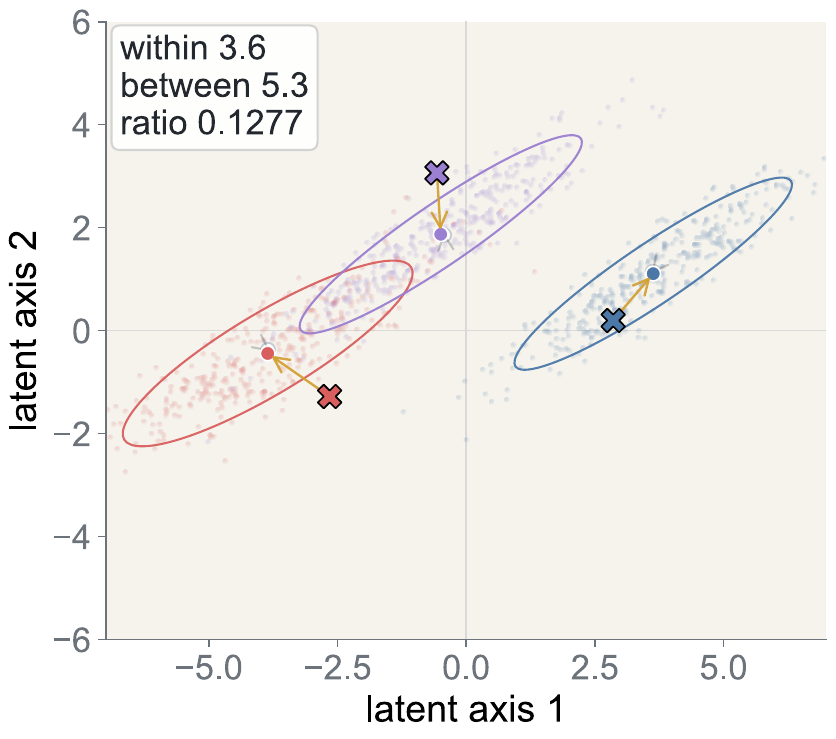}
        \vspace{-12pt}
           
        {\footnotesize \textbf{(c)} Conditional control}
    \end{minipage}\hfill
    \begin{minipage}[t]{0.24\textwidth}
        \centering
        \includegraphics[width=\linewidth]{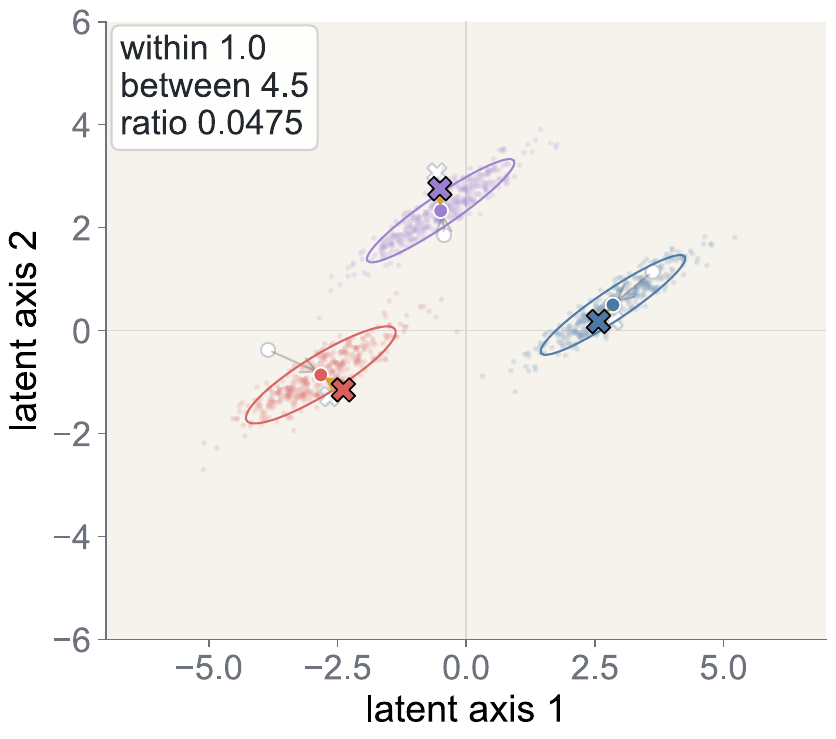}
        \vspace{-12pt}
        
        {\footnotesize \textbf{(d)} Dual control}
    \end{minipage}

    \caption{\textbf{Visual illustration of complementary control in latent space.}
    Global control mainly contracts overall spread, conditional control tightens within-condition structure, and dual control combines both effects in this illustrative example.}
    \label{fig:latent_compare}
    \vspace{0pt}
\end{figure}

\subsection{Main theorem: task-constrained two-axis upper-bound control}
\label{subsec:theory_main_theorem}
\label{subsec:pareto_synthesis}

We now collect the preceding information-control arguments into the main formal statement.
The goal is to characterize the regularizer-level control capacity of dual regularization under task adequacy, rather than to claim universal improvement for every optimized model.
Let
\[
\mathcal F_\varepsilon=\{(\theta,\omega,\psi):R_{\mathrm{task}}(\theta,\omega)\le\varepsilon\},
\quad
R_{\mathrm{task}}(\theta,\omega)
=
\mathbb{E}_{p(x,y)q_\theta(z\mid x)}
[\ell_{\mathrm{task}}(g_\omega(z),y)] .
\]
Under the nuisance interpretation of Appendix~\ref{app:nuisance}
(Proposition~\ref{prop:nuisance}) and the structural refinement regime of
Appendix~\ref{subsec:invisible_condition_level_encoding}
(Proposition~\ref{prop:invisible_condition_level_encoding}), define
\[
H_\lambda(\theta)
=
\lambda_{\mathrm{vis}}I_\theta(Z;\nuis\mid\condvar)
+
\lambda_{\mathrm{inv}}I_\theta(Z;B),
\qquad
U_\lambda(\theta,\psi)
=
\lambda_{\mathrm{vis}}J_{\mathrm{cond}}(\theta,\psi)
+
\lambda_{\mathrm{inv}}J_{\mathrm{std}}(\theta),
\]
with \(\lambda_{\mathrm{vis}},\lambda_{\mathrm{inv}}>0\).
Here \(H_\lambda\) measures visible nuisance dependence and conditional-invisible condition-level encoding, while \(U_\lambda\) is its KL-regularizer surrogate; \(J_{\mathrm{cond}}\) and \(J_{\mathrm{std}}\) are defined in Sec.~\ref{subsec:objective} and Appendix~\ref{app:final-logic}.

\begin{theorem}[Task-constrained upper-bound dominance of dual control]
\label{thm:pareto_main}
For every \((\theta,\omega,\psi)\in\mathcal F_\varepsilon\),
\[
H_\lambda(\theta)\le U_\lambda(\theta,\psi).
\]
\end{theorem}

The theorem states a pointwise upper bound: \(U_\lambda\) jointly controls the two residual-information axes identified in Sec.~\ref{subsec:theory_single_axis}.
This is a regularizer-level statement, not a claim that every optimized dual model universally improves task accuracy or realized nuisance information.
A complete proof is given in Appendix~\ref{app:final-logic} (Theorem~\ref{thm:pareto_upper_bound}).

\begin{corollary}[Family inclusion]
\label{cor:family_inclusion}
Let \(\Phi_{\mathrm{dual}}^\star(\varepsilon;\lambda)\),
\(\Phi_{\mathrm{std}}^\star(\varepsilon;\lambda)\), and
\(\Phi_{\mathrm{cond}}^\star(\varepsilon;\lambda)\) denote the task-feasible optimal values of \(U_\lambda\) over the supported dual, std-only, and cond-only families, respectively.
Then
\[
\Phi_{\mathrm{dual}}^\star(\varepsilon;\lambda)
\le
\min\{\Phi_{\mathrm{std}}^\star(\varepsilon;\lambda),
\Phi_{\mathrm{cond}}^\star(\varepsilon;\lambda)\}.
\]
\end{corollary}

This corollary records the family-level consequence that std-only and cond-only objectives are degenerate subfamilies of the dual family.
It is therefore a non-strict objective-family dominance statement under the same task-feasibility constraint.

\paragraph{Tightness.}
The pointwise bound becomes tight only when both data-processing gaps and variational prior-mismatch gaps vanish.
For the visible axis, this requires \(I(Z;X\mid\nuis,\condvar)=0\) and exact conditional prior calibration \(r_\psi(z\mid\condval)=\aggq(z\mid\condval)\) for \(p(\condval)\)-a.e. \(\condval\).
For the conditional-invisible condition-level axis controlled through global capacity, it requires \(I(Z;X\mid B)=0\) and exact global prior calibration \(r_0(z)=\aggq(z)\).
Thus, Theorem~\ref{thm:pareto_main} should be read as a controllable upper-bound result with strong equality conditions.

\section{Engineering instantiation}
\label{sec:engineering_instantiation}

\subsection{From mutual-information control to KL surrogates}
\label{subsec:objective}
\label{subsec:interpretation}
\label{subsec:kl_surrogate}

The theory in Sec.~\ref{sec:theory_control} identifies the information quantities that should be controlled, but the mutual informations \(I(Z;X)\) and \(I(Z;X\mid \condvar)\) are not directly optimized in our neural training pipeline.
We therefore instantiate the information-control view through variational KL upper bounds.

Define the global KL control
\(J_{\mathrm{std}}(\theta)=
\mathbb{E}_{p(x)}
\mathrm{KL}\!\bigl(q_\theta(z\mid x)\,\|\,r_0(z)\bigr)\),
where \(r_0(z)\) is an unconditional prior, and let
\(\aggq(z)=\int q_\theta(z\mid x)p(x)\,dx\)
be the aggregated posterior. Then
\begin{equation}
\label{eq:std_decomp}
J_{\mathrm{std}}(\theta)
=
I_\theta(Z;X)
+
\mathrm{KL}\!\bigl(\aggq(z)\,\|\,r_0(z)\bigr).
\end{equation}
Thus \(J_{\mathrm{std}}\) upper-bounds the global information budget \(I_\theta(Z;X)\); see Appendix~\ref{app:stdkl} (Lemma~\ref{lem:stdkl}).
The bound is tight under exact global prior calibration \(r_0(z)=\aggq(z)\), as stated in Appendix~\ref{app:stdkl} (Proposition~\ref{prop:tight_calibration}).

Similarly, define the conditional KL control
\(J_{\mathrm{cond}}(\theta,\psi)=
\mathbb{E}_{p(x,\condval)}
\mathrm{KL}\!\bigl(q_\theta(z\mid x)\,\|\,r_\psi(z\mid \condval)\bigr)\),
where \(r_\psi(z\mid \condval)\) is a condition-dependent prior, and let
\(\aggq(z\mid \condval)=
\int q_\theta(z\mid x)p(x\mid \condval)\,dx\)
be the condition-wise aggregated posterior. Then the conditional KL admits the exact decomposition
\begin{equation}
\label{eq:conddecomp}
J_{\mathrm{cond}}(\theta,\psi)
=
I_\theta(Z;X\mid \condvar)
+
\mathbb{E}_{p(\condval)}
\mathrm{KL}\!\bigl(\aggq(z\mid \condval)\,\|\,r_\psi(z\mid \condval)\bigr).
\end{equation}
A formal derivation is provided in Appendix~\ref{app:thm3} (Theorem~\ref{thm:condkl}).
Consequently,
\begin{equation}
\label{eq:cond_upper}
I_\theta(Z;X\mid \condvar)
\le
J_{\mathrm{cond}}(\theta,\psi).
\end{equation}
Equation~\eqref{eq:conddecomp} is the key closure property used by the engineering instantiation: the conditional KL is not merely a heuristic penalty, but an exact decomposition into within-condition information plus a condition-wise prior-mismatch term.
Under exact conditional prior calibration
\(r_\psi(z\mid \condval)=\aggq(z\mid \condval)\) for \(p(\condval)\)-a.e. \(\condval\),
the mismatch term vanishes and \(J_{\mathrm{cond}}(\theta,\psi)=I_\theta(Z;X\mid \condvar)\).

We therefore optimize the following dual-bottleneck objective:
\begin{equation}
\label{eq:finalobjective}
\mathcal{L}(\theta,\omega,\psi)
=
\mathbb{E}_{p(x,y)q_\theta(z\mid x)}
\bigl[\ell_{\mathrm{task}}(g_\omega(z),y)\bigr]
+
\beta_{\mathrm{std}}J_{\mathrm{std}}(\theta)
+
\beta_{\mathrm{cond}}J_{\mathrm{cond}}(\theta,\psi).
\end{equation}
Equivalently, expanding the two KL controls gives
\[
\mathcal{L}
=
\mathcal{L}_{\mathrm{task}}
+
\beta_{\mathrm{std}}
\mathbb{E}_{p(x)}
\mathrm{KL}\!\bigl(q_\theta(z\mid x)\,\|\,r_0(z)\bigr)
+
\beta_{\mathrm{cond}}
\mathbb{E}_{p(x,\condval)}
\mathrm{KL}\!\bigl(q_\theta(z\mid x)\,\|\,r_\psi(z\mid \condval)\bigr).
\]
Here \(\condval=\partmap(y)\) indexes the conditional prior, not the posterior encoder. The standard KL implements global capacity control, while the conditional KL implements condition-aware residual control through Eq.~\eqref{eq:conddecomp}; Fig.~\ref{fig:framework} summarizes the instantiation.

\begin{figure*}[t!]
    \centering
    \includegraphics[width=1.0\textwidth]{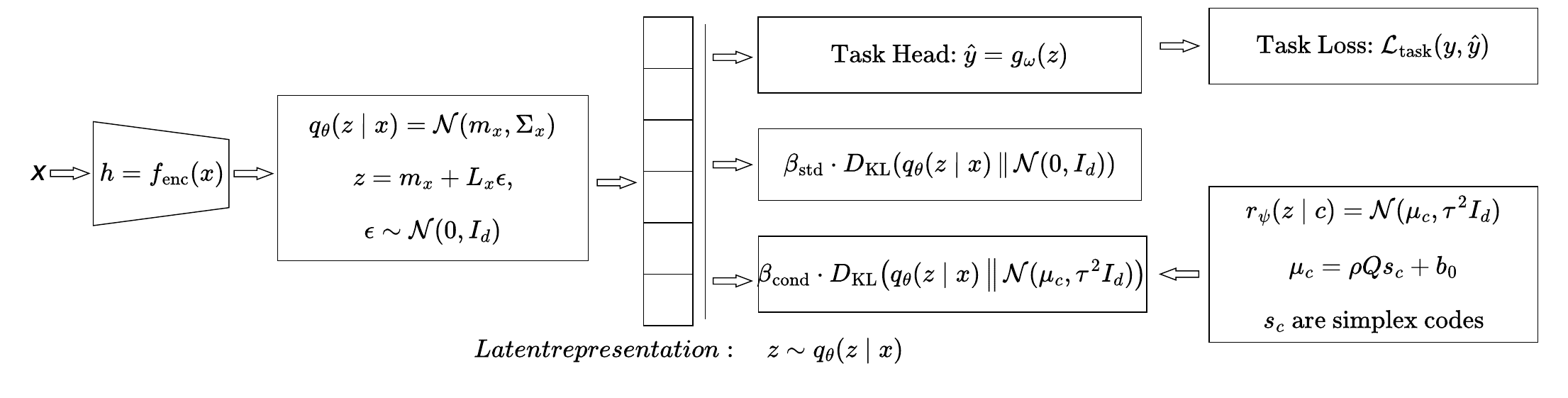}
    \vspace{-24pt}
    \caption{\textbf{Engineering instantiation of the dual-bottleneck objective.}
The encoder produces a stochastic latent representation \(Z\), the task head optimizes the supervised loss, and two KL surrogates implement global and condition-aware information control.}
    \label{fig:framework}
\end{figure*}

\subsection{Structured conditional prior}
\label{subsec:simplex_prior}

For the default simplex-constrained instantiation, we use a Gaussian conditional prior
\[
r_\psi(z\mid \condval)
=
\mathcal{N}(\mu_{\condval},\tau^2 \Id_d),
\qquad
\mu_{\condval}
=
\rho \embedmat s_{\condval}+\priorbias,
\]
where \(s_{\condval}\) is the \(\condval\)-th vertex of a regular simplex, \(\rho=\operatorname{softplus}(\rho_{\mathrm{raw}})>0\), \(\rho_{\mathrm{raw}}\), \(R_{\mathrm{raw}}\), and \(\priorbias\) are learned, and \(\tau=1\) is fixed by default.

We take the simplex coordinates to satisfy
\(s_{\condval}\in\mathbb{R}^{\ncond-1}\), \(\condval\in\{1,\ldots,\ncond\}\), and assume \(d\ge \ncond-1\).
Accordingly, \(\embedmat=\operatorname{qf}(R_{\mathrm{raw}})\) satisfies \(\embedmat^\top\embedmat=\Id_{\ncond-1}\).
Thus \(\embedmat\) acts as an isometric embedding of the simplex coordinates into the latent space.

In particular, the pairwise center distance is constant:
\begin{equation}
\label{eq:simplex_distance}
\|\mu_{\condval}-\mu_{\condval'}\|^2
=
2\rho^2\frac{\ncond}{\ncond-1},
\qquad
\condval\neq \condval'.
\end{equation}

This prior does not introduce an additional theoretical assumption for the decomposition in Sec.~\ref{sec:theory_control}.
Instead, it provides a symmetric and controllable parameterization of \(r_\psi(z\mid \condval)\), avoiding direct dependence on empirical class centers at initialization.
The induced symmetric separability property is proved in Appendix~\ref{app:simplex} (Proposition~\ref{prop:simplex}).

\subsection{Shared variance and operational training}
\label{subsec:tau}
\label{subsec:scope}
\label{subsec:operational_instantiation}

In practice, we use a shared isotropic variance \(\tau^2\Id_d\) across conditions.
This choice is motivated by control and interpretability rather than by a universal noise model.
Under a Gaussian posterior family \(q_\theta(z\mid x)=\mathcal{N}(m_x,\Sigma_x)\), the average conditional KL given current centers takes the form
\begin{equation}
\label{eq:tau_objective}
J(\tau^2)
=
\frac12
\left(
\frac{A}{\tau^2}
+
d\log \tau^2
\right)
+
\mathrm{const},
\end{equation}
where
\[
A
=
\mathbb{E}_{p(x,\condval)}
\bigl[
\mathrm{tr}(\Sigma_x)+\|m_x-\mu_{\condval}\|^2
\bigr].
\]
In the nondegenerate case \(A>0\), its unique minimizer is
\[
\tau_\star^2=\frac{A}{d}.
\]
Thus, under a shared isotropic parameterization, learning \(\tau\) mainly adjusts a single global dispersion scale.
Fixing \(\tau=1\) provides a controlled way to decouple condition-center geometry from variance adaptation during regularization.
The corrected minimization argument is given in Appendix~\ref{app:tau} (Proposition~\ref{prop:tau}).

Operationally, the method follows a standard variational pipeline with minimal changes.
The backbone \(\featenc\), task head \(g_\omega\), and supervised loss \(\ell_{\mathrm{task}}\) are unchanged.
The encoder parameterizes a stochastic posterior \(q_\theta(z\mid x)\), and the latent representation \(Z\) is used by the task head as shown in Fig.~\ref{fig:framework}.
Under Gaussian posterior and prior families, both KL terms admit closed-form evaluation, so the additional computational overhead is negligible relative to standard variational objectives.

The two weights \(\beta_{\mathrm{std}}\) and \(\beta_{\mathrm{cond}}\) tune the two controls independently.
Together with the simplex-structured conditional prior in Sec.~\ref{subsec:simplex_prior} and the shared-variance setting above, Eq.~\eqref{eq:finalobjective} gives a symmetric and operationally interpretable regularization scheme that can be integrated into existing classification or segmentation pipelines without architectural changes.
The loss-agnostic representational interpretation of this regularizer is summarized in Appendix~\ref{app:universality} (Proposition~\ref{prop:loss_agnostic}).

\section{Experiments}
\label{sec:experiments}

We structure the main experiments as an objective-level ablation motivated by the two control axes
in our theory. Under the same backbone, data split, optimizer, schedule, and task loss, we compare
\textsc{No Reg}, \textsc{StdKL}, \textsc{CondKL}, and \textsc{DualKL}, isolating global control,
conditional within-condition control, and their combination. We report BAcc/Macro-F1 for
classification and IoU/Dice for segmentation, and evaluate both full-data and low-data regimes.
A separate instantiation ablation studies the conditional-prior geometry and $\tau$ parameterization.

On BloodMNIST~\citep{yang2023medmnistv2,acevedo2020blood}, detailed results are reported in Tab.~\ref{tab:app_bloodmnist_results}. Both single-KL variants already improve over the no-regularization baseline, but \textsc{DualKL} consistently achieves the best BAcc and Macro-F1 across all regimes. The improvement is visible even in the standard regime, where BAcc/Macro-F1 increase from 92.90/93.14 to 94.52/94.72, and becomes substantially larger when supervision is limited: under 1\% training data, BAcc/Macro-F1 increases from 63.49/61.43 to 72.22/69.98, and under 4-shot training, they increase from 45.31/43.19 to 48.72/46.57. This pattern supports our motivation that, when supervision is scarce, models are more vulnerable to unstable residual cues, so representation-level regularization, and especially the additional structured control in \textsc{DualKL}, becomes more beneficial (Fig.~\ref{fig:bloodmnist_lowdata_vis}). Compared with external regularization and geometry baselines, the strongest baseline is Center Loss, while \textsc{DualKL} remains higher by 2.44 BAcc and 1.99 Macro-F1 points in the 1\% regime; see Appendix~\ref{app:exp_results}, Tab.~\ref{tab:app_bloodmnist_external_baselines}. Appendix~\ref{app:cifar100_lowdata} further reports a CIFAR-100~\citep{krizhevsky2009learning} low-data natural-image check, where KL-regularized variants improve over \noreg{} and \dualkl{} gives the best mean scores.

\begin{figure}[!htbp]
    \centering
    \includegraphics[width=0.213\linewidth]{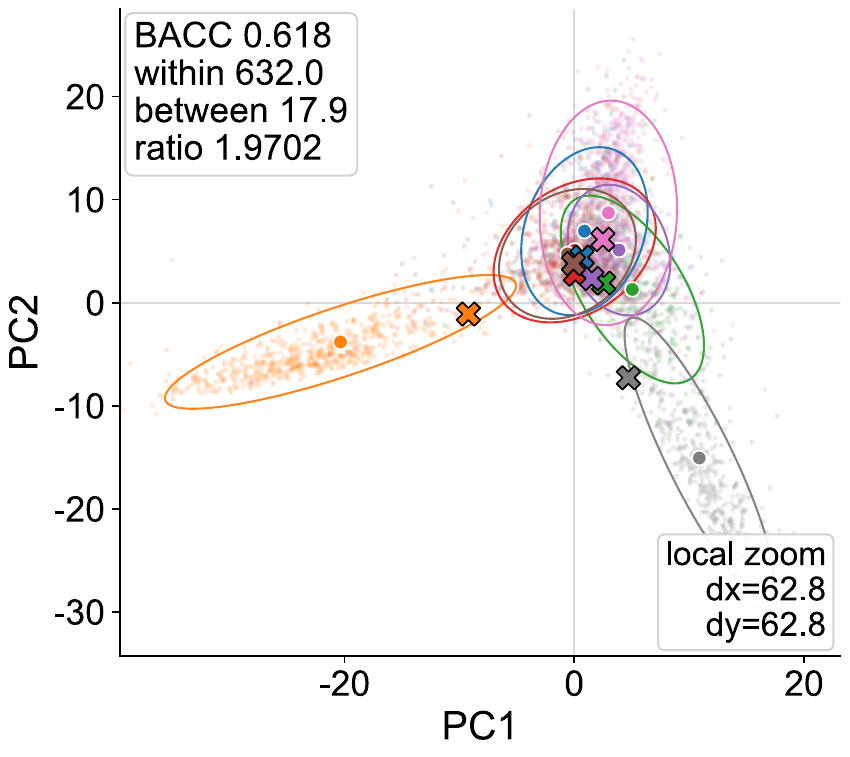}\hspace{0.004\linewidth}
    \includegraphics[width=0.237\linewidth]{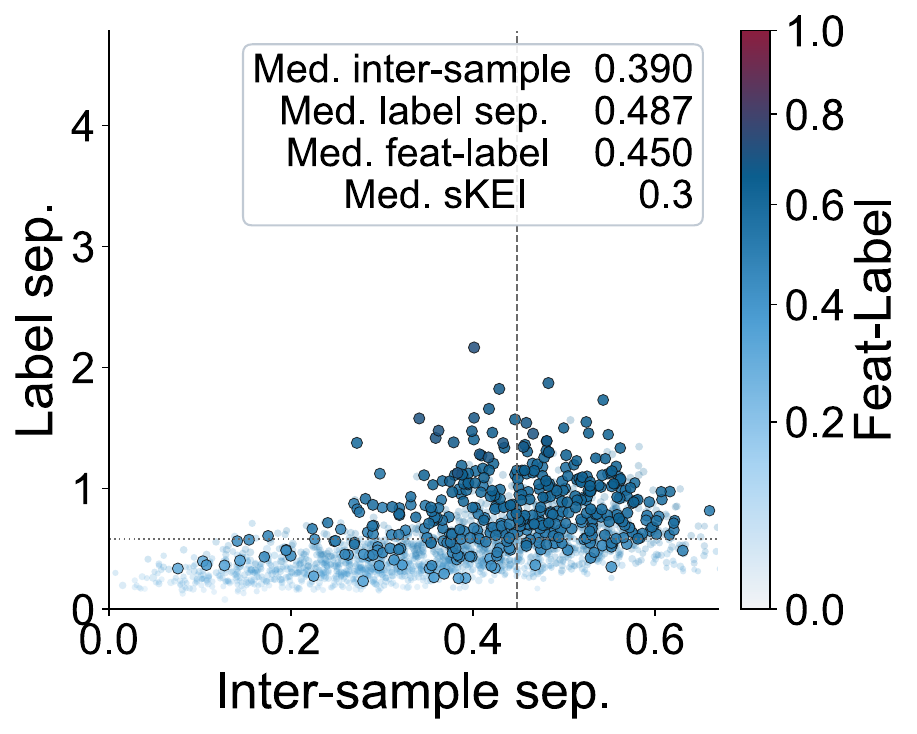}\hspace{0.010\linewidth}
    \includegraphics[width=0.213\linewidth]{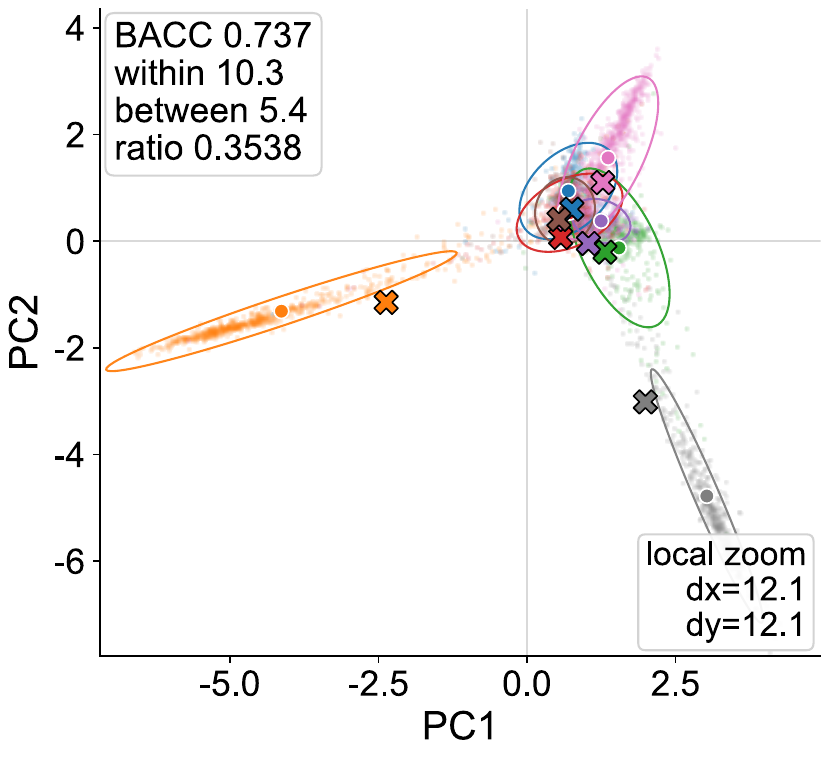}\hspace{0.004\linewidth}
    \includegraphics[width=0.237\linewidth]{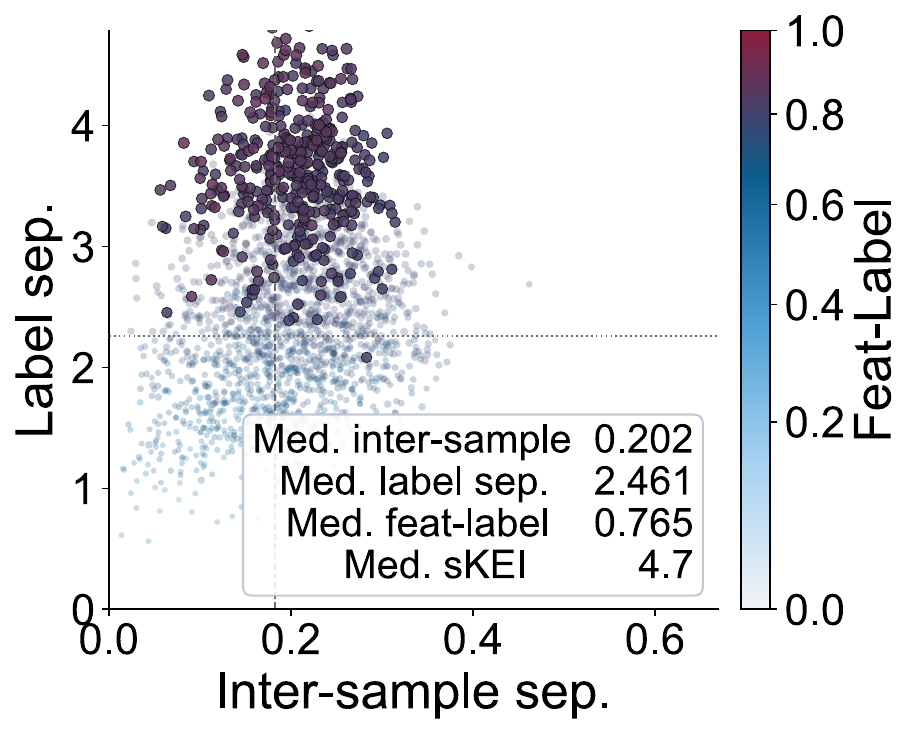}
    \caption{BloodMNIST, 1\% training data. \noreg{} (left two panels) versus \dualkl{} (right two panels), shown in PCA and latent-statistics views. Full four-way comparisons are in Appendix~\ref{app:vis_results}.}
    \label{fig:bloodmnist_lowdata_vis}
\end{figure}

For retinal vessel segmentation~\citep{staal2004ridge,hoover2000locating,fraz2012ensemble}, Tab.~\ref{tab:app_retinal_results} shows a consistent cross-dataset pattern. \textsc{DualKL} obtains the best IoU and Dice on DRIVE, STARE, and CHASEDB1, while both single-KL variants also improve over the unregularized baseline. The consistent gains across KL variants suggest that dense prediction benefits from both forms of information control, with their combination giving the strongest overall setting. On GLAS~\citep{sirinukunwattana2017gland}, Tab.~\ref{tab:app_glas_results} further shows that \textsc{DualKL} improves over the unregularized baseline in both full-data and few-shot regimes, with the clearest gain appearing in the 4-shot setting (IoU 0.5368 to 0.5478; Dice 0.6958 to 0.7049). Representative latent-geometry comparisons on CHASEDB1 are shown in Fig.~\ref{fig:chasedb1_main_vis}, where \textsc{DualKL} yields a more compact and separable structure than the unregularized baseline. Exact experimental protocols are provided in Appendix~\ref{app:experimental_setup}.

\begin{figure}[!t]
    \centering
    \includegraphics[width=0.205\linewidth]{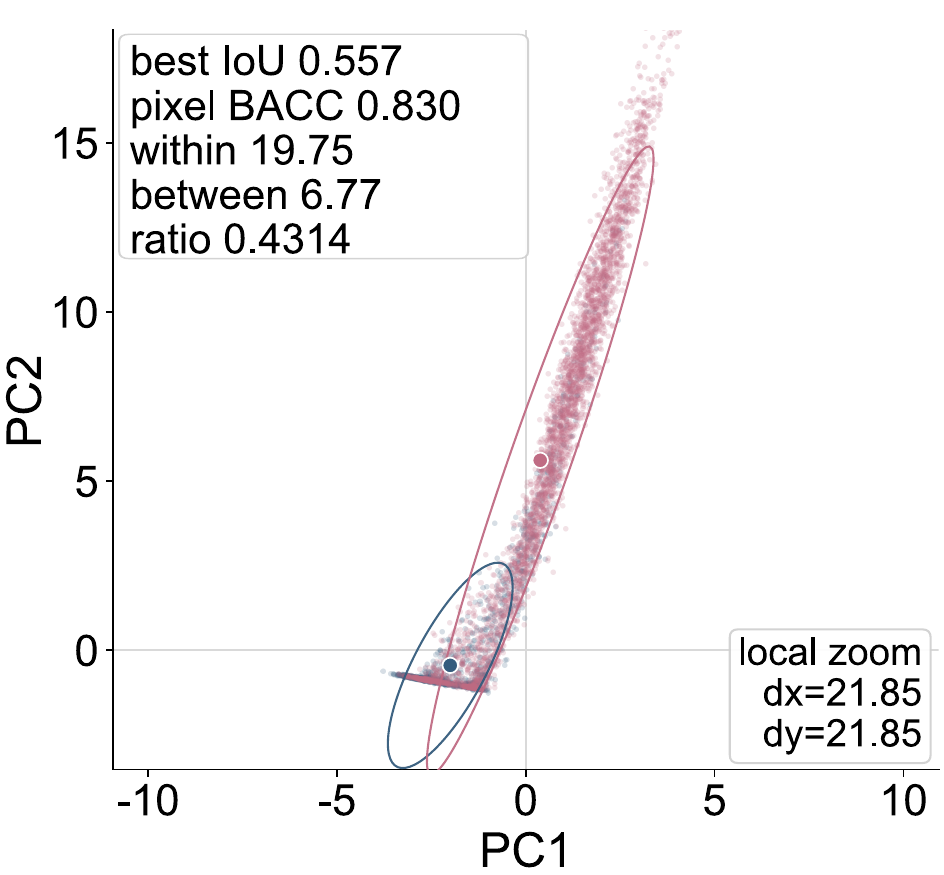}\hspace{0.003\linewidth}
    \includegraphics[width=0.225\linewidth]{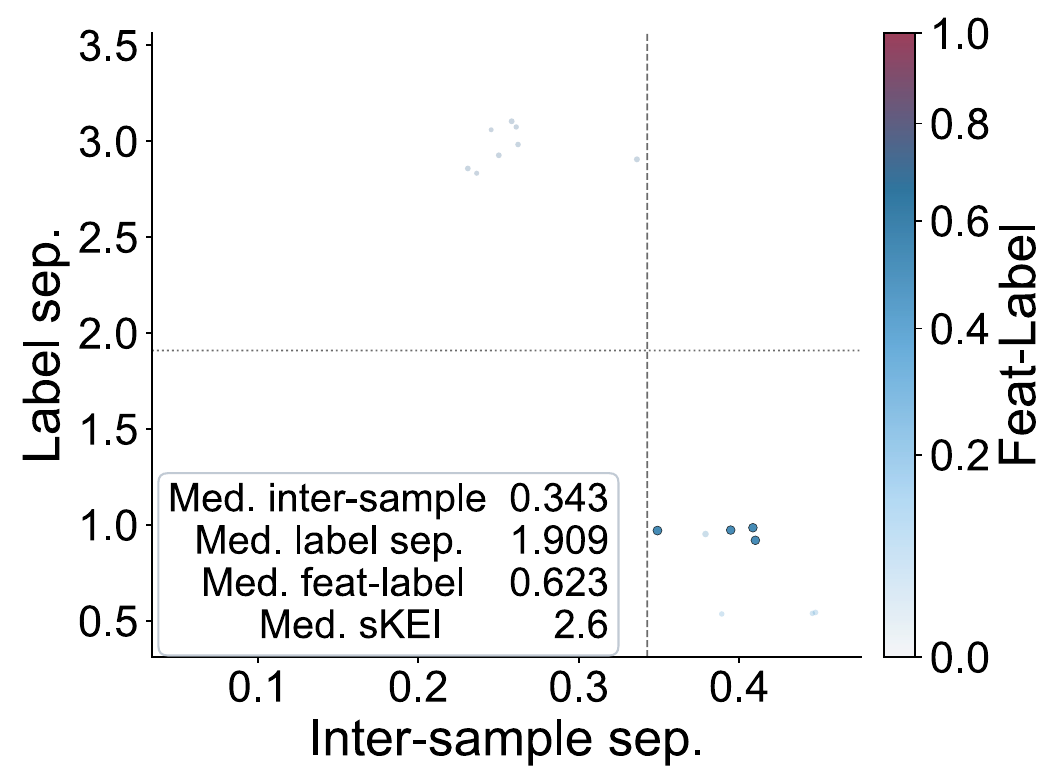}\hspace{0.008\linewidth}
    \includegraphics[width=0.205\linewidth]{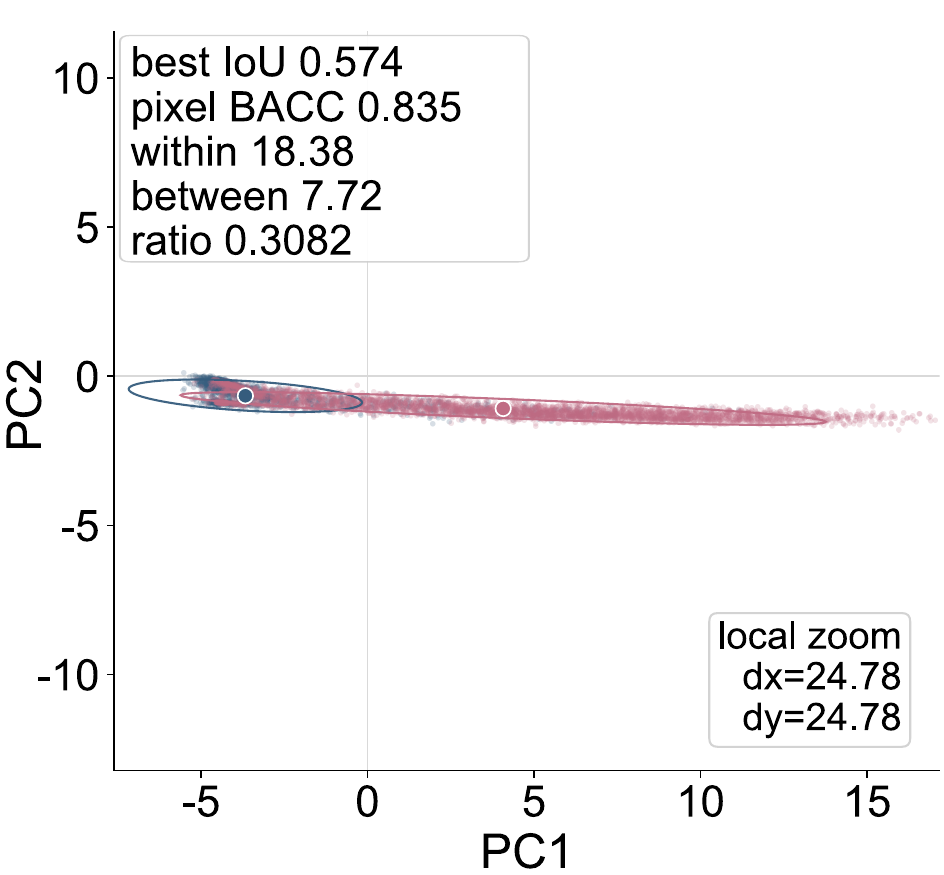}\hspace{0.003\linewidth}
    \includegraphics[width=0.225\linewidth]{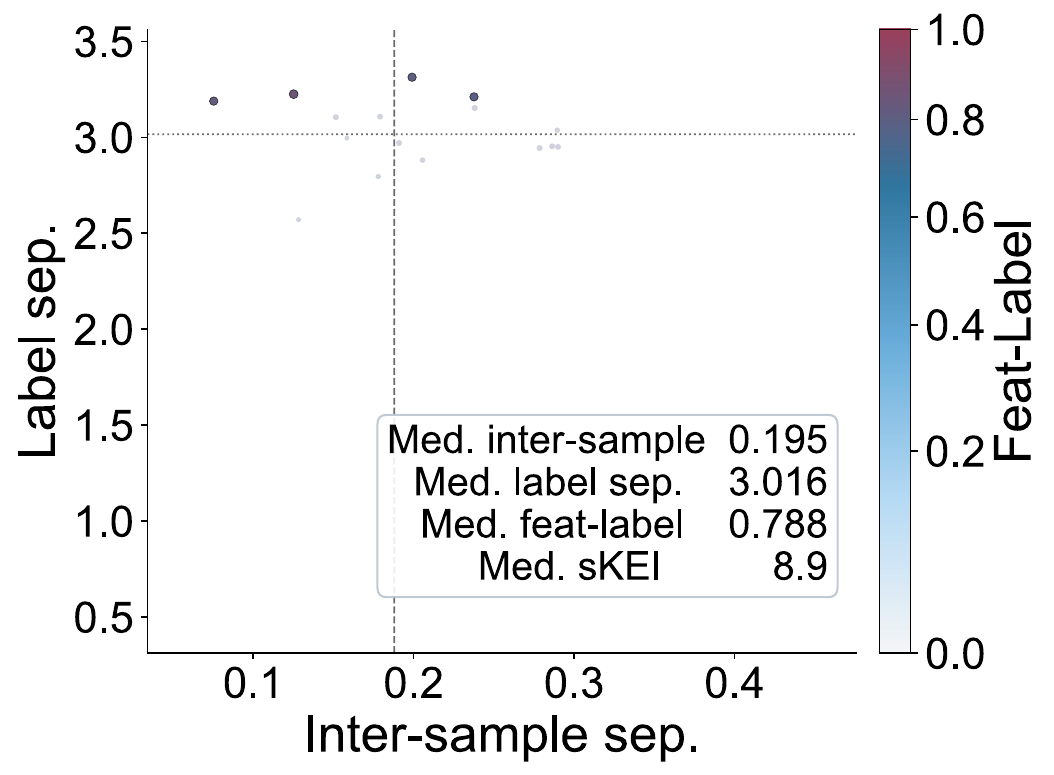}
    \caption{CHASEDB1. \noreg{} versus \dualkl{} in PCA and latent-statistics views; full four-way comparisons are in Appendix~\ref{app:vis_results}.}
    \label{fig:chasedb1_main_vis}
\end{figure}

Additional BloodMNIST ablations further support the default practical instantiation of our method. In the 1\% training-data regime, Tab.~\ref{tab:bloodmnist_ablation_1pct} shows that the default configuration, \textit{\dualkl{} + simplex + shared fixed $\tau$}, achieves the best mean BAcc and Macro-F1 among the tested variants. The same qualitative pattern also appears in the 4-shot and full-training settings (Appendix~\ref{app:bloodmnist_prior_tau_results}, Tabs.~\ref{tab:4shot_tuned_ablation_meanstd} and \ref{tab:full_ablation_subset_meanstd}), where the default configuration reaches 48.66/46.59 and 94.74/94.86 in BAcc/Macro-F1, respectively. The margin is clearest in the low-data settings and remains positive in the full-training regime, supporting the simplex prior with shared fixed $\tau$ as a stable and interpretable default among the tested variants.

\begin{table}[h]
\centering
\caption{BloodMNIST 1\% training-data ablation results reported as mean $\pm$ standard deviation.}
\label{tab:bloodmnist_ablation_1pct}
\small
\setlength{\tabcolsep}{4.5pt}
\begin{tabular}{lcc}
\toprule
Configuration & BAcc & Macro-F1 \\
\midrule
\dualkl{} + simplex + shared fixed $\tau$ & \textbf{72.11 $\pm$ 1.42} & \textbf{69.97 $\pm$ 1.71} \\
\dualkl{} + random + shared fixed $\tau$ & 71.42 $\pm$ 1.53 & 69.11 $\pm$ 1.37 \\
\dualkl{} + empirical + shared fixed $\tau$ & 71.63 $\pm$ 1.72 & 69.94 $\pm$ 2.08 \\
\dualkl{} + simplex + classwise learned $\tau$ & 71.88 $\pm$ 2.62 & 69.65 $\pm$ 2.54 \\
\bottomrule
\end{tabular}
\end{table}

As an auxiliary check, we vary $\pi$ in segmentation using local mask-derived
conditions. Appendix Tab.~\ref{tab:partition_choice_ablation} shows that the average
IoU/Dice over CHASEDB1, DRIVE, GLAS, and STARE changes from 62.40/76.24 under
direct label conditioning ($K=2$) to 62.66/76.44 under foreground-count
conditioning ($K=10$) and 62.75/76.52 under local-pattern conditioning
($K=512$). This indicates that \(\condvar=\partmap(Y)\) is not merely a notational
generalization; full details and sensitivity analyses are in
Appendices~\ref{app:partition_choice}--\ref{app:beta_sensitivity}.

\section{Conclusion, limitation, and future work}
\label{sec:futurework}

We introduced a structured information-control view of bottleneck regularization under label-induced partitions. 
Unlike standard IB/VIB objectives that constrain a single global information quantity, our formulation separates global capacity control from within-condition residual control. 
The main theorem shows that this dual regularizer simultaneously upper-bounds two structurally distinct residual-information axes\,---\,visible within-condition nuisance dependence and conditional-invisible condition-level encoding\,---\,neither of which is controlled by either single-term bottleneck alone. 
Operationally, we instantiate this view through KL upper bounds, a simplex-structured conditional prior, and a shared-variance parameterization, yielding a regularizer that can be integrated into existing supervised learning pipelines. 
Across classification and segmentation benchmarks, the method shows consistent improvements over the unregularized baseline, with the strongest gains in low-data classification and stable improvements across dense prediction benchmarks.

The proposed method is best viewed as a representation-level regularizer for settings where label-relevant structure coexists with residual within-condition variation, consistent with prior analyses showing that the benefit of information constraints depends on the data-generating structure~\citep{federici2021information,kawaguchi2023}. Its empirical benefit can depend on the model family, optimization, task loss, and whether within-condition variation is residual rather than task-relevant; gains may be smaller when classes are already highly separable or intra-class variation itself carries task semantics. Beyond medical imaging, related settings such as PCB defect inspection may also benefit when similar residual variation is present. Future work may explore richer conditional-prior families, more flexible geometries, and heterogeneous or condition-dependent residual structures.

\bibliographystyle{unsrt}
\bibliography{references}

%%%%%%%%%%%%%%%%%%%%%%%%%%%%%%%%%%%%%%%%%%%%%%%%%%%%%%%%%%%%
\newpage
\appendix{}

\section{Detailed proofs for the label-induced conditional bottleneck}
\label{app:proofs}

This appendix provides detailed derivations for the mathematical claims used in the main method.
Our goal is not only to state the final objective, but also to make explicit:
(i) which statements are exact identities,
(ii) which statements are upper bounds,
(iii) which statements are existence/insufficiency results,
and (iv) which parts specify practical geometric or initialization choices used to instantiate the regularizer.

\subsection{Notation and proof roadmap}
\label{app:notation}

\paragraph{Standing regularity assumptions.}
Throughout this note, we work with stochastic encoder families \(q_\theta(z\mid x)\).
We assume that the KL divergences and mutual informations appearing below are well-defined and finite,
and that the relevant aggregated posteriors \(\aggq(z)\) and \(\aggq(z\mid \condval)\) exist.
In the continuous setting, this excludes degenerate deterministic limits of the encoder
for the purpose of the stated identities and bounds.

Let \(X \in \mathcal{X}\) be the input, \(Y \in \mathcal{Y}\) the supervised target, and
\(Z \in \mathbb{R}^d\) the stochastic latent variable generated by the encoder \(q_\theta(z\mid x)\).

We use a label-induced finite condition random variable
\[
\condvar = \partmap(Y),
\]
where
\[
\partmap : \mathcal{Y}\to\condset,
\qquad
\condset=\{1,\dots,\ncond\},
\]
is a finite partition map. In the clean supervised interpretation used in
Sec.~\ref{subsec:assumptions_scope}, one may write \(Y=\cleanrule(X)\), and hence
\(\condvar=\partmap(\cleanrule(X))\). This deterministic clean-label case is a
concrete instance of our setting, but the Markov-based decomposition in
Theorem~\ref{thm:decomp} only requires the encoder-induced conditional
independence stated below.

\paragraph{Encoder-induced Markov convention.}
All information quantities involving \(Z\) are evaluated under the joint law induced by the data distribution and the stochastic encoder. For any upstream variable \(W\) that is not directly provided as an encoder input,
\[
p_\theta(w,x,z)=p(w,x)q_\theta(z\mid x),
\]
and hence
\[
p_\theta(z\mid x,w)=q_\theta(z\mid x)=p_\theta(z\mid x).
\]
Thus \(W\to X\to Z\) forms a Markov chain. We use this convention with \(W=\condvar\) in Theorem~\ref{thm:decomp}, with \(W=\nuis\) in the conditional nuisance interpretation of Proposition~\ref{prop:nuisance}, and with \(W=B\) in the proof of Theorem~\ref{thm:pareto_upper_bound}.

\paragraph{Proof roadmap.}
The proof chain is organized into four layers according to their role in the paper.

\begin{enumerate}
    \item \textbf{Core justification.}
    These results establish the main theoretical backbone of the method:
    \begin{itemize}
        \item Theorem~\ref{thm:decomp}: exact decomposition of \(I(Z;X)\) under a label-induced condition.
        \item Theorems~\ref{thm:insufficient} and~\ref{thm:productspace}: insufficiency of a global-only bottleneck for the stated design goal.
        \item Theorem~\ref{thm:condkl}: exact decomposition of the conditional KL term.
        \item Theorems~\ref{thm:conditionalinsufficient} and~\ref{thm:conditionalinsufficient_product}: insufficiency of a purely conditional bottleneck for controlling total global capacity.
    \end{itemize}

    \item \textbf{Supporting calibration results.}
    These results clarify how the two KL terms relate to their corresponding information quantities under variational parameterization:
    \begin{itemize}
        \item Lemma~\ref{lem:stdkl}: exact decomposition of the standard KL term.
        \item Proposition~\ref{prop:tight_calibration}: tightness statements under exact prior calibration.
    \end{itemize}

    \item \textbf{Optional structural interpretations.}
    These results are not required for the main dual-bottleneck justification, but provide additional structural readings of the framework:
    \begin{itemize}
        \item Proposition~\ref{prop:nuisance}: a nuisance-suppression interpretation under an additional conditional Markov assumption.
        \item Proposition~\ref{prop:invisible_condition_level_encoding}: a structural regime in which condition-level encoding is invisible to the conditional term.
    \end{itemize}

    \item \textbf{Instantiation results.}
    These results concern the specific parameterization used to operationalize the framework:
    \begin{itemize}
        \item Proposition~\ref{prop:simplex}: the symmetric separability induced by the simplex prior.
        \item Proposition~\ref{prop:tau}: the corrected minimization argument for the shared variance parameter \(\tau\).
    \end{itemize}
\end{enumerate}

\subsection{Exact decomposition under a label-induced condition}
\label{app:thm1}

\begin{theorem}[Exact decomposition under encoder-induced Markovity]
\label{thm:decomp}
Suppose that the representation is generated from the input through the stochastic
encoder \(q_\theta(z\mid x)\), so that \(\condvar\to X\to Z\) forms a Markov chain
as in Appendix~\ref{app:notation}. Then
\[
I(Z;X)=I(Z;\condvar)+I(Z;X\mid \condvar).
\]
\end{theorem}

\begin{proof}
Start from the chain rule for mutual information:
\[
I(Z;X,\condvar)=I(Z;\condvar)+I(Z;X\mid \condvar).
\]
On the other hand, the chain rule also gives
\[
I(Z;X,\condvar)=I(Z;X)+I(Z;\condvar\mid X).
\]
By the encoder-induced Markov chain \(\condvar\to X\to Z\),
\[
I(Z;\condvar\mid X)=0.
\]
Hence \(I(Z;X,\condvar)=I(Z;X)\). Combining the two identities yields
\[
I(Z;X)=I(Z;\condvar)+I(Z;X\mid \condvar).
\]
This proves the claim. In the clean supervised setting described in
Appendix~\ref{app:notation} and Sec.~\ref{subsec:assumptions_scope},
\(\condvar=\partmap(\cleanrule(X))\) is a sufficient concrete instance of this
Markov condition, but determinism of \(\condvar\) is not required for the decomposition.
\end{proof}

\paragraph{Interpretation.}
The decomposition is exact. It shows that the total information budget \(I(Z;X)\) mixes:
\[
\underbrace{I(Z;C)}_{\text{condition-level / label-induced structure}}
\quad\text{and}\quad
\underbrace{I(Z;X\mid \condvar)}_{\text{within-condition information}}.
\]
Therefore, a single global information budget does not by itself distinguish these two roles.

\subsection{Why a single global bottleneck is insufficient}
\label{app:thm2}

\begin{theorem}
\label{thm:insufficient}
There exist two encoders \(q_1(z\mid x)\) and \(q_2(z\mid x)\) such that
\[
I_{q_1}(Z;X)=I_{q_2}(Z;X),
\qquad
I_{q_1}(Z;X\mid \condvar)\neq I_{q_2}(Z;X\mid \condvar).
\]
Hence, any regularizer depending only on \(I(Z;X)\) cannot, in general, separately control
within-condition information.
\end{theorem}

\begin{proof}
Let \(\condvar\sim \mathrm{Bernoulli}(1/2)\) and \(\nuis\sim \mathrm{Bernoulli}(1/2)\) be independent random variables.
Define the input
\[
X=(\condvar,\nuis).
\]
Now define two deterministic encoders:
\[
Z_1=\condvar,\qquad Z_2=\nuis.
\]

Since \(Z_1\) is a deterministic one-bit function of \(X\),
\[
I(Z_1;X)=H(Z_1)-H(Z_1\mid X)=H(Z_1)=1.
\]
Similarly, since \(Z_2\) is also a deterministic one-bit function of \(X\),
\[
I(Z_2;X)=H(Z_2)-H(Z_2\mid X)=H(Z_2)=1.
\]
Therefore,
\[
I(Z_1;X)=I(Z_2;X).
\]

We next compare the conditional mutual informations.

For \(Z_1=\condvar\), once we condition on \(\condvar\), the latent variable \(Z_1\) is fully determined.
Hence
\[
H(Z_1\mid \condvar)=0,
\qquad
H(Z_1\mid X,\condvar)=0,
\]
and therefore
\[
I(Z_1;X\mid \condvar)=H(Z_1\mid \condvar)-H(Z_1\mid X,\condvar)=0.
\]

For \(Z_2=\nuis\), independence of \(\nuis\) and \(\condvar\) gives
\[
H(Z_2\mid \condvar)=H(\nuis\mid \condvar)=H(\nuis)=1.
\]
Moreover, since \(Z_2=\nuis\) is determined by \(X=(\condvar,\nuis)\),
\[
H(Z_2\mid X,\condvar)=0.
\]
Thus
\[
I(Z_2;X\mid \condvar)=H(Z_2\mid \condvar)-H(Z_2\mid X,\condvar)=1.
\]

We have shown
\[
I(Z_1;X)=I(Z_2;X)=1,
\qquad
I(Z_1;X\mid \condvar)=0,\quad I(Z_2;X\mid \condvar)=1.
\]
Hence, two encoders can have the same global information budget but different within-condition information. Therefore, a regularizer depending only on \(I(Z;X)\) cannot, in general, distinguish them.
\end{proof}

\paragraph{Scope.}
The theorem establishes that a condition-aware term is necessary for separately controlling within-condition variation under the stated design goal. It does not imply that the conditional KL term is the unique possible realization of such a term; that point is addressed by Theorem~\ref{thm:condkl}.

\begin{theorem}[Product-space necessity beyond the minimal discrete example]
\label{thm:productspace}
Let \(X=(\condvar,\nuis)\), where \(\condvar\) and \(\nuis\) are independent random variables defined on
arbitrary measurable spaces. Suppose there exist two encoder channels
\(T_C(z\mid \condval)\) and \(T_N(z\mid n)\) such that
\[
I_{T_C}(Z;\condvar)=I_{T_N}(Z;\nuis)=\infoscale \in (0,\infty).
\]
Define two encoders on \(X\) by
\[
q_1(z\mid x)=T_C(z\mid \condval),
\qquad
q_2(z\mid x)=T_N(z\mid n).
\]
Then
\[
I_{q_1}(Z;X)=I_{q_2}(Z;X)=\infoscale,
\]
but
\[
I_{q_1}(Z;X\mid \condvar)=0,
\qquad
I_{q_2}(Z;X\mid \condvar)=\infoscale.
\]
Therefore, any regularizer depending only on \(I(Z;X)\) cannot, in general,
separately control within-condition information in a broad product-space setting
covering both discrete and continuous instantiations.
\end{theorem}

\begin{proof}
For \(q_1(z\mid x)=T_C(z\mid \condval)\), the encoder depends on \(x\) only through \(\condvar\).
Hence
\[
I_{q_1}(Z;X\mid \condvar)=0.
\]
Using the chain rule,
\[
I_{q_1}(Z;X)=I_{q_1}(Z;\condvar)+I_{q_1}(Z;X\mid \condvar)=I_{T_C}(Z;\condvar)=\infoscale.
\]

For \(q_2(z\mid x)=T_N(z\mid n)\), the encoder depends on \(x\) only through \(\nuis\).
Since \(\condvar\) and \(\nuis\) are independent, \(Z\) is independent of \(\condvar\) under \(q_2\).
Therefore,
\[
I_{q_2}(Z;X)=I_{q_2}(Z;\nuis)=I_{T_N}(Z;\nuis)=\infoscale.
\]
Moreover,
\[
I_{q_2}(Z;X\mid \condvar)=I_{q_2}(Z;\nuis\mid \condvar)=I_{q_2}(Z;\nuis)=\infoscale.
\]
This proves the claim.
\end{proof}

\subsection{Exact decomposition of the conditional KL term}
\label{app:thm3}

Define the condition-wise aggregated posterior
\[
\aggq(z\mid \condval)
=
\int q_\theta(z\mid x)\,p(x\mid \condval)\,dx.
\]

\begin{theorem}
\label{thm:condkl}
For any encoder \(q_\theta(z\mid x)\) and any conditional prior family \(r_\psi(z\mid \condval)\),
\[
\mathbb{E}_{p(x,\condval)}
\mathrm{KL}\!\bigl(q_\theta(z\mid x)\,\|\,r_\psi(z\mid \condval)\bigr)
=
I_\theta(Z;X\mid \condvar)
+
\mathbb{E}_{p(\condval)}
\mathrm{KL}\!\bigl(\aggq(z\mid \condval)\,\|\,r_\psi(z\mid \condval)\bigr).
\]
Consequently,
\[
I_\theta(Z;X\mid \condvar)
\le
\mathbb{E}_{p(x,\condval)}
\mathrm{KL}\!\bigl(q_\theta(z\mid x)\,\|\,r_\psi(z\mid \condval)\bigr).
\]
\end{theorem}

\begin{proof}
Start from the definition:
\[
\mathbb{E}_{p(x,\condval)}
\mathrm{KL}\!\bigl(q_\theta(z\mid x)\,\|\,r_\psi(z\mid \condval)\bigr)
=
\mathbb{E}_{p(x,\condval)}\mathbb{E}_{q_\theta(z\mid x)}
\left[
\log\frac{q_\theta(z\mid x)}{r_\psi(z\mid \condval)}
\right].
\]

Add and subtract \(\log \aggq(z\mid \condval)\):
\[
\log\frac{q_\theta(z\mid x)}{r_\psi(z\mid \condval)}
=
\log\frac{q_\theta(z\mid x)}{\aggq(z\mid \condval)}
+
\log\frac{\aggq(z\mid \condval)}{r_\psi(z\mid \condval)}.
\]

Therefore,
\begin{align*}
&\mathbb{E}_{p(x,\condval)}
\mathrm{KL}\!\bigl(q_\theta(z\mid x)\,\|\,r_\psi(z\mid \condval)\bigr) \\
&=
\mathbb{E}_{p(x,\condval)}\mathbb{E}_{q_\theta(z\mid x)}
\left[
\log\frac{q_\theta(z\mid x)}{\aggq(z\mid \condval)}
\right]
+
\mathbb{E}_{p(x,\condval)}\mathbb{E}_{q_\theta(z\mid x)}
\left[
\log\frac{\aggq(z\mid \condval)}{r_\psi(z\mid \condval)}
\right].
\end{align*}

The first term is exactly
\[
I_\theta(Z;X\mid \condvar).
\]

For the second term, once \(\condval\) is fixed, the integrand no longer depends on \(x\) except through
the averaging induced by \(q_\theta(z\mid x)\). Integrating out \(x\) conditional on \(\condval\) yields
\[
\mathbb{E}_{p(\condval)}\mathbb{E}_{\aggq(z\mid \condval)}
\left[
\log\frac{\aggq(z\mid \condval)}{r_\psi(z\mid \condval)}
\right]
=
\mathbb{E}_{p(\condval)}
\mathrm{KL}\!\bigl(\aggq(z\mid \condval)\,\|\,r_\psi(z\mid \condval)\bigr).
\]

Combining the two terms gives
\[
\mathbb{E}_{p(x,\condval)}
\mathrm{KL}\!\bigl(q_\theta(z\mid x)\,\|\,r_\psi(z\mid \condval)\bigr)
=
I_\theta(Z;X\mid \condvar)
+
\mathbb{E}_{p(\condval)}
\mathrm{KL}\!\bigl(\aggq(z\mid \condval)\,\|\,r_\psi(z\mid \condval)\bigr).
\]

Finally, KL divergence is nonnegative, so
\[
\mathbb{E}_{p(\condval)}
\mathrm{KL}\!\bigl(\aggq(z\mid \condval)\,\|\,r_\psi(z\mid \condval)\bigr)\ge 0,
\]
which implies
\[
I_\theta(Z;X\mid \condvar)
\le
\mathbb{E}_{p(x,\condval)}
\mathrm{KL}\!\bigl(q_\theta(z\mid x)\,\|\,r_\psi(z\mid \condval)\bigr).
\]
\end{proof}

\paragraph{Remark.}
This is the key closure property of the method:
\[
\text{conditional KL}
=
\text{within-condition information}
+
\text{prior mismatch}.
\]
Thus the conditional KL term is not merely heuristic; it is an exact decomposition plus a tractable upper bound on \(I(Z;X\mid \condvar)\).

\subsection{Nuisance interpretation}
\label{app:nuisance}

We now add an optional structural interpretation.
Assume there exist latent variables \(S\) and \(\nuis\) such that
\[
X=T(S,\nuis),\qquad Y=f(S),\qquad \condvar=\partmap(Y).
\]
Here \(S\) denotes task-relevant signal and \(\nuis\) denotes within-condition nuisance variation.

This structural reading uses the same encoder-induced convention stated in
Appendix~\ref{app:notation}. In particular, whenever \(\nuis\) affects the
representation only through the observed input \(X\) after conditioning on
\(\condvar\), the encoder \(q_\theta(z\mid x)\) induces the conditional Markov
chain \(\nuis\to X\to Z\) given \(\condvar\).

\begin{proposition}
\label{prop:nuisance}
Suppose that, conditioned on \(\condvar\), the encoder-induced Markov chain
\[
\nuis \to X \to Z
\]
holds. Then
\[
I(Z;\nuis\mid \condvar)\le I(Z;X\mid \condvar).
\]
Combining this with Theorem~\ref{thm:condkl} in Appendix~\ref{app:thm3} yields
\[
I(Z;\nuis\mid \condvar)
\le
I(Z;X\mid \condvar)
\le
\mathbb{E}_{p(x,\condval)}
\mathrm{KL}\!\bigl(q_\theta(z\mid x)\,\|\,r_\psi(z\mid \condval)\bigr).
\]
\end{proposition}

\begin{proof}
Fix any value \(\condval\) of \(\condvar\).
Under the conditional Markov chain
\[
\nuis \to X \to Z \quad \text{given } \condvar=\condval,
\]
the conditional data processing inequality gives
\[
I(Z;\nuis\mid \condvar=\condval)\le I(Z;X\mid \condvar=\condval).
\]
Taking expectation over \(\condval\) yields
\[
I(Z;\nuis\mid \condvar)\le I(Z;X\mid \condvar).
\]
The second inequality follows directly from Theorem~\ref{thm:condkl}.
\end{proof}

\paragraph{Scope.}
The proposition does not identify \(I(Z;X\mid \condvar)\) with nuisance information in full generality. Under the stated structural assumption, it yields an upper bound on nuisance dependence through \(I(Z;X\mid \condvar)\).

\subsection{Exact decomposition of the standard KL term}
\label{app:stdkl}

The standard VIB term also admits an exact decomposition, completely parallel to the conditional case.

\begin{lemma}
\label{lem:stdkl}
Let
\[
\aggq(z)=\int q_\theta(z\mid x)p(x)\,dx
\]
be the aggregated posterior. Then for any unconditional prior \(r_0(z)\),
\[
\mathbb{E}_{p(x)}
\mathrm{KL}\!\bigl(q_\theta(z\mid x)\,\|\,r_0(z)\bigr)
=
I_\theta(Z;X)
+
\mathrm{KL}\!\bigl(\aggq(z)\,\|\,r_0(z)\bigr).
\]
Consequently,
\[
I_\theta(Z;X)\le
\mathbb{E}_{p(x)}
\mathrm{KL}\!\bigl(q_\theta(z\mid x)\,\|\,r_0(z)\bigr).
\]
\end{lemma}

\begin{proof}
We expand:
\[                           
\mathbb{E}_{p(x)}
\mathrm{KL}\!\bigl(q_\theta(z\mid x)\,\|\,r_0(z)\bigr)
=
\mathbb{E}_{p(x)}\mathbb{E}_{q_\theta(z\mid x)}
\left[
\log\frac{q_\theta(z\mid x)}{r_0(z)}
\right].
\]
Add and subtract \(\log \aggq(z)\):
\[
\log\frac{q_\theta(z\mid x)}{r_0(z)}
=
\log\frac{q_\theta(z\mid x)}{\aggq(z)}
+
\log\frac{\aggq(z)}{r_0(z)}.
\]
Therefore,
\begin{align*}
&\mathbb{E}_{p(x)}
\mathrm{KL}\!\bigl(q_\theta(z\mid x)\,\|\,r_0(z)\bigr) \\
&=
\mathbb{E}_{p(x)}\mathbb{E}_{q_\theta(z\mid x)}
\left[
\log\frac{q_\theta(z\mid x)}{\aggq(z)}
\right]
+
\mathbb{E}_{p(x)}\mathbb{E}_{q_\theta(z\mid x)}
\left[
\log\frac{\aggq(z)}{r_0(z)}
\right].
\end{align*}
The first term is
\[
I_\theta(Z;X).
\]
The second term becomes, after integrating out \(x\),
\[
\mathbb{E}_{\aggq(z)}
\left[
\log\frac{\aggq(z)}{r_0(z)}
\right]
=
\mathrm{KL}\!\bigl(\aggq(z)\,\|\,r_0(z)\bigr).
\]
Combining the two terms proves the identity. The inequality follows from the non-negativity of KL.

\end{proof}

\paragraph{Interpretation.}
The standard KL term controls global bandwidth:
\[
\text{standard KL}
=
\text{global information}
+
\text{global prior mismatch}.
\]
This makes its role fully parallel to the conditional KL term.

\begin{proposition}[Tightness under exact prior calibration]
\label{prop:tight_calibration}
Let
\[
\aggq(z)=\int q_\theta(z\mid x)p(x)\,dx
\]
and
\[
\aggq(z\mid \condval)=\int q_\theta(z\mid x)p(x\mid \condval)\,dx.
\]

Then the following hold.

\begin{enumerate}
    \item If the unconditional prior is chosen as
    \[
    r_0(z)=\aggq(z),
    \]
    then
    \[
    \mathbb{E}_{p(x)}
    \mathrm{KL}\!\bigl(q_\theta(z\mid x)\,\|\,r_0(z)\bigr)
    =
    I_\theta(Z;X).
    \]

    \item If the conditional prior is chosen as
    \[
    r_\psi(z\mid \condval)=\aggq(z\mid \condval)
    \qquad\text{for } p(\condval)\text{-a.e. } \condval,
    \]
    then
    \[
    \mathbb{E}_{p(x,\condval)}
    \mathrm{KL}\!\bigl(q_\theta(z\mid x)\,\|\,r_\psi(z\mid \condval)\bigr)
    =
    I_\theta(Z;X\mid \condvar).
    \]
\end{enumerate}
\end{proposition}

\begin{proof}
The first claim follows directly from Lemma~\ref{lem:stdkl}:
\[
\mathbb{E}_{p(x)}
\mathrm{KL}\!\bigl(q_\theta(z\mid x)\,\|\,r_0(z)\bigr)
=
I_\theta(Z;X)
+
\mathrm{KL}\!\bigl(\aggq(z)\,\|\,r_0(z)\bigr).
\]
Setting \(r_0(z)=\aggq(z)\) makes the mismatch term vanish.

The second claim follows directly from Theorem~\ref{thm:condkl}:
\[
\mathbb{E}_{p(x,\condval)}
\mathrm{KL}\!\bigl(q_\theta(z\mid x)\,\|\,r_\psi(z\mid \condval)\bigr)
=
I_\theta(Z;X\mid \condvar)
+
\mathbb{E}_{p(\condval)}
\mathrm{KL}\!\bigl(\aggq(z\mid \condval)\,\|\,r_\psi(z\mid \condval)\bigr).
\]
Setting \(r_\psi(z\mid \condval)=\aggq(z\mid \condval)\) for \(p(\condval)\)-a.e.\ \(\condval\)
makes the second mismatch term vanish.
\end{proof}

\paragraph{Scope.}
Proposition~\ref{prop:tight_calibration} identifies the ideal calibration regime under which the KL surrogates coincide with the corresponding information quantities. In the main method, the structured Gaussian/simplex prior is not introduced to exactly realize the aggregated posterior, but to provide a stable and interpretable parameterized prior family for regularization and initialization.

\begin{theorem}[Insufficiency of a purely conditional bottleneck]
\label{thm:conditionalinsufficient}
There exist two encoders \(q_1(z\mid x)\) and \(q_2(z\mid x)\) such that
\[
I_{q_1}(Z;X\mid \condvar)=I_{q_2}(Z;X\mid \condvar),
\qquad
I_{q_1}(Z;X)\neq I_{q_2}(Z;X).
\]
Therefore, any regularizer depending only on \(I(Z;X\mid \condvar)\) cannot, in general,
control the total global information capacity.
\end{theorem}

\begin{proof}
Let \(X=(\condvar,\nuis)\), where \(\condvar\) and \(\nuis\) are independent and \(H(\condvar)>0\).
Consider two encoders:
\[
Z_1 \equiv z_0,
\qquad
Z_2 = \condvar.
\]
For \(Z_1\),
\[
I(Z_1;X)=0,
\qquad
I(Z_1;X\mid \condvar)=0.
\]
For \(Z_2=\condvar\), conditioning on \(\condvar\) fully determines \(Z_2\), so
\[
I(Z_2;X\mid \condvar)=0.
\]
However,
\[
I(Z_2;X)=I(\condvar;X)=H(\condvar)>0.
\]
Thus the two encoders have the same conditional information but different global information.
Hence a regularizer depending only on \(I(Z;X\mid \condvar)\) cannot in general control total global capacity.
\end{proof}

\begin{theorem}[Product-space insufficiency of a purely conditional bottleneck]
\label{thm:conditionalinsufficient_product}
Let \(X=(\condvar,\nuis)\), where \(\condvar\) and \(\nuis\) are independent random variables on arbitrary measurable spaces.
Suppose there exists an encoder channel \(T_C(z\mid \condval)\) such that
\[
0 < I_{T_C}(Z;\condvar) < \infty.
\]
Define two encoders
\[
q_1(z\mid x)=\delta_{z_0}(z), \qquad q_2(z\mid x)=T_C(z\mid \condval).
\]
Then
\[
I_{q_1}(Z;X\mid \condvar)=I_{q_2}(Z;X\mid \condvar)=0,
\]
but
\[
I_{q_1}(Z;X)=0, \qquad I_{q_2}(Z;X)=I_{T_C}(Z;\condvar)>0.
\]
Therefore, any regularizer depending only on \(I(Z;X\mid \condvar)\) cannot, in general,
control total global information capacity in a broad product-space setting.
\end{theorem}

\begin{proof}
For \(q_1\), the latent variable is deterministic: \(Z\equiv z_0\). Hence
\[
I_{q_1}(Z;X\mid \condvar)=0,
\qquad
I_{q_1}(Z;X)=0.
\]

For \(q_2\), the encoder depends on \(x\) only through \(\condvar\). Therefore,
\[
I_{q_2}(Z;X\mid \condvar)=0.
\]
Using the chain rule,
\[
I_{q_2}(Z;X)=I_{q_2}(Z;\condvar)+I_{q_2}(Z;X\mid \condvar)
=I_{T_C}(Z;\condvar)>0.
\]

Thus \(q_1\) and \(q_2\) have the same conditional information but different global information.
Hence any regularizer depending only on \(I(Z;X\mid \condvar)\) cannot, in general,
control total global information capacity.
\end{proof}

\subsection{A structural regime where condition-level encoding is invisible to the conditional term}
\label{subsec:invisible_condition_level_encoding}

We now introduce a second optional structural regime, separate from the within-condition nuisance
interpretation in Proposition~1. The purpose of this subsection is to show that, even when the
conditional KL term correctly regularizes the within-condition information that it can see, there may
still exist condition-level encoding that does not appear in $I(Z;X\mid C)$ at all. In such a case,
a purely conditional bottleneck is insufficient, and the standard KL term remains necessary as a
global-capacity regularizer.

Assume that the input admits a decomposition
\[
X = T(S,\nuis,B), \qquad Y=f(S), \qquad C=\partmap(Y),
\]
where:
(i) $\nuis$ denotes within-condition nuisance variation,
and (ii) $B$ denotes a class-predictable / condition-level component.
We keep the visible within-condition assumption from Proposition~1 for $\nuis$, and now add the
following stronger condition for the invisible condition-level component:
\[
H(B\mid C)=0, \qquad H(B)>0.
\]
Thus, $B$ is a non-trivial measurable function of $C$.

\begin{proposition}[Invisible condition-level encoding under $X=T(S,\nuis,B)$]
\label{prop:invisible_condition_level_encoding}
Assume
\[
X = T(S,\nuis,B), \qquad Y=f(S), \qquad C=\partmap(Y),
\]
and suppose
\[
H(B\mid C)=0, \qquad H(B)>0.
\]
Then there exist a stochastic encoder $q_\theta(z\mid x)$ and a conditional prior $r_\psi(z\mid c)$
such that:
\begin{align}
\mathbb{E}_{p(x,c)}
\operatorname{KL}\!\bigl(q_\theta(z\mid x)\,\|\,r_\psi(z\mid c)\bigr) &= 0, \label{eq:inv-cond-kl-zero}\\
I_\theta(Z;X\mid C) &= 0, \label{eq:inv-ixc-zero}\\
I_\theta(Z;C) &> 0, \label{eq:inv-izc-pos}\\
I_\theta(Z;X) &= I_\theta(Z;C) > 0. \label{eq:inv-ix-eq-izc}
\end{align}
Consequently, there exists a valid structural regime in which some condition-level encoding is
completely invisible to the conditional term, yet still occupies positive global information capacity.
\end{proposition}

\begin{proof}
Since $H(B\mid C)=0$, there exists a measurable map $m(\cdot)$ such that
\[
B=m(C) \qquad \text{almost surely.}
\]

Choose any non-degenerate stochastic channel $\kappa_{Z\mid B}(z\mid b)$ satisfying
\[
0 < I_{\kappa_{Z\mid B}}(Z;B) < \infty.
\]
Now, define the encoder by
\[
q_\theta(z\mid x) := \kappa_{Z\mid B}\!\bigl(z \mid m(c_{\partmap}(x))\bigr),
\]
where $c_{\partmap}(x):=\partmap(\cleanrule(x))$.
That is, the encoder depends on $x$ only through the induced condition $C=c_{\partmap}(x)$ (equivalently,
through $B=m(C)$). Define the conditional prior by
\[
r_\psi(z\mid c) := \kappa_{Z\mid B}(z\mid m(c)).
\]

Then for every pair $(x,c)$ in the support of $p(x,c)$ with $c=c_{\partmap}(x)$, we have
\[
q_\theta(z\mid x)=r_\psi(z\mid c).
\]
Therefore, the pointwise KL divergence is zero:
\[
\operatorname{KL}\!\bigl(q_\theta(z\mid x)\,\|\,r_\psi(z\mid c)\bigr)=0.
\]
Taking expectation over $(x,c)$ yields
\[
\mathbb{E}_{p(x,c)}
\operatorname{KL}\!\bigl(q_\theta(z\mid x)\,\|\,r_\psi(z\mid c)\bigr)=0,
\]
which proves \eqref{eq:inv-cond-kl-zero}.

Next, because $q_\theta(z\mid x)=q_\theta(z\mid c_{\partmap}(x))$, the conditional distribution of $Z$
depends only on $C$ and not on any residual variation of $X$ once $C$ is fixed. Hence
\[
Z \perp X \mid C,
\]
and therefore
\[
I_\theta(Z;X\mid C)=0,
\]
which proves \eqref{eq:inv-ixc-zero}.

Now observe that $B=m(C)$ and $Z$ is generated from $B$ through the non-degenerate channel $\kappa_{Z\mid B}$.
Hence
\[
C \to B \to Z.
\]
Moreover, because $B$ is a deterministic function of $C$, we have
\[
I_\theta(Z;B\mid C)=0.
\]
And since $Z$ is generated from $B$, we also have
\[
I_\theta(Z;C\mid B)=0.
\]
Applying the chain rule to the joint pair $(B,C)$ in two ways gives
\[
I_\theta(Z;B,C)
= I_\theta(Z;C)+I_\theta(Z;B\mid C)
= I_\theta(Z;C),
\]
and also
\[
I_\theta(Z;B,C)
= I_\theta(Z;B)+I_\theta(Z;C\mid B)
= I_\theta(Z;B).
\]
Thus
\[
I_\theta(Z;C)=I_\theta(Z;B).
\]
Since the channel $\kappa_{Z\mid B}$ was chosen to be non-degenerate with $I_{\kappa_{Z\mid B}}(Z;B)>0$, it follows that
\[
I_\theta(Z;C)>0,
\]
which proves \eqref{eq:inv-izc-pos}.

Finally, by Theorem~\ref{thm:decomp},
\[
I_\theta(Z;X)=I_\theta(Z;\condvar)+I_\theta(Z;X\mid \condvar).
\]
Substituting \eqref{eq:inv-ixc-zero} gives
\[
I_\theta(Z;X)=I_\theta(Z;C).
\]
Combining this with \eqref{eq:inv-izc-pos} yields
\[
I_\theta(Z;X)=I_\theta(Z;C)>0,
\]
which proves \eqref{eq:inv-ix-eq-izc}.
\end{proof}

\paragraph{Interpretation.}
This proposition is an existence result, not a universal statement about all datasets or all nuisance
decompositions. Its precise meaning is narrower and sufficient for our purpose:
the conditional KL term regularizes the within-condition information that is visible through
$I(Z;X\mid C)$, but there exist valid structural regimes in which some class-predictable,
condition-level encoding does not appear in $I(Z;X\mid C)$ at all. In such cases, the remaining
encoding lives entirely in the global / condition-level part $I(Z;C)$, and therefore still contributes
to the total information budget $I(Z;X)$. By Lemma~1, the standard KL term upper-bounds
$I_\theta(Z;X)$ up to a global prior mismatch, so it remains the mechanism that regularizes this
otherwise invisible condition-level capacity.

\paragraph{Connection to the dual-bottleneck objective.}
Together with Proposition~1 and Lemma~1, the picture is now explicit:
\begin{itemize}
    \item Proposition~1 explains why the conditional KL term regularizes visible within-condition
    nuisance dependence through $I(Z;X\mid C)$.
    \item Proposition~\ref{prop:invisible_condition_level_encoding} shows that a purely conditional
    bottleneck cannot, in general, exclude all condition-level encoding, because some of it may be
    invisible to $I(Z;X\mid C)$.
    \item Lemma~1 explains why the standard KL term remains necessary: it controls the total
    global information bandwidth, including condition-level encoding that does not show up in the
    conditional term.
\end{itemize}
Hence, the dual-bottleneck objective is justified not only by the visible within-condition role of the
conditional term, but also by the existence of hidden condition-level capacity that requires a
global regularizer.

\subsection{A structured conditional prior and its induced geometry}
\label{app:simplex}

We choose
\[
r_\psi(z\mid \condval)=\mathcal{N}(\mu_{\condval},\tau^2 \Id_d),
\qquad
\mu_{\condval}=\rho \embedmat s_{\condval}+\priorbias,
\]
where:
\begin{itemize}
    \item \(s_1,\dots,s_\ncond\in\mathbb{R}^{\ncond-1}\) are regular simplex vertices,
    \item \(\embedmat^\top \embedmat=\Id_{\ncond-1}\),
    \item \(\rho>0\), \(\priorbias\in\mathbb{R}^d\), and \(\tau>0\).
\end{itemize}
For a regular simplex,
\[
\|s_{\condval}\|=1,
\qquad
s_{\condval}^\top s_{\condval'}=-\frac{1}{\ncond-1}\quad (\condval\neq \condval').
\]

\paragraph{Dimensional assumption for the simplex construction.}
Throughout the simplex-prior analysis, we assume
\[
s_{\condval} \in \mathbb{R}^{\ncond-1}, \qquad
\embedmat \in \mathbb{R}^{d\times(\ncond-1)}, \qquad
\embedmat^\top \embedmat = \Id_{\ncond-1},
\]
which requires
\[
d \ge \ncond-1.
\]

\begin{proposition}
\label{prop:simplex}
For any \(\condval\neq \condval'\),
\[
\|\mu_{\condval}-\mu_{\condval'}\|^2
=
2\rho^2\frac{\ncond}{\ncond-1}.
\]
Hence all condition centers are equally separated.
\end{proposition}

\begin{proof}
By definition,
\[
\mu_{\condval}-\mu_{\condval'}=\rho \embedmat(s_{\condval}-s_{\condval'}).
\]
Therefore,
\[
\|\mu_{\condval}-\mu_{\condval'}\|^2
=
\rho^2\|\embedmat(s_{\condval}-s_{\condval'})\|^2.
\]
Since \(\embedmat^\top\embedmat=\Id_{\ncond-1}\), the linear map \(\embedmat\) preserves Euclidean norm on the simplex subspace, so
\[
\|\embedmat(s_{\condval}-s_{\condval'})\|^2
=
\|s_{\condval}-s_{\condval'}\|^2.
\]
Now expand:
\[
\|s_{\condval}-s_{\condval'}\|^2
=
\|s_{\condval}\|^2+\|s_{\condval'}\|^2-2s_{\condval}^\top s_{\condval'}
=
1+1-2\left(-\frac{1}{\ncond-1}\right)
=
2\frac{\ncond}{\ncond-1}.
\]
Hence
\[
\|\mu_{\condval}-\mu_{\condval'}\|^2
=
2\rho^2\frac{\ncond}{\ncond-1}.
\]
\end{proof}
\paragraph{Interpretation.}
This proposition shows that the simplex prior induces a symmetric separability pattern among the condition centers. In our method, this property justifies using simplex geometry as a constrained prior parameterization, offering a balanced and analytically transparent prior for conditional regularization.

\subsection{A corrected minimization argument for the shared variance}
\label{app:tau}

\begin{proposition}
\label{prop:tau}
Suppose the encoder family \(q_\theta(z\mid x)=\mathcal{N}(m_x,\Sigma_x)\), the centers \(\{\mu_c\}\),
and the data distribution \(p(x,c)\) are temporarily fixed.
Define
\[
A:=
\mathbb{E}_{p(x,c)}
\bigl[
\mathrm{tr}(\Sigma_x)+\|m_x-\mu_c\|^2
\bigr].
\]
Let
\[
J(\tau^2)
:=
\mathbb{E}_{p(x,c)}
\mathrm{KL}\!\left(
\mathcal{N}(m_x,\Sigma_x)\,\|\,\mathcal{N}(\mu_c,\tau^2 \Id_d)
\right).
\]
Then
\[
J(\tau^2)
=
\frac12\left(\frac{A}{\tau^2}+d\log\tau^2\right)+\mathrm{const},
\]
where \(\mathrm{const}\) does not depend on \(\tau\).
Furthermore, \(A\ge 0\). In the nondegenerate case \(A>0\), the unique minimizer is
\[
\tau_\star^2=\frac{A}{d}.
\]
\end{proposition}

\paragraph{Remark.}
The quantity \(A\) is nonnegative because both
\(\operatorname{tr}(\Sigma_x)\) and \(\|m_x-\mu_c\|^2\) are nonnegative.
Moreover, \(A=0\) can occur only in the degenerate case where
\[
\Sigma_x=0
\quad\text{and}\quad
m_x=\mu_c
\qquad\text{almost surely},
\]
namely, when the posterior collapses to condition-wise point masses.
Hence, under any nondegenerate Gaussian posterior family, the condition \(A>0\) is natural.

\begin{proof}
For Gaussian distributions,
\[
\mathrm{KL}\!\left(
\mathcal{N}(m_x,\Sigma_x)\,\|\,\mathcal{N}(\mu_c,\tau^2 \Id_d)
\right)
=
\frac12\Bigl[
\tau^{-2}\mathrm{tr}(\Sigma_x)
+\tau^{-2}\|m_x-\mu_c\|^2
-d
+d\log\tau^2
-\log\det\Sigma_x
\Bigr].
\]
Taking expectation over \(p(x,c)\), we obtain
\[
J(\tau^2)
=
\frac12\left(
\frac{A}{\tau^2}
+
d\log\tau^2
\right)
+
\frac12\,\mathbb{E}_{p(x,c)}[-d-\log\det\Sigma_x].
\]
The second term is independent of \(\tau\), so we may write
\[
J(\tau^2)
=
\frac12\left(
\frac{A}{\tau^2}
+
d\log\tau^2
\right)
+\mathrm{const}.
\]

Now let \(u=\tau^2>0\). Then
\[
J(u)=\frac12\left(\frac{A}{u}+d\log u\right)+\mathrm{const}.
\]
Differentiate with respect to \(u\):
\[
J'(u)
=
\frac12\left(-\frac{A}{u^2}+\frac{d}{u}\right)
=
\frac{du-A}{2u^2}.
\]
If \(A=0\), then
\[
J'(u)=\frac{d}{2u}>0
\qquad\text{for all } u>0,
\]
so \(J\) is strictly increasing on \((0,\infty)\) and has no interior minimizer there;
its infimum is approached only in the boundary limit \(u\to 0^+\). Now consider the
nondegenerate case \(A>0\).
Since \(2u^2>0\) for all \(u>0\), the sign of \(J'(u)\) is determined entirely by \(du-A\):
\[
J'(u)<0 \quad \text{if } u<\frac{A}{d},
\qquad
J'(u)>0 \quad \text{if } u>\frac{A}{d}.
\]
Therefore \(u=A/d\) is the unique stationary point, and the derivative changes from negative
to positive at that point. Hence, it is the unique global minimizer:
\[
u_\star=\frac{A}{d}.
\]
Returning to \(u=\tau^2\), we obtain
\[
\tau_\star^2=\frac{A}{d}.
\]
\end{proof}

\paragraph{Interpretation.}
This proposition clarifies the role of the shared-scale parameterization. When one adopts a common isotropic variance across conditions, learning \(\tau\) mainly adjusts a global average dispersion level. This makes fixed \(\tau\) a practical and interpretable choice for keeping center geometry and variance adaptation disentangled.

\subsection{Task-constrained Pareto synthesis theorem}
\label{app:final-logic}

We now synthesize the preceding results into a single task-constrained Pareto statement.
The purpose of this subsection is not to introduce a new primitive ingredient,
but to collect the earlier decomposition, upper-bound, and insufficiency results into
one sharpened theorem that formalizes the control advantage of the dual objective.

\paragraph{Task-feasible set.}
Define the supervised task risk
\[
R_{\mathrm{task}}(\theta,\omega)
=
\mathbb{E}_{p(x,y)q_\theta(z\mid x)}
\bigl[\ell_{\mathrm{task}}(g_\omega(z),y)\bigr],
\]
and for any tolerance \(\varepsilon>0\), define the task-feasible set
\[
\mathcal{F}_\varepsilon
=
\bigl\{
(\theta,\omega,\psi):
R_{\mathrm{task}}(\theta,\omega)\le \varepsilon
\bigr\}.
\]

\paragraph{Two residual-information axes.}
Under the conditional nuisance regime of Proposition~\ref{prop:nuisance}, define the
visible nuisance dependence
\[
H_{\mathrm{vis}}(\theta)
=
I_\theta(Z;\nuis\mid \condvar).
\]
Under the optional refinement regime of
Proposition~\ref{prop:invisible_condition_level_encoding}, define the
conditional-invisible condition-level encoding
\[
H_{\mathrm{inv}}(\theta)
=
I_\theta(Z;B).
\]
For any positive weights \(\lambda_{\mathrm{vis}},\lambda_{\mathrm{inv}}>0\), let
\[
H_\lambda(\theta)
=
\lambda_{\mathrm{vis}}H_{\mathrm{vis}}(\theta)
+
\lambda_{\mathrm{inv}}H_{\mathrm{inv}}(\theta).
\]

\paragraph{Provable upper-bound surrogate.}
Define
\[
J_{\mathrm{std}}(\theta)
=
\mathbb{E}_{p(x)}
\mathrm{KL}\!\bigl(q_\theta(z\mid x)\,\|\,r_0(z)\bigr),
\qquad
J_{\mathrm{cond}}(\theta,\psi)
=
\mathbb{E}_{p(x,\condval)}
\mathrm{KL}\!\bigl(q_\theta(z\mid x)\,\|\,r_\psi(z\mid \condval)\bigr),
\]
and
\[
U_\lambda(\theta,\psi)
=
\lambda_{\mathrm{vis}}J_{\mathrm{cond}}(\theta,\psi)
+
\lambda_{\mathrm{inv}}J_{\mathrm{std}}(\theta).
\]

\paragraph{Supported families.}
Define the dual supported family
\[
\mathcal{A}^{\mathrm{dual}}_\varepsilon
=
\bigcup_{\substack{\beta_{\mathrm{std}},\beta_{\mathrm{cond}}\ge 0\\
\beta_{\mathrm{std}}+\beta_{\mathrm{cond}}>0}}
\operatorname*{arg\,min}_{(\theta,\omega,\psi)\in \mathcal{F}_\varepsilon}
\Bigl[
\beta_{\mathrm{std}}J_{\mathrm{std}}(\theta)
+
\beta_{\mathrm{cond}}J_{\mathrm{cond}}(\theta,\psi)
\Bigr].
\]
Likewise, define the two degenerate single-term subfamilies
\[
\mathcal{A}^{\mathrm{std}}_\varepsilon
=
\bigcup_{\beta_{\mathrm{std}}>0}
\operatorname*{arg\,min}_{(\theta,\omega,\psi)\in \mathcal{F}_\varepsilon}
\Bigl[
\beta_{\mathrm{std}}J_{\mathrm{std}}(\theta)
\Bigr],
\]
\[
\mathcal{A}^{\mathrm{cond}}_\varepsilon
=
\bigcup_{\beta_{\mathrm{cond}}>0}
\operatorname*{arg\,min}_{(\theta,\omega,\psi)\in \mathcal{F}_\varepsilon}
\Bigl[
\beta_{\mathrm{cond}}J_{\mathrm{cond}}(\theta,\psi)
\Bigr].
\]
Assume these supported families are nonempty for the weight ranges considered.
(If minimizers are not attained, the same comparison may be stated at the level of
infimizing sequences.)

\begin{theorem}[Task-constrained Pareto upper-bound dominance of dual regularization]
\label{thm:pareto_upper_bound}
Under the standing assumptions and encoder-induced Markov convention in
Appendix~\ref{app:notation}, together with the nuisance assumption in
Appendix~\ref{app:nuisance} (Proposition~\ref{prop:nuisance}) and the structural
refinement regime in Appendix~\ref{subsec:invisible_condition_level_encoding}
(Proposition~\ref{prop:invisible_condition_level_encoding}), the following hold.

\begin{enumerate}
    \item For every \((\theta,\omega,\psi)\in \mathcal{F}_\varepsilon\),
    \begin{equation}
    \label{eq:pareto_pointwise_bound}
    H_\lambda(\theta)\le U_\lambda(\theta,\psi).
    \end{equation}

    \item Define the optimal provable upper bounds on residual information
    \[
    \Phi_{\mathrm{dual}}^\star(\varepsilon;\lambda)
    =
    \inf_{(\theta,\omega,\psi)\in\mathcal{A}^{\mathrm{dual}}_\varepsilon}
    U_\lambda(\theta,\psi),
    \]
    \[
    \Phi_{\mathrm{std}}^\star(\varepsilon;\lambda)
    =
    \inf_{(\theta,\omega,\psi)\in\mathcal{A}^{\mathrm{std}}_\varepsilon}
    U_\lambda(\theta,\psi),
    \qquad
    \Phi_{\mathrm{cond}}^\star(\varepsilon;\lambda)
    =
    \inf_{(\theta,\omega,\psi)\in\mathcal{A}^{\mathrm{cond}}_\varepsilon}
    U_\lambda(\theta,\psi).
    \]
    Then
    \begin{equation}
    \label{eq:pareto_value_compare}
    \Phi_{\mathrm{dual}}^\star(\varepsilon;\lambda)
    \le
    \Phi_{\mathrm{std}}^\star(\varepsilon;\lambda),
    \qquad
    \Phi_{\mathrm{dual}}^\star(\varepsilon;\lambda)
    \le
    \Phi_{\mathrm{cond}}^\star(\varepsilon;\lambda).
    \end{equation}
\end{enumerate}
\end{theorem}

\begin{proof}
We prove the two claims in order.

\paragraph{Step 1: visible nuisance dependence is upper-bounded by the conditional KL term.}
By Proposition~\ref{prop:nuisance}, under the conditional Markov assumption
\[
\nuis \to X \to Z
\qquad\text{given } \condvar,
\]
we have
\begin{equation}
\label{eq:pareto_step_vis_1}
I_\theta(Z;\nuis\mid \condvar)
\le
I_\theta(Z;X\mid \condvar).
\end{equation}
By Theorem~\ref{thm:condkl},
\begin{equation}
\label{eq:pareto_step_vis_2}
I_\theta(Z;X\mid \condvar)
\le
\mathbb{E}_{p(x,\condval)}
\mathrm{KL}\!\bigl(q_\theta(z\mid x)\,\|\,r_\psi(z\mid \condval)\bigr)
=
J_{\mathrm{cond}}(\theta,\psi).
\end{equation}
Combining \eqref{eq:pareto_step_vis_1} and \eqref{eq:pareto_step_vis_2}, we obtain
\begin{equation}
\label{eq:pareto_step_vis_final}
H_{\mathrm{vis}}(\theta)
=
I_\theta(Z;\nuis\mid \condvar)
\le
J_{\mathrm{cond}}(\theta,\psi).
\end{equation}

\paragraph{Step 2: conditional-invisible condition-level encoding is upper-bounded by the standard KL term.}
Under the structural regime of Appendix~\ref{subsec:invisible_condition_level_encoding}
(Proposition~\ref{prop:invisible_condition_level_encoding}), the condition
\[
H(B\mid \condvar)=0
\]
specifies that \(B\) is a condition-level component determined by \(\condvar\).
This condition explains why such a component can be invisible to the conditional
axis \(I(Z;X\mid\condvar)\), but it is not needed to establish the data-processing
step. By the encoder-induced Markov convention in Appendix~\ref{app:notation},
the stochastic encoder \(q_\theta(z\mid x)\) gives
\[
B \to X \to Z.
\]
Therefore, by the data processing inequality,
\begin{equation}
\label{eq:pareto_step_inv_1}
I_\theta(Z;B)\le I_\theta(Z;X).
\end{equation}
By Lemma~\ref{lem:stdkl},
\begin{equation}
\label{eq:pareto_step_inv_2}
I_\theta(Z;X)
\le
\mathbb{E}_{p(x)}
\mathrm{KL}\!\bigl(q_\theta(z\mid x)\,\|\,r_0(z)\bigr)
=
J_{\mathrm{std}}(\theta).
\end{equation}
Combining \eqref{eq:pareto_step_inv_1} and \eqref{eq:pareto_step_inv_2}, we obtain
\begin{equation}
\label{eq:pareto_step_inv_final}
H_{\mathrm{inv}}(\theta)
=
I_\theta(Z;B)
\le
J_{\mathrm{std}}(\theta).
\end{equation}

\paragraph{Step 3: weighted combination of the two axes.}
Multiply \eqref{eq:pareto_step_vis_final} by \(\lambda_{\mathrm{vis}}>0\) and
\eqref{eq:pareto_step_inv_final} by \(\lambda_{\mathrm{inv}}>0\), then add the two inequalities:
\begin{align}
\lambda_{\mathrm{vis}}H_{\mathrm{vis}}(\theta)
+
\lambda_{\mathrm{inv}}H_{\mathrm{inv}}(\theta)
&\le
\lambda_{\mathrm{vis}}J_{\mathrm{cond}}(\theta,\psi)
+
\lambda_{\mathrm{inv}}J_{\mathrm{std}}(\theta)
\nonumber\\
&=
U_\lambda(\theta,\psi).
\label{eq:pareto_weighted_bound}
\end{align}
By the definition of \(H_\lambda(\theta)\), this is exactly
\[
H_\lambda(\theta)\le U_\lambda(\theta,\psi),
\]
which proves \eqref{eq:pareto_pointwise_bound}.

\paragraph{Step 4: family inclusions.}
By construction, the std-only family is obtained from the dual scalarization by setting
\(\beta_{\mathrm{cond}}=0\), and the cond-only family is obtained by setting
\(\beta_{\mathrm{std}}=0\). Therefore,
\begin{equation}
\label{eq:pareto_family_inclusion}
\mathcal{A}^{\mathrm{std}}_\varepsilon
\subseteq
\mathcal{A}^{\mathrm{dual}}_\varepsilon,
\qquad
\mathcal{A}^{\mathrm{cond}}_\varepsilon
\subseteq
\mathcal{A}^{\mathrm{dual}}_\varepsilon.
\end{equation}

\paragraph{Step 5: infimum comparison over nested families.}
A basic property of the infimum is the following:
if \(A\subseteq D\) and \(f\) is any extended-real-valued function, then
\begin{equation}
\label{eq:pareto_inf_property}
\inf_{x\in D} f(x)\le \inf_{x\in A} f(x).
\end{equation}
Apply \eqref{eq:pareto_inf_property} first with
\[
D=\mathcal{A}^{\mathrm{dual}}_\varepsilon,
\qquad
A=\mathcal{A}^{\mathrm{std}}_\varepsilon,
\qquad
f(\theta,\omega,\psi)=U_\lambda(\theta,\psi).
\]
Using \eqref{eq:pareto_family_inclusion}, we obtain
\[
\inf_{(\theta,\omega,\psi)\in\mathcal{A}^{\mathrm{dual}}_\varepsilon}
U_\lambda(\theta,\psi)
\le
\inf_{(\theta,\omega,\psi)\in\mathcal{A}^{\mathrm{std}}_\varepsilon}
U_\lambda(\theta,\psi),
\]
that is,
\[
\Phi_{\mathrm{dual}}^\star(\varepsilon;\lambda)
\le
\Phi_{\mathrm{std}}^\star(\varepsilon;\lambda).
\]

Apply \eqref{eq:pareto_inf_property} again with
\[
D=\mathcal{A}^{\mathrm{dual}}_\varepsilon,
\qquad
A=\mathcal{A}^{\mathrm{cond}}_\varepsilon,
\qquad
f(\theta,\omega,\psi)=U_\lambda(\theta,\psi).
\]
Using \eqref{eq:pareto_family_inclusion}, we obtain
\[
\inf_{(\theta,\omega,\psi)\in\mathcal{A}^{\mathrm{dual}}_\varepsilon}
U_\lambda(\theta,\psi)
\le
\inf_{(\theta,\omega,\psi)\in\mathcal{A}^{\mathrm{cond}}_\varepsilon}
U_\lambda(\theta,\psi),
\]
that is,
\[
\Phi_{\mathrm{dual}}^\star(\varepsilon;\lambda)
\le
\Phi_{\mathrm{cond}}^\star(\varepsilon;\lambda).
\]

Combining the two displayed inequalities proves \eqref{eq:pareto_value_compare}.
\end{proof}

\paragraph{Tightness and equality cases.}
The pointwise bound in Theorem~\ref{thm:pareto_upper_bound} is tight only when both the data-processing gaps and the variational prior-mismatch gaps vanish.

For the visible nuisance axis, the proof uses the chain
\[
I_\theta(Z;\nuis\mid \condvar)
\le
I_\theta(Z;X\mid \condvar)
\le
J_{\mathrm{cond}}(\theta,\psi).
\]
Under the conditional Markov chain \(\nuis\to X\to Z\) given \(\condvar\), the first inequality is tight when
\[
I_\theta(Z;X\mid \nuis,\condvar)=0.
\]
The second inequality is tight when the conditional prior is exactly calibrated to the condition-wise aggregated posterior,
\[
r_\psi(z\mid \condval)=\aggq(z\mid \condval)
\quad
\text{for }p(\condval)\text{-a.e. }\condval.
\]

For the conditional-invisible axis controlled through global capacity, the proof uses
\[
I_\theta(Z;B)
\le
I_\theta(Z;X)
\le
J_{\mathrm{std}}(\theta).
\]
Under the encoder-induced Markov chain \(B\to X\to Z\) from Appendix~\ref{app:notation}, the first inequality is tight when
\[
I_\theta(Z;X\mid B)=0.
\]
The second inequality is tight when the global prior is exactly calibrated to the aggregated posterior,
\[
r_0(z)=\aggq(z).
\]

Therefore, equality in the weighted pointwise bound
\[
H_\lambda(\theta)=U_\lambda(\theta,\psi)
\]
requires the corresponding tightness conditions on both axes, since
\(\lambda_{\mathrm{vis}},\lambda_{\mathrm{inv}}>0\).
The value comparison
\[
\Phi_{\mathrm{dual}}^\star(\varepsilon;\lambda)
\le
\Phi_{\mathrm{std}}^\star(\varepsilon;\lambda),
\qquad
\Phi_{\mathrm{dual}}^\star(\varepsilon;\lambda)
\le
\Phi_{\mathrm{cond}}^\star(\varepsilon;\lambda)
\]
is instead a family-inclusion statement. Equality there occurs only when enlarging the supported family from the single-term subfamily to the dual family does not reduce the infimum of \(U_\lambda\).

\paragraph{Interpretation.}
The theorem should be read as a statement about \emph{provable upper bounds on residual information},
not as a universal claim about realized residual information or realized task performance.
Its content is the following:
under a task-adequacy constraint, the dual objective provides provable upper bounds of two residual-information axes: visible nuisance dependence and conditional-invisible condition-level encoding,
whereas the two single-term families are only degenerate subfamilies of that
Pareto control class.

\paragraph{Connection to the dual objective.}
The theorem sharpens the earlier synthesis logic in three ways.
First, it makes explicit the two residual-information axes behind the dual objective:
visible nuisance dependence and conditional-invisible condition-level encoding.
Second, it turns the earlier qualitative complementarity argument into the pointwise upper-bound
statement \(H_\lambda(\theta)\le U_\lambda(\theta,\psi)\).
Third, it formalizes the family-level advantage of dual regularization through the value comparison
\eqref{eq:pareto_value_compare}, which is a constrained Pareto statement under task adequacy.

\subsection{Loss-agnostic interpretation of the dual regularizer}
\label{app:universality}

This subsection clarifies the sense in which the proposed dual-bottleneck objective admits a loss-agnostic information-theoretic interpretation. A central observation is that the two KL terms play consistent representational roles across supervised objectives: the standard KL regularizes global information capacity, whereas the conditional KL regularizes within-condition information relative to the label-induced partition.

\begin{proposition}[Regularizer-level universality]
\label{prop:loss_agnostic}
Consider any supervised objective with target \(Y\), and suppose:
(i) there exists a finite label-induced partition \(\condvar=\partmap(Y)\);
(ii) the encoder/prior families are such that the displayed KL terms are well-defined;
(iii) the aggregated posteriors \(\aggq(z)\) and \(\aggq(z\mid \condval)\) exist;
(iv) the task loss \(\ell_{\mathrm{task}}(g_\omega(z),y)\) is well-defined.

Then the dual-bottleneck objective is a well-defined regularized supervised objective.
Moreover, the information-theoretic roles of the two KL regularizers are independent of the
specific choice of \(\ell_{\mathrm{task}}\): the standard KL controls global information
bandwidth up to prior mismatch, while the conditional KL controls within-condition information
up to condition-wise prior mismatch.
\end{proposition}

\begin{proof}
Well-definedness follows directly from assumptions (i)--(iv): a finite label-induced partition
\(\condvar=\partmap(Y)\), well-defined KL terms, existing aggregated posteriors, and a
well-defined task loss together ensures that each term in the dual-bottleneck objective is
properly defined.

For the regularizer-level loss-agnostic interpretation, apply the decomposition identities for
the two KL terms. By Lemma~\ref{lem:stdkl}, the unconditional KL controls \(I_\theta(Z;X)\)
up to a global prior mismatch; by Theorem~\ref{thm:condkl}, the conditional KL controls
\(I_\theta(Z;X\mid \condvar)\) up to a condition-wise prior mismatch.
These identities depend on the encoder and prior families rather than on the specific form of
\(\ell_{\mathrm{task}}\). Therefore, the representational roles of the two KL regularizers are
independent of the supervised loss choice at the regularizer level.
\end{proof}

\paragraph{Remark on loss dependence.}
The information-theoretic interpretation above concerns the regularizing action imposed on the
encoded representation. The eventual task-level behaviour of a learned model is jointly determined
by the regularizer, the predictive loss, the model family, and the optimization setting.
In particular, a well-defined supervised loss can still be task-agnostic; for example,
\[
\ell_{\mathrm{task}}(g(z),y)\equiv 0,
\]
or more generally \(\ell_{\mathrm{task}}(g(z),y)=h(g(z))\), which is independent of \(y\).
Such losses satisfy well-definedness but provide no task-discriminative learning signal.
Hence, this proposition is a regularizer-level universality statement, not a task-performance
invariance guarantee.

\paragraph{Interpretation.}
These results characterize the proposed method as a loss-agnostic representation regularizer in the following precise sense: the two KL terms preserve stable information-theoretic roles across supervised objectives, thereby defining a structured mutual-information-based mechanism for constraining encoded features. This interpretation is centered on how the regularizer shapes the learned representation through complementary control of global capacity and within-condition variability.

\subsection{Scope and non-claims}
\label{app:scope}

For clarity, we explicitly state what the present theory does and does not establish.

\paragraph{Established in this note.}
\begin{itemize}
    \item Exact decomposition of \(I(Z;X)\) under a label-induced partition.
    \item Existence-level insufficiency of a single global bottleneck for the stated design goal.
    \item Exact decomposition of the conditional KL term into within-condition mutual information and prior mismatch.
    \item A symmetric structured conditional prior with explicit center separation.
    \item An objective-level justification of the dual-bottleneck training criterion.
\end{itemize}

\paragraph{Not established in this note.}
\begin{itemize}
    \item Universal superiority of the conditional KL term over all possible condition-aware regularizers.
    \item A direct task-specific performance analysis under particular loss families, model classes, and optimization regimes.
    \item Optimization consistency or estimator optimality in minibatch finite-sample training.
    \item A task-specific comparative analysis of alternative prior geometries and initialization schemes.
\end{itemize}

Stating these non-claims explicitly improves the credibility of the method by aligning the theorem
strength with the actual proof strength.

\subsection{Family-restricted \texorpdfstring{$\Delta_{\mathrm{opt}}$}{Delta-opt}-optimality under a weighted target criterion}
\label{app:family_restricted_epsilon_optimality}

We now restate the comparison result in a form that makes its scope precise.
The result below is a family-restricted $\Delta_{\mathrm{opt}}$-optimality statement on a nonempty local
feasible family. It is not a global optimality theorem over all encoders.

\paragraph{Standing assumptions.}
Throughout this subsection, we inherit the standing regularity assumptions from Appendix~A.1:
all mutual informations and KL divergences below are assumed to be well-defined and finite,
and the aggregated posteriors $\bar q(z)$ and $\bar q(z\mid c)$ are assumed to exist.

\paragraph{Feasible family and target criterion.}
Fix a semantic floor $\kappa\ge 0$, and define
\[
\mathcal Q_\kappa
:=
\left\{
q:\ I_q(Z;C)\ge \kappa
\right\}.
\]
Let
\[
\emptyset \neq \mathcal F_\kappa \subseteq \mathcal Q_\kappa
\]
be a nonempty local feasible family.

For $q\in\mathcal F_\kappa$, define
\[
b(q):=I_q(Z;B),
\qquad
u(q):=I_q(Z;\nuis\mid C),
\]
and for weights satisfying
\[
0\le \lambda_B \le \lambda_U,
\]
define
\[
\mathcal J_{\kappa}^{\lambda_B,\lambda_U}(q)
:=
\lambda_B\,b(q)+\lambda_U\,u(q).
\]

\paragraph{Local assumptions.}
We assume the following.

\begin{enumerate}
    \item \textbf{Refined decomposition.}
    \[
    X=T(S,\nuis,B),\qquad Y=f(S),\qquad C=\partmap(Y)=\partmap(y^\star(X)).
    \]
    By the encoder-induced Markov convention in Appendix~\ref{app:notation}, the decomposition \(I_q(Z;X)=I_q(Z;C)+I_q(Z;X\mid C)\) applies.

    \item \textbf{Condition-level component.}
    The variable \(B\) denotes the conditional-invisible condition-level component from
    Appendix~\ref{subsec:invisible_condition_level_encoding}. We assume \(B\) is \(C\)-measurable, i.e.,
    \[
    H(B\mid C)=0.
    \]
    Thus \(b(q)=I_q(Z;B)\) measures condition-level encoding rather than within-condition
    residual information.

    \item \textbf{Interpretive nuisance visibility only.}
    Conditioned on $C$, the variables satisfy
    \[
    \nuis\to X\to Z.
    \]
    This assumption is used only to justify the interpretation of $u(q)=I_q(Z;\nuis\mid C)$ as a
    within-condition nuisance-retention quantity. It is not used in the algebraic proof below.

    \item \textbf{Bounded identification error.}
    There exist functions $\delta_B(q)$ and $\delta_U(q)$ such that for every
    $q\in\mathcal F_\kappa$,
    \[
    I_q(Z;C)=\kappa+b(q)+\delta_B(q),
    \qquad
    I_q(Z;X\mid C)=u(q)+\delta_U(q),
    \]
    with
    \[
    |\delta_B(q)|\le \eta_B,
    \qquad
    |\delta_U(q)|\le \eta_U
    \qquad \forall q\in\mathcal F_\kappa.
    \]

    \item \textbf{Bounded mismatch oscillation.}
    Define
    \[
    m_0(q):=\mathrm{KL}\!\bigl(\bar q(z)\,\|\,r_0(z)\bigr),
    \qquad
    m_c(q):=\mathbb E_{p(c)}\mathrm{KL}\!\bigl(\bar q(z\mid c)\,\|\,r_\psi(z\mid c)\bigr),
    \]
    and
    \[
    M(q):=\lambda_B\,m_0(q)+(\lambda_U-\lambda_B)\,m_c(q).
    \]
    Assume
    \[
    \operatorname{osc}_{\mathcal F_\kappa}(M)
    :=
    \sup_{q\in\mathcal F_\kappa}M(q)-\inf_{q\in\mathcal F_\kappa}M(q)
    \le \xi
    \]
    for some $\xi\ge 0$.

    \item \textbf{Existence of constrained minimizers.}
    Assume that $\mathcal F_\kappa$ is compact under the chosen topology and that the objectives
    defined below are lower semicontinuous on $\mathcal F_\kappa$.
\end{enumerate}

\paragraph{Dual and scalar regularizers.}
Set
\[
\beta_{\mathrm{std}}=\lambda_B,
\qquad
\beta_{\mathrm{cond}}=\lambda_U-\lambda_B.
\]
Define
\[
\mathcal R_{\mathrm{dual}}(q)
:=
\beta_{\mathrm{std}}\,
\mathbb E_{p(x)}\mathrm{KL}\!\left(q(z\mid x)\,\|\,r_0(z)\right)
+
\beta_{\mathrm{cond}}\,
\mathbb E_{p(x,c)}\mathrm{KL}\!\left(q(z\mid x)\,\|\,r_\psi(z\mid c)\right),
\]
and
\[
\mathcal R_{\mathrm{std}}(q)
:=
\mathbb E_{p(x)}
\mathrm{KL}\!\left(q(z\mid x)\,\|\,r_0(z)\right).
\]
Let
\[
q_{\mathrm{dual}}\in\arg\min_{q\in\mathcal F_\kappa}\mathcal R_{\mathrm{dual}}(q),
\qquad
q_{\mathrm{std}}\in\arg\min_{q\in\mathcal F_\kappa}\mathcal R_{\mathrm{std}}(q).
\]

\begin{proposition}[Family-restricted \texorpdfstring{$\Delta_{\mathrm{opt}}$}{Delta-opt}-optimality]
\label{prop:family_restricted_epsilon_optimality}
Define
\[
\Delta_{\mathrm{opt}}:=2(\lambda_B\eta_B+\lambda_U\eta_U)+\xi.
\]
Then, for every $q\in\mathcal F_\kappa$,
\[
\mathcal J_{\kappa}^{\lambda_B,\lambda_U}(q_{\mathrm{dual}})
\le
\mathcal J_{\kappa}^{\lambda_B,\lambda_U}(q)+\Delta_{\mathrm{opt}}.
\]
In particular,
\[
\mathcal J_{\kappa}^{\lambda_B,\lambda_U}(q_{\mathrm{dual}})
\le
\mathcal J_{\kappa}^{\lambda_B,\lambda_U}(q_{\mathrm{std}})+\Delta_{\mathrm{opt}}.
\]
Thus $q_{\mathrm{dual}}$ is a $\Delta_{\mathrm{opt}}$-optimal solution of the weighted target criterion on the
local feasible family $\mathcal F_\kappa$.
\end{proposition}

\begin{proof}
By the exact KL decompositions,
\[
\mathbb E_{p(x)}
\mathrm{KL}\!\left(q(z\mid x)\,\|\,r_0(z)\right)
=
I_q(Z;X)+m_0(q),
\]
and
\[
\mathbb E_{p(x,c)}
\mathrm{KL}\!\left(q(z\mid x)\,\|\,r_\psi(z\mid c)\right)
=
I_q(Z;X\mid C)+m_c(q).
\]
Hence
\[
\mathcal R_{\mathrm{dual}}(q)
=
\lambda_B\bigl(I_q(Z;X)+m_0(q)\bigr)
+
(\lambda_U-\lambda_B)\bigl(I_q(Z;X\mid C)+m_c(q)\bigr).
\]
Using
\[
I_q(Z;X)=I_q(Z;C)+I_q(Z;X\mid C),
\]
we obtain
\[
\mathcal R_{\mathrm{dual}}(q)
=
\lambda_B I_q(Z;C)+\lambda_U I_q(Z;X\mid C)+M(q).
\]
Substituting
\[
I_q(Z;C)=\kappa+b(q)+\delta_B(q),
\qquad
I_q(Z;X\mid C)=u(q)+\delta_U(q),
\]
gives
\[
\mathcal R_{\mathrm{dual}}(q)
=
\lambda_B\kappa+\mathcal J_{\kappa}^{\lambda_B,\lambda_U}(q)+E(q),
\]
where
\[
E(q):=\lambda_B\delta_B(q)+\lambda_U\delta_U(q)+M(q).
\]
Therefore
\[
\operatorname{osc}_{\mathcal F_\kappa}(E)
\le
2(\lambda_B\eta_B+\lambda_U\eta_U)+\operatorname{osc}_{\mathcal F_\kappa}(M)
\le
2(\lambda_B\eta_B+\lambda_U\eta_U)+\xi
=
\Delta_{\mathrm{opt}}.
\]

Since $q_{\mathrm{dual}}$ minimizes $\mathcal R_{\mathrm{dual}}$ on $\mathcal F_\kappa$, for every
$q\in\mathcal F_\kappa$,
\[
\lambda_B\kappa+\mathcal J_{\kappa}^{\lambda_B,\lambda_U}(q_{\mathrm{dual}})+E(q_{\mathrm{dual}})
\le
\lambda_B\kappa+\mathcal J_{\kappa}^{\lambda_B,\lambda_U}(q)+E(q).
\]
Canceling $\lambda_B\kappa$ yields
\[
\mathcal J_{\kappa}^{\lambda_B,\lambda_U}(q_{\mathrm{dual}})
\le
\mathcal J_{\kappa}^{\lambda_B,\lambda_U}(q)+E(q)-E(q_{\mathrm{dual}}).
\]
Finally,
\[
E(q)-E(q_{\mathrm{dual}})
\le
\operatorname{osc}_{\mathcal F_\kappa}(E)
\le
\Delta_{\mathrm{opt}},
\]
which proves
\[
\mathcal J_{\kappa}^{\lambda_B,\lambda_U}(q_{\mathrm{dual}})
\le
\mathcal J_{\kappa}^{\lambda_B,\lambda_U}(q)+\Delta_{\mathrm{opt}}.
\]
Taking $q=q_{\mathrm{std}}$ gives the second inequality.
\end{proof}

\begin{corollary}[Argmin invariance under constant mismatch remainder]
\label{cor:argmin_invariance_constant_remainder}
Assume in addition that
\[
\eta_B=0,\qquad \eta_U=0,
\]
and that there exists a constant $M_\ast\in\mathbb R$ such that
\[
M(q)\equiv M_\ast
\qquad \forall q\in\mathcal F_\kappa.
\]
Then
\[
\mathcal R_{\mathrm{dual}}(q)
=
\lambda_B\kappa+\mathcal J_{\kappa}^{\lambda_B,\lambda_U}(q)+M_\ast
\qquad \forall q\in\mathcal F_\kappa,
\]
and therefore
\[
\arg\min_{q\in\mathcal F_\kappa}\mathcal R_{\mathrm{dual}}(q)
=
\arg\min_{q\in\mathcal F_\kappa}\mathcal J_{\kappa}^{\lambda_B,\lambda_U}(q).
\]
\end{corollary}

\begin{proof}
Under $\eta_B=\eta_U=0$, we have
\[
I_q(Z;C)=\kappa+b(q),\qquad I_q(Z;X\mid C)=u(q)
\]
for all $q\in\mathcal F_\kappa$. The identity then follows from the proof of
Proposition~\ref{prop:family_restricted_epsilon_optimality}. Since $\lambda_B\kappa+M_\ast$ is constant
on $\mathcal F_\kappa$, adding it does not change the minimizer set.
\end{proof}

\begin{corollary}[Exact reduction under zero mismatch]
\label{cor:exact_reduction_zero_mismatch}
Assume in addition that
\[
\eta_B=0,\qquad \eta_U=0,
\]
and
\[
m_0(q)=0,\qquad m_c(q)=0
\qquad \forall q\in\mathcal F_\kappa.
\]
Then
\[
\mathcal R_{\mathrm{dual}}(q)
=
\lambda_B\kappa+\mathcal J_{\kappa}^{\lambda_B,\lambda_U}(q)
\qquad \forall q\in\mathcal F_\kappa,
\]
and consequently
\[
\arg\min_{q\in\mathcal F_\kappa}\mathcal R_{\mathrm{dual}}(q)
=
\arg\min_{q\in\mathcal F_\kappa}\mathcal J_{\kappa}^{\lambda_B,\lambda_U}(q).
\]
\end{corollary}

\begin{proof}
Under the stated assumptions, $M(q)\equiv 0$ on $\mathcal F_\kappa$, so this is the special case
$M_\ast=0$ of Corollary~\ref{cor:argmin_invariance_constant_remainder}.
\end{proof}

\paragraph{Interpretation.}
The main proposition is a local feasible-family result, not a global optimality theorem.
Its force is the following: on a nonempty local family where the $B/\nuis$ identification error and
prior-mismatch variation are controlled, the representation selected by the dual regularizer is
$\Delta_{\mathrm{opt}}$-optimal for the weighted target criterion. When the identification error vanishes, and
the mismatch remainder is constant, one obtains argmin invariance; when the mismatch remainder
vanishes pointwise, one obtains the exact reduction
\[
\mathcal R_{\mathrm{dual}}=\lambda_B\kappa+\mathcal J_{\kappa}^{\lambda_B,\lambda_U}.
\]

\section{Full experimental results}
\label{app:full_exp}

This appendix is organized into three parts: experimental results, experimental setup, and visualization-based interpretation.

\subsection{Experimental results}
\label{app:exp_results}

Exact protocol details for all tables in this subsection are centralized in Appendix~\ref{app:experimental_setup}.

\subsubsection{BloodMNIST.}
Tab.~\ref{tab:app_bloodmnist_results} reports the full classification results under standard, 1\%-training-data, and 4-shot regimes. Exact protocol details for this table are provided in Appendix~\ref{app:setup_bloodmnist}. Both single-KL variants improve over the \noreg{} baseline in all three settings, while \dualkl{} consistently gives the best BAcc and Macro-F1. The improvement is already visible in the standard regime and becomes substantially larger in the low-data settings, especially for 1\% training data, which is consistent with the view that limited supervision makes the model more vulnerable to unstable residual cues and increases the value of representation-level regularization.

\begin{table}[h]
\centering
\caption{BloodMNIST classification results reported as mean $\pm$ standard deviation. }
\label{tab:app_bloodmnist_results}
\small
\setlength{\tabcolsep}{4.5pt}
\begin{tabular}{llcccc}
\toprule
Setting & Metric & \noreg{} & \condkl{} & \stdkl{} & \dualkl{} \\
\midrule
\multirow{2}{*}{Standard}
& BAcc     & 92.90 $\pm$ 0.34 & 94.36 $\pm$ 0.53 & 94.35 $\pm$ 0.17 & \textbf{94.52 $\pm$ 0.22} \\
& Macro-F1 & 93.14 $\pm$ 0.26 & 94.51 $\pm$ 0.50 & 94.53 $\pm$ 0.18 & \textbf{94.72 $\pm$ 0.19} \\
\midrule
\multirow{2}{*}{1\% Training Data}
& BAcc     & 63.49 $\pm$ 1.07 & 71.84 $\pm$ 1.78 & 68.11 $\pm$ 2.65 & \textbf{72.22 $\pm$ 1.38} \\
& Macro-F1 & 61.43 $\pm$ 0.74 & 69.97 $\pm$ 1.83 & 66.25 $\pm$ 3.01 & \textbf{69.98 $\pm$ 1.85} \\
\midrule
\multirow{2}{*}{4-shot}
& BAcc     & 45.31 $\pm$ 1.77 & 46.50 $\pm$ 3.09 & 47.10 $\pm$ 2.97 & \textbf{48.72 $\pm$ 2.59} \\
& Macro-F1 & 43.19 $\pm$ 1.97 & 44.47 $\pm$ 2.90 & 45.02 $\pm$ 2.44 & \textbf{46.57 $\pm$ 2.58} \\
\bottomrule
\end{tabular}
\end{table}

\subsubsection{External BloodMNIST baselines.}
To test whether the BloodMNIST gains can be explained by generic regularization or empirical representation geometry, we further compare against strong weight decay, head dropout, label smoothing, and Center Loss in Tab.~\ref{tab:app_bloodmnist_external_baselines}. Strong weight decay and dropout represent standard capacity-control or stochastic regularization strategies~\citep{krogh1991simple,srivastava2014dropout}, Label Smoothing is a standard target-smoothing regularizer~\citep{szegedy2016rethinking}, while Center Loss represents an empirical class-compactness objective~\citep{wen2016discriminative}. These baselines are intentionally chosen to preserve the same backbone and task pipeline, avoiding protocol changes that would confound the comparison. Although Center Loss is the strongest external baseline, \textsc{DualKL} consistently outperforms it across all three regimes, with the largest margin under 1\% training data.

\begin{table}[h]
\centering
\caption{External baseline comparison on BloodMNIST. Results are reported as mean $\pm$ standard deviation.}
\label{tab:app_bloodmnist_external_baselines}
\scriptsize
\setlength{\tabcolsep}{5pt}
\begin{tabular}{llccccc}
\toprule
Setting & Metric & Strong WD & Head Dropout & Label Smoothing & Center Loss & \dualkl{} \\
\midrule
\multirow{2}{*}{Standard}
& BAcc
& 93.02 $\pm$ 0.35
& 93.08 $\pm$ 0.63
& 93.05 $\pm$ 0.35
& 93.81 $\pm$ 0.46
& \textbf{94.70 $\pm$ 0.30} \\
& Macro-F1
& 93.11 $\pm$ 0.39
& 93.23 $\pm$ 0.64
& 93.33 $\pm$ 0.32
& 93.95 $\pm$ 0.43
& \textbf{94.88 $\pm$ 0.23} \\
\midrule
\multirow{2}{*}{1\% Training Data}
& BAcc
& 62.21 $\pm$ 1.26
& 63.32 $\pm$ 1.45
& 65.05 $\pm$ 0.97
& 68.85 $\pm$ 1.56
& \textbf{71.29 $\pm$ 1.57} \\
& Macro-F1
& 60.40 $\pm$ 1.53
& 61.79 $\pm$ 1.78
& 62.85 $\pm$ 1.07
& 66.99 $\pm$ 1.23
& \textbf{68.98 $\pm$ 1.36} \\
\midrule
\multirow{2}{*}{4-shot}
& BAcc
& 43.88 $\pm$ 1.36
& 45.22 $\pm$ 2.74
& 45.74 $\pm$ 3.46
& 46.80 $\pm$ 1.95
& \textbf{48.01 $\pm$ 2.95} \\
& Macro-F1
& 42.08 $\pm$ 1.10
& 42.51 $\pm$ 2.60
& 43.75 $\pm$ 3.22
& 44.72 $\pm$ 2.14
& \textbf{45.91 $\pm$ 2.98} \\
\bottomrule
\end{tabular}
\end{table}

\subsubsection{CIFAR-100 low-data natural-image check.}
\label{app:cifar100_lowdata}
To test whether the low-data trend is specific to medical images, we additionally
evaluate CIFAR-100~\citep{krizhevsky2009learning}. Full public natural-image
benchmarks can provide broader class coverage, which may reduce the relative
effect of idiosyncratic residual cues; therefore, we focus on 4\% dataset and 4-shot
settings as a low-data natural-image check. Tab.~\ref{tab:cifar100_lowdata_results}
shows that all KL-regularized variants improve over \noreg{}, and \dualkl{}
achieves the best mean BAcc and Macro-F1 in both settings.

\begin{table}[h]
\centering
\caption{CIFAR-100 low-data classification results. We report mean $\pm$ standard deviation over seeds.}
\label{tab:cifar100_lowdata_results}
\small
\setlength{\tabcolsep}{4.2pt}
\begin{tabular}{llcccc}
\toprule
Setting & Metric & \noreg{} & \condkl{} & \stdkl{} & \dualkl{} \\
\midrule
\multirow{2}{*}{4\% dataset}
& BAcc
& 56.01 $\pm$ 0.54
& 57.79 $\pm$ 0.31
& 57.95 $\pm$ 0.36
& \textbf{58.04 $\pm$ 0.19} \\
& Macro-F1
& 55.87 $\pm$ 0.52
& 57.43 $\pm$ 0.32
& 57.62 $\pm$ 0.38
& \textbf{57.76 $\pm$ 0.19} \\
\midrule
\multirow{2}{*}{4-shot}
& BAcc
& 28.86 $\pm$ 0.92
& 30.04 $\pm$ 0.90
& 30.12 $\pm$ 1.02
& \textbf{30.18 $\pm$ 1.27} \\
& Macro-F1
& 28.79 $\pm$ 0.80
& 29.45 $\pm$ 0.85
& 29.52 $\pm$ 0.96
& \textbf{29.63 $\pm$ 1.24} \\
\bottomrule
\end{tabular}
\end{table}

\subsubsection{CIFAR-100 external baselines.}
We further compare \textsc{DualKL} with generic regularization and empirical
geometry baselines under the same CIFAR-100 low-data protocols. 
Tab.~\ref{tab:cifar100_external_baselines} shows that \textsc{DualKL} gives the
best mean BAcc in both settings and remains comparable to the strongest
external baseline on Macro-F1.

\begin{table}[h]
\centering
\caption{External baseline comparison on CIFAR-100 low-data classification. Results are reported as mean $\pm$ standard deviation over evaluated seeds.}
\label{tab:cifar100_external_baselines}
\small
\setlength{\tabcolsep}{3.2pt}
\begin{tabular}{llccccc}
\toprule
Setting & Metric & Strong WD & Head Dropout & Label Smoothing & Center Loss & \dualkl{} \\
\midrule
\multirow{2}{*}{4\% dataset}
& BAcc
& 55.54 $\pm$ 0.45
& 55.50 $\pm$ 0.27
& 57.14 $\pm$ 0.12
& 57.98 $\pm$ 0.37
& \textbf{58.14 $\pm$ 0.14} \\
& Macro-F1
& 55.44 $\pm$ 0.51
& 55.35 $\pm$ 0.25
& 56.73 $\pm$ 0.19
& 57.80 $\pm$ 0.35
& \textbf{57.86 $\pm$ 0.17} \\
\midrule
\multirow{2}{*}{4-shot}
& BAcc
& 28.94 $\pm$ 1.03
& 29.11 $\pm$ 1.07
& 29.62 $\pm$ 1.47
& 30.50 $\pm$ 1.19
& \textbf{30.51 $\pm$ 1.37} \\
& Macro-F1
& 28.81 $\pm$ 0.89
& 29.02 $\pm$ 1.01
& 28.91 $\pm$ 1.46
& \textbf{29.96 $\pm$ 1.13}
& 29.94 $\pm$ 1.35 \\
\bottomrule
\end{tabular}
\end{table}

\subsubsection{BloodMNIST prior and $\tau$ ablations.}
\label{app:bloodmnist_prior_tau_results}

Together with the 1\% ablation shown in the main text (Sec.~\ref{sec:experiments}, Tab.~\ref{tab:bloodmnist_ablation_1pct}), Tabs.~\ref{tab:4shot_tuned_ablation_meanstd} and \ref{tab:full_ablation_subset_meanstd} report additional BloodMNIST ablations that isolate the practical role of the conditional-prior instantiation and the $\tau$ parameterization within the \textsc{DualKL} family. Across the reported settings, the default configuration with a simplex prior and shared fixed $\tau$ remains the strongest overall among the tested variants. These results support the default simplex-plus-shared-$\tau$ design as a stable and interpretable practical instantiation, motivating its use as the default configuration in our experiments. Exact protocol details are provided in Appendix~\ref{app:setup_bloodmnist_prior_tau}.

\begin{table}[H]
\centering
\caption{4-shot BloodMNIST ablation results reported as mean $\pm$ standard deviation.}
\label{tab:4shot_tuned_ablation_meanstd}
\small
\setlength{\tabcolsep}{4.5pt}
\begin{tabular}{lcc}
\toprule
Configuration & BAcc & Macro-F1 \\
\midrule
\dualkl{} + simplex + shared fixed $\tau$ 
& \textbf{48.66 $\pm$ 1.99} 
& \textbf{46.59 $\pm$ 2.25} \\
\dualkl{} + random + shared fixed $\tau$ 
& 46.91 $\pm$ 2.40 
& 44.51 $\pm$ 2.82 \\
\dualkl{} + empirical + shared fixed $\tau$ 
& 46.97 $\pm$ 2.83 
& 45.19 $\pm$ 0.70 \\
\dualkl{} + simplex + classwise learned $\tau$ 
& 46.43 $\pm$ 1.39 
& 44.55 $\pm$ 0.38 \\
\bottomrule
\end{tabular}
\end{table}

\begin{table}[H]
\centering
\caption{Full-training BloodMNIST ablation results reported as mean $\pm$ standard deviation.}
\label{tab:full_ablation_subset_meanstd}
\small
\setlength{\tabcolsep}{4.5pt}
\begin{tabular}{lcc}
\toprule
Configuration & BAcc & Macro-F1 \\
\midrule
\dualkl{} + simplex + shared fixed $\tau$ & \textbf{94.74 $\pm$ 0.28} & \textbf{94.86 $\pm$ 0.12} \\
\dualkl{} + random + shared fixed $\tau$ & 94.56 $\pm$ 0.24 & 94.63 $\pm$ 0.11 \\
\dualkl{} + empirical + shared fixed $\tau$ & 94.52 $\pm$ 0.16 & 94.62 $\pm$ 0.16 \\
\dualkl{} + simplex + classwise learned $\tau$ & 94.68 $\pm$ 0.35 & 94.79 $\pm$ 0.32 \\
\bottomrule
\end{tabular}
\end{table}
\FloatBarrier

\subsubsection{Retinal vessel segmentation.}
Tab.~\ref{tab:app_retinal_results} shows a consistent pattern across DRIVE, STARE, and CHASEDB1. Exact protocol details for this table are provided in Appendix~\ref{app:setup_retinal}. \textsc{DualKL} achieves the best IoU and Dice on all three datasets, while both single-KL variants also improve over the unregularized baseline. Together with the larger low-data classification gains, these dense-prediction results show that the proposed regularizer provides stable improvements across task types.

\begin{table}[!htbp]
\centering
\caption{Repeated-seed retinal vessel segmentation results reported as mean $\pm$ standard deviation.}
\label{tab:app_retinal_results}
\small
\setlength{\tabcolsep}{4.5pt}
\begin{tabular}{llcccc}
\toprule
Dataset & Metric & \noreg{} & \stdkl{} & \condkl{} & \dualkl{} \\
\midrule
\multirow{2}{*}{DRIVE}
& IoU  & 0.5558$\pm$0.0061 & 0.5823$\pm$0.0037 & 0.5820$\pm$0.0044 & \textbf{0.5826$\pm$0.0053} \\
& Dice & 0.7143$\pm$0.0051 & 0.7359$\pm$0.0030 & 0.7356$\pm$0.0035 & \textbf{0.7362$\pm$0.0042} \\
\midrule
\multirow{2}{*}{STARE}
& IoU  & 0.5049$\pm$0.0311 & 0.5302$\pm$0.0121 & 0.5312$\pm$0.0156 & \textbf{0.5318$\pm$0.0145} \\
& Dice & 0.6701$\pm$0.0278 & 0.6926$\pm$0.0104 & 0.6934$\pm$0.0134 & \textbf{0.6939$\pm$0.0125} \\
\midrule
\multirow{2}{*}{CHASEDB1}
& IoU  & 0.5525$\pm$0.0066 & 0.5724$\pm$0.0043 & 0.5720$\pm$0.0034 & \textbf{0.5739$\pm$0.0044} \\
& Dice & 0.7116$\pm$0.0054 & 0.7280$\pm$0.0035 & 0.7276$\pm$0.0028 & \textbf{0.7291$\pm$0.0036} \\
\bottomrule
\end{tabular}
\end{table}

\subsubsection{GLAS.}
Tab.~\ref{tab:app_glas_results} extends the segmentation evaluation to both full-data and few-shot regimes. Exact protocol details for this table are provided in Appendix~\ref{app:setup_glas}. \textsc{DualKL} consistently outperforms the unregularized baseline and \textsc{StdKL} in all reported settings, with the clearest gain appearing in the 4-shot regime. This mirrors the BloodMNIST trend and provides additional evidence for the same low-supervision advantage beyond classification, suggesting that the proposed regularizer is particularly beneficial when supervision is limited.

\begin{table}[!htbp]
\centering
\caption{GLAS segmentation results in full-data and few-shot regimes, reported as mean $\pm$ standard deviation.}
\label{tab:app_glas_results}
\small
\setlength{\tabcolsep}{4.5pt}
\begin{tabular}{llcccc}
\toprule
Setting & Metric & \noreg{} & \stdkl{} & \condkl{} & \dualkl{} \\
\midrule
\multirow{2}{*}{Full}
& IoU  & 0.8131$\pm$0.0008 & 0.8149$\pm$0.0040 & 0.8118$\pm$0.0051 & \textbf{0.8158$\pm$0.0017} \\
& Dice & 0.8957$\pm$0.0005 & 0.8970$\pm$0.0032 & 0.8949$\pm$0.0031 & \textbf{0.8976$\pm$0.0011} \\
\midrule
\multirow{2}{*}{8-shot}
& IoU  & 0.6009$\pm$0.0145 & 0.6039$\pm$0.0141 & 0.5439$\pm$0.0471 & \textbf{0.6106$\pm$0.0111} \\
& Dice & 0.7489$\pm$0.0116 & 0.7506$\pm$0.0117 & 0.7015$\pm$0.0394 & \textbf{0.7562$\pm$0.0087} \\
\midrule
\multirow{2}{*}{4-shot}
& IoU  & 0.5368$\pm$0.0248 & 0.5460$\pm$0.0221 & 0.5440$\pm$0.0202 & \textbf{0.5478$\pm$0.0220} \\
& Dice & 0.6958$\pm$0.0218 & 0.7036$\pm$0.0188 & 0.7019$\pm$0.0175 & \textbf{0.7049$\pm$0.0190} \\
\bottomrule
\end{tabular}
\end{table}

\subsection{Experimental Setup}
\label{app:experimental_setup}

All experiments are implemented in PyTorch~\citep{paszke2019pytorch} using one
GeForce RTX 5090 GPU, 32GB RAM, and an AMD Ryzen 9 9950X3D CPU. All experiments are trained with automatic mixed precision (AMP) using PyTorch autocast and gradient scaling~\citep{micikevicius2018mixed}. The approximate training speed is 0.3--1.0 seconds per epoch for classification task and 0.2--0.7 seconds per epoch for the segmentation benchmarks, depending on the dataset and data regime. The total training compute for the reported experiments is approximately 50 GPU-hours, excluding failed or preliminary debugging runs.
Within each dataset and experimental regime, we keep the data split, backbone family, optimizer, training schedule, and evaluation metrics fixed unless explicitly stated otherwise. 
For the objective-level comparisons in Tabs.~\ref{tab:app_bloodmnist_results}, \ref{tab:cifar100_lowdata_results}, \ref{tab:app_retinal_results}, and \ref{tab:app_glas_results}, the compared variants differ only in the enabled KL controls and their coefficients. Unless otherwise stated, classification experiments use the identity partition \(C=Y\), while segmentation experiments use the \(3\times3\) foreground-count condition \(C_u=\sum_{v\in\mathcal N(u)}Y_v\) with \(K=10\). For auxiliary ablations and external baselines, the same base protocol is used while varying only the factor named in the corresponding table. External baseline strengths are validation-selected and fixed; for segmentation, KL coefficients are selected on training-only internal splits before evaluation on the fixed validation splits below. We use 10 seeds for the main objective-level ablations, 3 seeds for prior/\(\tau\)-parameterization ablations, and 5 seeds for external-baseline comparisons; within each table, all methods share the same seed count, split protocol, and fixed hyperparameters.

\subsubsection{BloodMNIST classification}
\label{app:setup_bloodmnist}

For the BloodMNIST classification experiments in Tab.~\ref{tab:app_bloodmnist_results}, we use the same ResNet-based classification architecture~\citep{he2016resnet} for all compared methods, with input resolution $28\times 28$ and a batch size of 1024. All models are optimized using Adam~\citep{kingma2015adam} with an initial learning rate of $10^{-4}$, weight decay $10^{-4}$, and a cosine annealing schedule~\citep{loshchilov2017sgdr} with a minimum learning rate of 0. The base classification loss is CrossEntropyLoss, and performance is evaluated using balanced accuracy (BAcc) and Macro-F1.

In the standard regime, all methods use the full BloodMNIST split, with 11{,}959 training images, 1{,}712 validation images, and 3{,}421 test images. We train all methods for 100 epochs, and compare \noreg{} $(0,0)$, \condkl{} $(0,0.003)$, \stdkl{} $(0.003,0)$, and \dualkl{} $(0.003,0.003)$.

In the 1\%-training-data regime, the validation and test sets remain unchanged, while the training set is class-balanced subsampled from the original training split to 20 samples per class, resulting in 160 training images in total. All methods use the same subsampling indices. Under this protocol, we train all methods for 1600 epochs, and compare \noreg{} $(0,0)$, \condkl{} $(0,0.0027)$, \stdkl{} $(0.003,0)$, and \dualkl{} $(0.0003,0.0027)$.

In the 4-shot regime, the validation and test sets again remain unchanged, while the training set is reduced to 4 samples per class, resulting in 32 training images in total. All methods share the same few-shot sampling result. Under this protocol, we train all methods for 1600 epochs, and compare \noreg{} $(0,0)$, \condkl{} $(0,0.000667)$, \stdkl{} $(0.001333,0)$, and \dualkl{} $(0.001333,0.000667)$.

\subsubsection{CIFAR-100 low-data classification}
\label{app:setup_cifar100}

For the CIFAR-100~\citep{krizhevsky2009learning} results in
Tab.~\ref{tab:cifar100_lowdata_results}, we use fine labels as the supervised
targets. All methods use the same
VIB-ResNet50 classifier with latent dimension 2048, $32\times32$ inputs, optimizer, data split protocol, and metrics.

Training uses random crop with padding 4 and horizontal flip; validation and
test use no random augmentation. We optimize with Adam, initial learning rate
$10^{-4}$, weight decay $10^{-4}$, cosine annealing, CrossEntropyLoss, batch
size 256, and 60 epochs. Performance is evaluated using balanced accuracy
(BAcc) and Macro-F1 on the official CIFAR-100 test split.

The 4\% dataset setting uses 20 samples per class, resulting in 2{,}000 training
images; the 4-shot setting uses 4 samples per class, resulting in 400 training
images. We compare \noreg{} $(0,0)$, \condkl{} $(0,0.0001)$, \stdkl{} $(0.0001,0)$, and 
\dualkl{} $(0.0001,0.0001)$.

For CIFAR-100 external baselines, we use the same data split, VIB-ResNet50
backbone, input resolution, batch size, optimizer, training schedule,
augmentation, and evaluation metrics as the corresponding CIFAR-100
objective-level runs. For these runs, the KL coefficients are set to zero and
only one external regularization mechanism is enabled at a time. Strong WD
increases the weight-decay coefficient, Head Dropout inserts dropout at the
representation head, Label Smoothing uses a smoothed cross-entropy target, and
Center Loss adds a class-center compactness penalty on the latent mean.

\subsubsection{External BloodMNIST baselines.}
The external baselines in Tab.~\ref{tab:app_bloodmnist_external_baselines} use the same BloodMNIST splits, ResNet-based backbone, input resolution, batch size, optimizer, training schedules, and evaluation metrics as Tab.~\ref{tab:app_bloodmnist_results}. 
For these runs, the KL coefficients are set to zero and only one external regularization mechanism is enabled at a time. 
Strong WD~\citep{krogh1991simple} increases the weight-decay coefficient, Head Dropout~\citep{srivastava2014dropout} inserts dropout only at the representation head, Label Smoothing~\citep{szegedy2016rethinking} uses a smoothed cross-entropy target, and Center Loss~\citep{wen2016discriminative} adds a class-center compactness penalty on the latent representation.
We focus on these matched controls because they test generic regularization and empirical feature-geometry explanations without changing the data modality, backbone family, supervision protocol, or requiring additional multi-view or coding-rate training pipelines.
External-baseline strengths are selected once using the same validation protocol as the KL coefficients and then fixed before the 5-seed repeated evaluation. For BloodMNIST full/1\% and CIFAR-100 20-shot/4-shot, we use Strong WD \(=10^{-3}\), Head Dropout \(=0.2\), Label Smoothing \(=0.1\), and Center Loss weight \(=10^{-3}\); for BloodMNIST 4-shot, we use \(3\times10^{-4}\), \(0.1\), \(0.05\), and \(3\times10^{-4}\), respectively.

\subsubsection{BloodMNIST prior and $\tau$ ablation setup}
\label{app:setup_bloodmnist_prior_tau}

For the BloodMNIST prior and $\tau$ ablation experiments reported in
Tab.~\ref{tab:bloodmnist_ablation_1pct} and Tabs.~\ref{tab:4shot_tuned_ablation_meanstd}--\ref{tab:full_ablation_subset_meanstd},
we use a ResNet-based classification architecture with input resolution $28\times 28$
and batch size 1024. All models are optimized using Adam with an initial learning
rate of $10^{-4}$, weight decay $10^{-4}$, and a cosine annealing schedule with
minimum learning rate 0. The base classification loss is CrossEntropyLoss, and
performance is evaluated using balanced accuracy (BAcc) and Macro-F1. In all cases,
we use the \textsc{DualKL} objective and only vary the conditional-prior
instantiation or the $\tau$ parameterization.

We compare the default setting, \textit{\dualkl{} + simplex prior + shared fixed \(\tau\)}, against the alternative configurations reported in Tab.~\ref{tab:bloodmnist_ablation_1pct} and Tabs.~\ref{tab:4shot_tuned_ablation_meanstd}--\ref{tab:full_ablation_subset_meanstd}. For the simplex variant, \(\rho_{\mathrm{raw}}\), \(R_{\mathrm{raw}}\), and \(\priorbias\) are learned, with \(\rho=\operatorname{softplus}(\rho_{\mathrm{raw}})\) and \(\embedmat=\operatorname{qf}(R_{\mathrm{raw}})\), the \(Q\)-factor of a thin QR decomposition. For random and empirical variants, the conditional centers are initialized differently but are then optimized as unconstrained trainable parameters. Random centers are sampled from \(\mathcal{N}(0,0.02^2\Id)\) using a fixed prior seed; empirical centers are initialized as class-wise mean encoder posterior means on the effective training split. All shared-fixed-\(\tau\) variants use \(\tau=1\) without learning \(\tau\); the classwise-\(\tau\) variant learns one positive scalar variance per class.

In the 1\% training data regime, the training split is class-balanced
subsampled to 20 samples per class, resulting in 160 training images in total,
while the validation and test sets remain unchanged. All compared variants use the
same subsampling indices. Under this protocol, all models are trained for 1600
epochs using $(\beta_{\mathrm{std}}, \beta_{\mathrm{cond}}) = (0.0003, 0.0027)$.

In the 4-shot regime, the training split is class-balanced subsampled to 4 samples
per class, resulting in 32 training images in total, while the validation and test
sets remain unchanged. All compared variants use the same subsampling indices.
Under this protocol, all models are trained for 1600 epochs using
$(\beta_{\mathrm{std}}, \beta_{\mathrm{cond}}) = (0.001333, 0.000667)$.

All ablation results are reported as mean $\pm$ standard deviation over the same
set of repeated runs used to summarize the corresponding low-data ablation tables.

\subsubsection{Retinal vessel segmentation}
\label{app:setup_retinal}

For the retinal vessel segmentation experiments in Tab.~\ref{tab:app_retinal_results}, we use the same UNet-based architecture~\citep{ronneberger2015unet} for all compared methods, with input resolution $224\times 224$ and a batch size of 4. Images and masks are directly resized to \(224\times224\), without patch extraction or sliding-window evaluation. The same preprocessing is used for all compared methods. All models are optimized using Adam~\citep{kingma2015adam} with an initial learning rate of $10^{-4}$, weight decay $10^{-4}$, and a cosine annealing schedule~\citep{loshchilov2017sgdr} with a minimum learning rate of $10^{-5}$, and are trained for 800 epochs. The base segmentation loss is BCEDiceLoss, defined as
\begin{equation}
\label{eq:app_bce_dice_loss}
\mathcal{L}_{\mathrm{seg}}=0.5\,\mathcal{L}_{\mathrm{BCE}}+\mathcal{L}_{\mathrm{Dice}},
\end{equation}
where $\mathcal{L}_{\mathrm{BCE}}=\mathrm{BCEWithLogitsLoss}$ and $\mathcal{L}_{\mathrm{Dice}}=\mathrm{DiceLoss}$. The BCE term provides pixel-wise supervision and the Dice term improves overlap-oriented optimization under class imbalance. We report validation IoU and Dice for all retinal vessel experiments.

On DRIVE, all methods use the same fixed official split with 20 training images and 20 validation images. The compared variants are \noreg{} with $(\beta_{\mathrm{std}},\beta_{\mathrm{cond}})=(0,0)$, \stdkl{} with $(0.001,0)$, \condkl{} with $(0,0.001)$, and \dualkl{} with $(0.0004,0.0006)$.

On STARE, all methods use the same split configuration with $\mathrm{stare\_mask\_variant}=\mathrm{ah}$ and $\mathrm{val\_ratio}=0.4$, which yields 12 training images and 8 validation images. Under this protocol, we compare \noreg{} $(0,0)$, \stdkl{} $(0.001,0)$, \condkl{} $(0,0.001)$, and \dualkl{} $(0.0007,0.0005)$.

On CHASEDB1, all methods use the same split configuration with $\mathrm{val\_ratio}=0.27$, resulting in 20 training images and 8 validation images. The compared variants are \noreg{} $(0,0)$, \stdkl{} $(0.001,0)$, \condkl{} $(0,0.001)$, and \dualkl{} $(0.0007,0.0003)$.

\subsubsection{GLAS segmentation}
\label{app:setup_glas}

For the GLAS segmentation experiments in Tab.~\ref{tab:app_glas_results}, we use the same UNet-based architecture~\citep{ronneberger2015unet} for all compared methods, with input resolution $256\times 256$ and a batch size of 4. Images and masks are directly resized to \(256\times256\), without patch extraction or sliding-window evaluation. All models are optimized using Adam~\citep{kingma2015adam} with an initial learning rate of $10^{-4}$, weight decay $10^{-4}$, and a cosine annealing schedule~\citep{loshchilov2017sgdr} with a minimum learning rate of $10^{-5}$. The base segmentation loss is the same BCEDiceLoss used in the retinal vessel experiments (Eq.~\eqref{eq:app_bce_dice_loss}), and performance is evaluated using IoU and Dice.

In the full-data regime, all methods use the same fixed split rather than the official dataset split. Specifically, all 165 images are pooled together, and 45 images are used for validation, resulting in 120 training images and 45 validation images. All GLAS results are reported as validation-set IoU/Dice under this fixed train/validation protocol. We train all methods for 200 epochs. The compared variants are \noreg{} $(0,0)$, \stdkl{} $(10^{-4},0)$, \condkl{} $(0,10^{-4})$, and \dualkl{} $(10^{-4},10^{-4})$.

In the few-shot regimes, we keep the same 120/45 base split fixed and subsample either 4 or 8 training images from the 120-image training pool, while leaving the 45-image validation set unchanged. Thus, 4-shot and 8-shot refer to the total number of training images rather than the number of images per class. We again train all methods for 200 epochs, and compare \noreg{} $(0,0)$, \stdkl{} $(10^{-4},0)$, \condkl{} $(0,10^{-4})$, and \dualkl{} $(10^{-4},10^{-4})$.

\subsection{Visualization-Based Interpretation}
\label{app:vis}

\subsubsection{Visualization Construction and Proxy Interpretation}
\label{app:vis_proxy}

This subsection clarifies how we construct the two visualization families used in this paper: \emph{PCA latent geometry} and \emph{Latent-statistics geometry}. In both families, the plotted features are extracted from the learned representation immediately before the task head $g_\omega$, namely the latent variable $Z$ used in Sec.~\ref{subsec:objective} and Eq.~\eqref{eq:finalobjective}, as indicated by the pipeline in Fig.~\ref{fig:framework}. Concretely, under deterministic inference we use the encoder mean $\mu(x)$ as the latent feature, so the plots reflect the learned representation itself rather than extra sampling noise.

We use a unified representation-unit view across tasks. For classification, one sample corresponds to one latent vector. For segmentation, one valid pixel corresponds to one latent vector at the pre-head feature stage. Therefore, the geometric point density follows the task information granularity (sample-level for classification, pixel-level for segmentation), and all reported geometric summaries are computed on the same representation units used for plotting.

\paragraph{PCA latent geometry.}
Let $\mathcal U$ be the set of representation units, with latent vectors $\{z_u\in\mathbb R^d\}_{u\in\mathcal U}$ and labels $\{y_u\}_{u\in\mathcal U}$. For each dataset/regime, we fit one shared PCA projection jointly over the compared KL settings and then project each setting into this common 2D coordinate system. Thus, cross-panel differences reflect representation changes rather than projection mismatch.
This PCA geometry view is the empirical counterpart of the KL-complementarity comparison illustrated conceptually in Fig.~\ref{fig:latent_compare}.
The two PCA axes are treated as visualization coordinates only; we do not assign direct information-theoretic semantics to these axes.

Given class index set $\mathcal K$ and $S_k=\{u\in\mathcal U: y_u=k\}$, define centroid
\begin{equation}
\label{eq:app_vis_centroid}
c_k = \frac{1}{|S_k|}\sum_{u\in S_k} z_u.
\end{equation}
Let $\mathcal P=\{(a,b): a,b\in\mathcal K,\ a<b\}$ be unordered class pairs. We report:
\begin{equation}
\label{eq:app_vis_within}
\mathrm{Within}
=
\frac{1}{|\mathcal U|}\sum_{u\in\mathcal U}\|z_u-c_{y_u}\|_2^2,
\end{equation}
\begin{equation}
\label{eq:app_vis_between}
\mathrm{Between}
=
\frac{1}{|\mathcal P|}\sum_{(a,b)\in\mathcal P}\|c_a-c_b\|_2,
\end{equation}
\begin{equation}
\label{eq:app_vis_ratio}
\mathrm{Ratio}
=
\frac{\mathrm{Within}}{\max(\mathrm{Between}^2,\delta_{\mathrm{num}})},\qquad \delta_{\mathrm{num}}>0.
\end{equation}
All three quantities are computed in the original latent space before PCA projection. Smaller \textit{Within} and \textit{Ratio}, together with larger \textit{Between}, indicate better compactness--separation geometry. For clarity and to avoid redundant notation, we do not introduce additional anchor-based terms in this subsection; all geometric summaries are directly computed from pre-head representations.
Accordingly, whenever panel text/captions report \textit{Within}, \textit{Between}, or \textit{Ratio}, they refer exactly to Eqs.~\eqref{eq:app_vis_within}--\eqref{eq:app_vis_ratio}.

\paragraph{Latent-statistics geometry.}
In this view, each plotted point is one latent coordinate (not one sample). For latent dimension $d$, let $v_d=\{z_{u,d}:u\in\mathcal U\}$ be its pooled values over representation units. We compute three groups of statistics.

First, feature--label association:
\begin{equation}
\label{eq:app_vis_featlabel}
\mathrm{FeatLabel}_d
=
\sqrt{\frac{\mathrm{SS}_{\mathrm{between}}(v_d)}{\mathrm{SS}_{\mathrm{total}}(v_d)}}.
\end{equation}

\begin{equation}
\mathrm{SS}_{\mathrm{total}}(v_d)
=
\sum_{u \in \mathcal U} (z_{u,d}-\bar z_d)^2,
\qquad
\bar z_d = \frac{1}{|\mathcal U|}\sum_{u \in \mathcal U} z_{u,d},
\end{equation}

\begin{equation}
\mathrm{SS}_{\mathrm{between}}(v_d)
=
\sum_{k\in\mathcal K} n_k(\bar z_{k,d}-\bar z_d)^2,
\qquad
\bar z_{k,d}=\frac{1}{n_k}\sum_{u:y_u=k} z_{u,d}.
\end{equation}

Second, class-separation strength using average pairwise symmetric KL divergence:
\begin{equation}
\label{eq:app_vis_labelsep}
\mathrm{LabelSep}_d
=
\frac{1}{|\mathcal P|}\sum_{(a,b)\in\mathcal P}
\frac{1}{2}
\Bigl(\mathrm{KL}(p_{a,d}\|p_{b,d})+\mathrm{KL}(p_{b,d}\|p_{a,d})\Bigr),
\end{equation}
where $p_{k,d}$ is the class-conditional one-dimensional distribution of dimension $d$ estimated from the corresponding histogram.

Third, label-unexplained dispersion. Let $\mathrm{std}_d$, $\mathrm{cv}_d$, and $\mathrm{mad}_d$ denote standard deviation, coefficient of variation, and median absolute deviation of $v_d$, respectively. Define
\begin{equation}
\label{eq:app_vis_dispersion}
\mathrm{DispersionScore}_d
=
\frac{1}{3}\Bigl(
\operatorname{ranknorm}(\mathrm{std}_d)
+\operatorname{ranknorm}(\mathrm{cv}_d)
+\operatorname{ranknorm}(\mathrm{mad}_d)
\Bigr),
\end{equation}
\begin{equation}
\label{eq:app_vis_intersample}
\mathrm{InterSampleSep}_d
=
\mathrm{DispersionScore}_d\cdot\bigl(1-\mathrm{FeatLabel}_d^2\bigr).
\end{equation}
We then place each latent dimension at
\[
x_d=\mathrm{InterSampleSep}_d,\qquad y_d=\mathrm{LabelSep}_d,
\]
and color it by $\mathrm{FeatLabel}_d$.
Here, $\operatorname{ranknorm}(\cdot)$ denotes rank normalization to $[0,1]$ within the latent dimensions of the same model.

For visualization summarization, we additionally define a heuristic score for each latent dimension:
\begin{equation}
\label{eq:app_vis_skei}
\mathrm{sKEI}_d
=
\frac{\mathrm{LabelSep}_d \cdot \mathrm{FeatLabel}_d}
{(\mathrm{InterSampleSep}_d + 10^{-6}) \cdot \bigl(1 + \sigma^{\mathrm{label}}_d + \mathrm{CV}^{\mathrm{sep}}_d \bigr)},
\end{equation}
where $\sigma^{\mathrm{label}}_d$ denotes the split-wise standard deviation of the feature--label association score, and $\mathrm{CV}^{\mathrm{sep}}_d$ denotes the split-wise coefficient of variation of label separation. Concretely, over the evaluation splits $s\in\mathcal S$ used for visualization statistics,
\[
\sigma^{\mathrm{label}}_d
=
\operatorname{Std}_{s\in\mathcal S}\!\bigl(\mathrm{FeatLabel}^{(s)}_d\bigr),
\qquad
\mathrm{CV}^{\mathrm{sep}}_d
=
\frac{\operatorname{Std}_{s\in\mathcal S}\!\bigl(\mathrm{LabelSep}^{(s)}_d\bigr)}
{\operatorname{Mean}_{s\in\mathcal S}\!\bigl(\mathrm{LabelSep}^{(s)}_d\bigr)+10^{-6}}.
\]
This score is not an estimator of mutual information or downstream task performance. Rather, it serves as a heuristic proxy that rewards latent dimensions with strong label-separable structure and strong feature--label association, while penalizing label-unexplained dispersion and split instability. All sKEI values reported in panel text/captions follow Eq.~\eqref{eq:app_vis_skei}.

\paragraph{Proxy semantics and relation to theory.}
The axis interpretation is proxy-based rather than estimator identity. The x-axis is not a direct estimator of $I(Z;X\mid \condvar)$, but it emphasizes label-unexplained dispersion and therefore serves as an empirical proxy for within-condition residual variation. This matches the informational role of $I(Z;X\mid \condvar)$ in the encoder-induced decomposition of Theorem~\ref{thm:decomp} (Appendix~\ref{app:thm1}) and in the conditional-KL decomposition of Theorem~\ref{thm:condkl} (Appendix~\ref{app:thm3}).

Similarly, the y-axis is not a direct estimator of $I(Z;\condvar)$, but it quantifies condition-level separability and thus serves as an empirical proxy for condition-level structure under
\[
I(Z;X)=I(Z;\condvar)+I(Z;X\mid \condvar).
\]
The color value is also non-redundant with the y-axis: \textit{FeatLabel} measures association strength, while \textit{LabelSep} measures distributional separation. They are complementary summaries rather than duplicates. In this sense, the latent-statistics plot provides an empirical counterpart to the same two-role coordinate system presented schematically in Fig.~\ref{fig:info_plane}. It also complements the KL-complementarity comparison in Fig.~\ref{fig:latent_compare}, while Fig.~\ref{fig:framework} only indicates where the plotted features are extracted in the pipeline, namely the pre-head latent representation $Z$. Under matched protocols (Appendix~\ref{app:experimental_setup}), these visualization families therefore serve as diagnostic support for the decomposition-motivated reading in Sec.~\ref{subsec:interpretation} and Sec.~\ref{subsec:why_both}, and for the optional nuisance interpretation in Sec.~\ref{subsec:nuisance} and Appendix~\ref{app:nuisance} (Proposition~\ref{prop:nuisance}). Correspondingly, Fig.~\ref{fig:bloodmnist_feature_heterogeneity} plays the same interpretive role on fixed low-level input statistics, whereas the latent-statistics geometry extends that logic to learned latent coordinates before the task head.

\subsubsection{Visualization Results Across Four Regimes}
\label{app:vis_results}

Across all four regimes, the two visualization families show a consistent representation-level pattern. The PCA panels are sample-level for classification and pixel-level for segmentation, and indicate that the learned latent geometry becomes more compact within classes and more separable across classes, consistent with the latent-space complementarity direction in Fig.~\ref{fig:latent_compare}; this geometric trend is also consistent with the structured conditional-prior view in Sec.~\ref{subsec:simplex_prior} and Appendix~\ref{app:simplex} (Proposition~\ref{prop:simplex}). The latent-statistics panels are latent-dimension-level summaries and, under the proxy construction in Appendix~\ref{app:vis_proxy} (B3.1), show a redistribution toward lower label-unexplained dispersion, stronger label-separable structure, and stronger feature--label alignment, extending the fixed-statistics intuition in Fig.~\ref{fig:bloodmnist_feature_heterogeneity}. Most importantly, the repeated \textsc{DualKL}-over-single-KL pattern is consistent with the complementarity claim in Sec.~\ref{subsec:why_both}: global-only control is insufficient for isolating within-condition information (Appendix~\ref{app:thm2}, Theorems~\ref{thm:insufficient}--\ref{thm:productspace}), while purely conditional control is insufficient for constraining total global capacity (Theorems~\ref{thm:conditionalinsufficient}--\ref{thm:conditionalinsufficient_product}).

\paragraph{BloodMNIST (standard): PCA latent geometry.}
As shown in Fig.~\ref{fig:app_vis_bloodmnist_pca}, the PCA panels indicate a clear compactness--separation improvement under KL regularization. \stdkl{} and \condkl{} already improve over \noreg{}, while \dualkl{} gives the strongest overall geometric summary in this regime.

\begin{figure}[H]
\centering
\setlength{\tabcolsep}{1pt}
\begin{tabular}{cccc}
\includegraphics[width=0.24\linewidth]{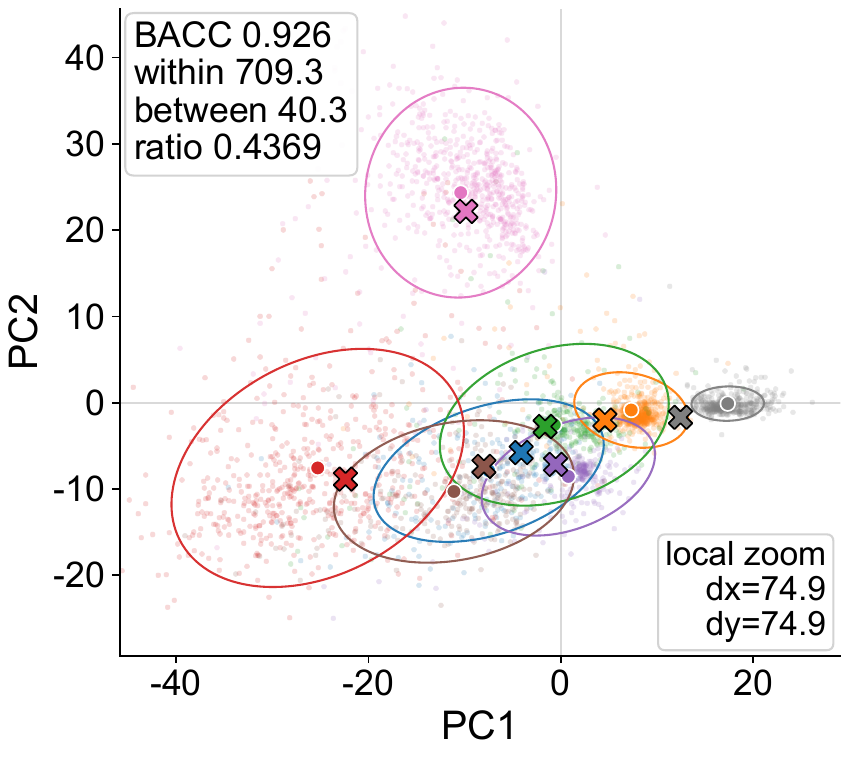} &
\includegraphics[width=0.24\linewidth]{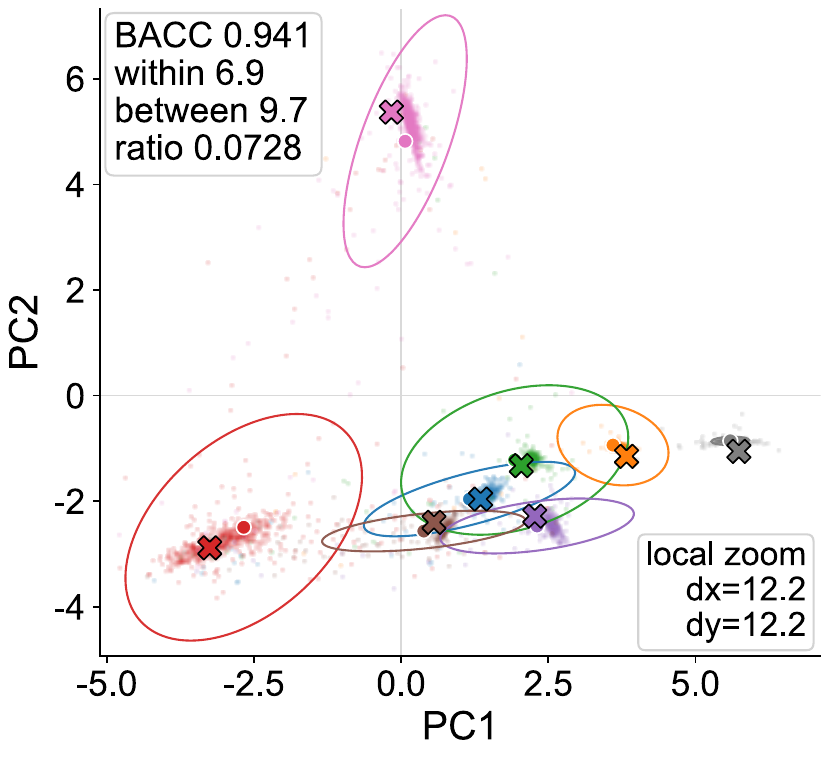} &
\includegraphics[width=0.24\linewidth]{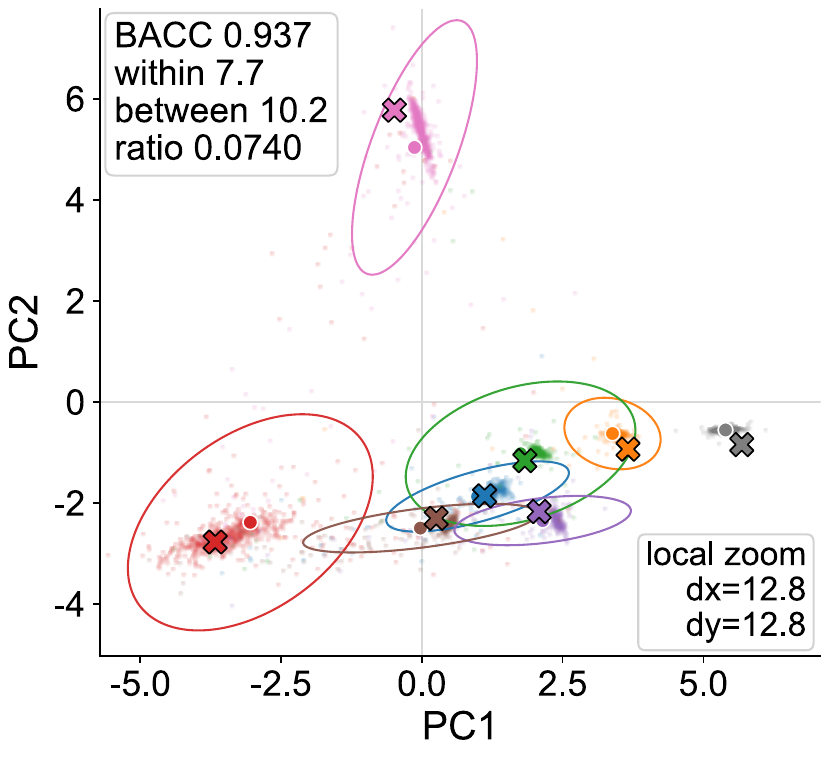} &
\includegraphics[width=0.24\linewidth]{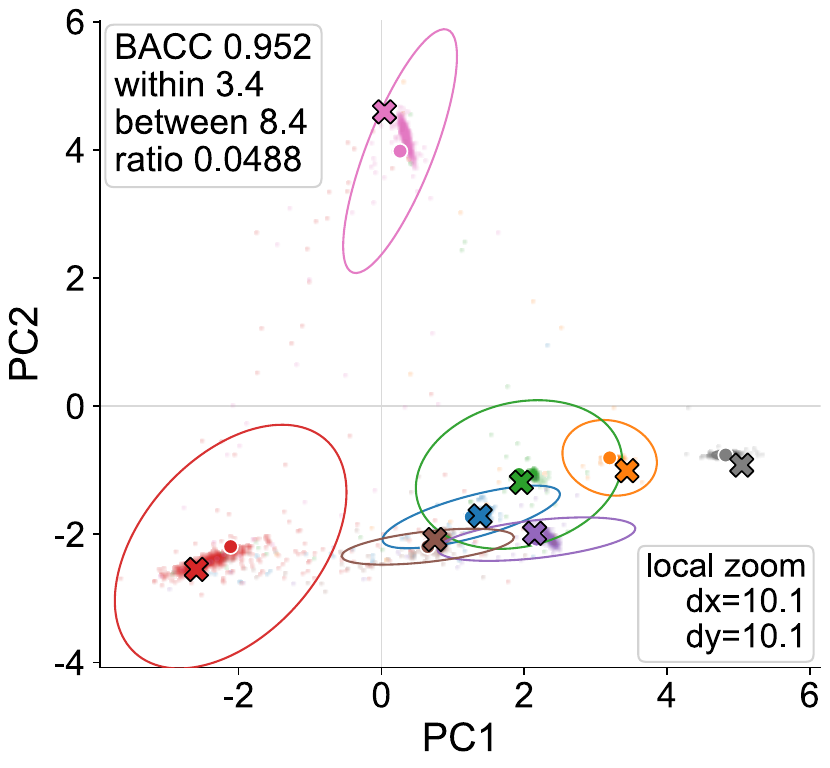} \\
\scriptsize \noreg{} & \scriptsize \stdkl{} & \scriptsize \condkl{} & \scriptsize \dualkl{}
\end{tabular}
\caption{BloodMNIST (standard), PCA latent geometry. Representative panel numbers: \noreg{} (BAcc $0.926$, \textit{Within} $709.3$, \textit{Ratio} $0.4369$), \dualkl{} (BAcc $0.952$, \textit{Within} $3.4$, \textit{Ratio} $0.0488$).}
\label{fig:app_vis_bloodmnist_pca}
\end{figure}

\paragraph{BloodMNIST (standard): latent-statistics geometry.}
In Fig.~\ref{fig:app_vis_bloodmnist_stats}, the latent-dimension population is strongly shifted away from high inter-sample separation and toward higher label separation. This is consistent with lower label-unexplained dispersion and stronger condition-level structure under the proxy interpretation in Appendix~\ref{app:vis_proxy}.

\begin{figure}[H]
\centering
\setlength{\tabcolsep}{1pt}
\begin{tabular}{cccc}
\includegraphics[width=0.24\linewidth]{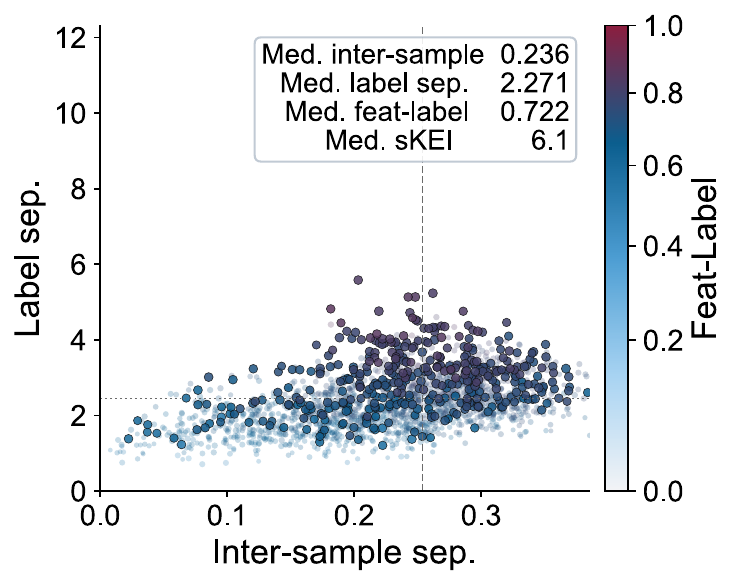} &
\includegraphics[width=0.24\linewidth]{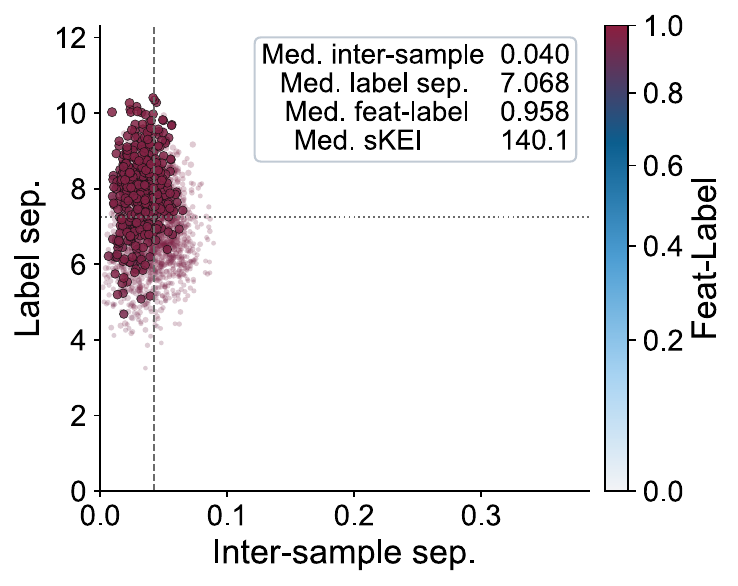} &
\includegraphics[width=0.24\linewidth]{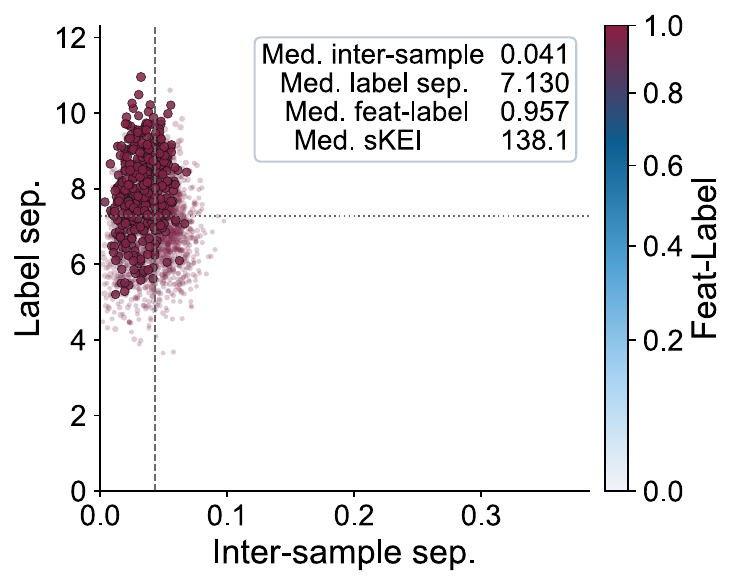} &
\includegraphics[width=0.24\linewidth]{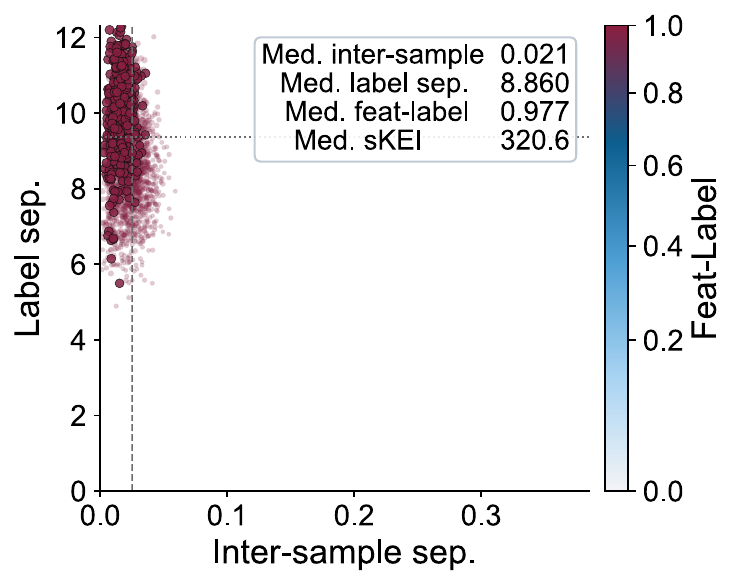} \\
\scriptsize \noreg{} & \scriptsize \stdkl{} & \scriptsize \condkl{} & \scriptsize \dualkl{}
\end{tabular}
\caption{BloodMNIST (standard), latent-statistics geometry. Representative panel medians: \noreg{} ($0.236$, $2.271$, $0.722$, $6.1$) vs \dualkl{} ($0.021$, $8.860$, $0.977$, $320.6$) for (inter-sample, label-sep, feat-label, sKEI).}
\label{fig:app_vis_bloodmnist_stats}
\end{figure}

\paragraph{BloodMNIST (1\% training data): PCA latent geometry.}
The low-data PCA trend in Fig.~\ref{fig:app_vis_bloodmnist1pct_pca} remains consistent with the standard-data regime. \stdkl{} and \condkl{} already improve over \noreg{}, but \dualkl{} provides the strongest combined geometric summary in this low-supervision setting.

\begin{figure}[H]
\centering
\setlength{\tabcolsep}{1pt}
\begin{tabular}{cccc}
\includegraphics[width=0.24\linewidth]{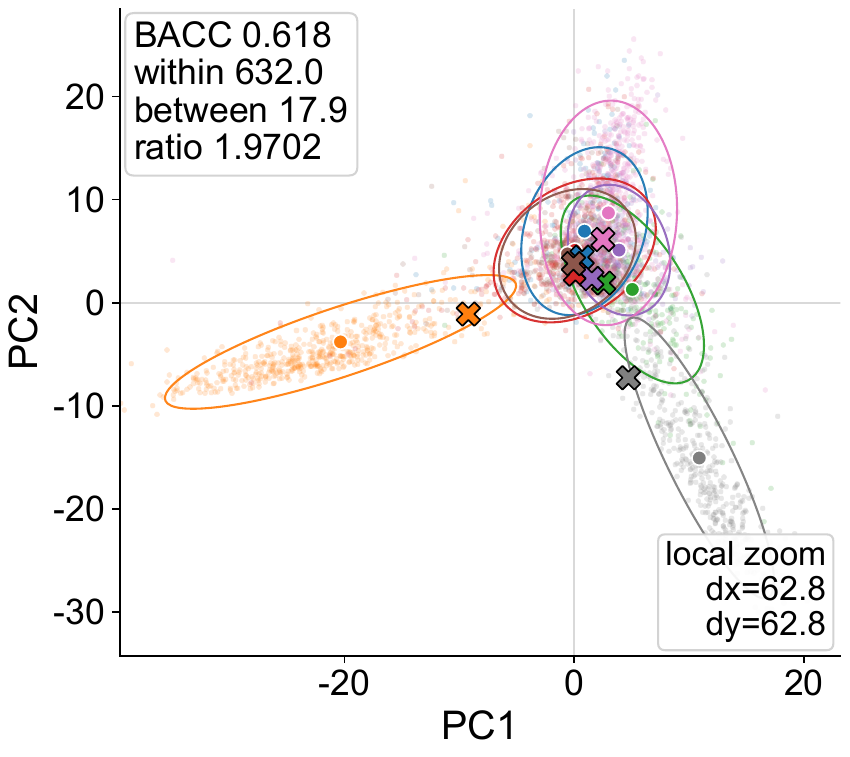} &
\includegraphics[width=0.24\linewidth]{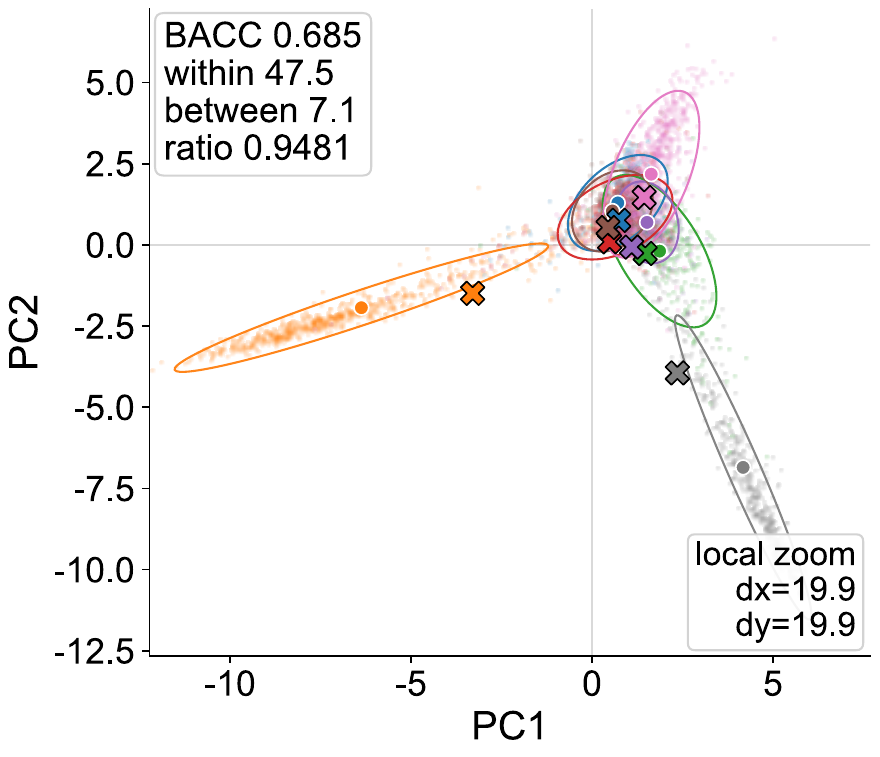} &
\includegraphics[width=0.24\linewidth]{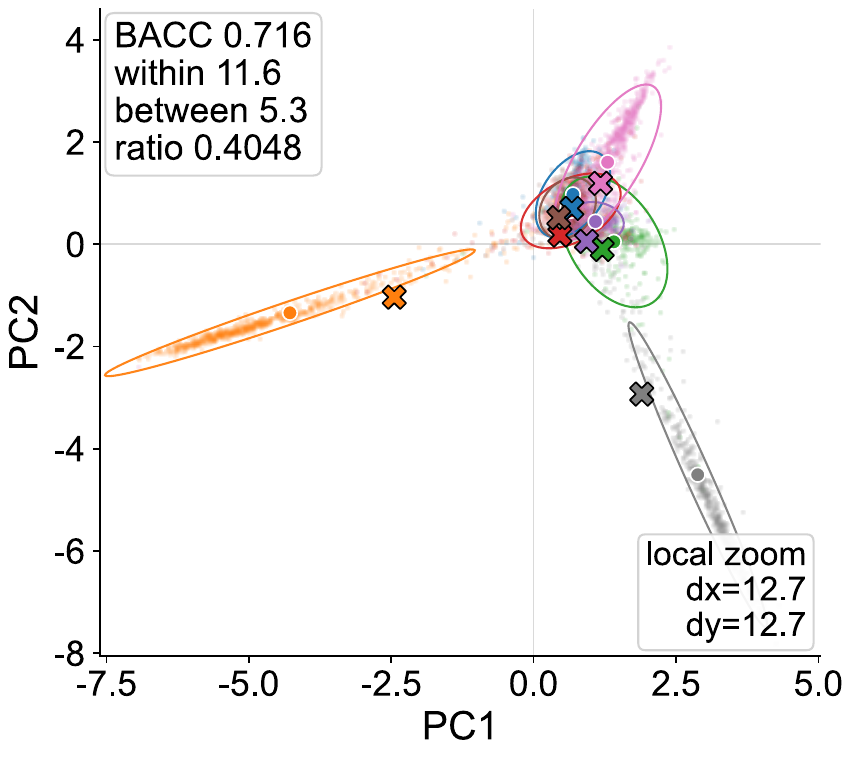} &
\includegraphics[width=0.24\linewidth]{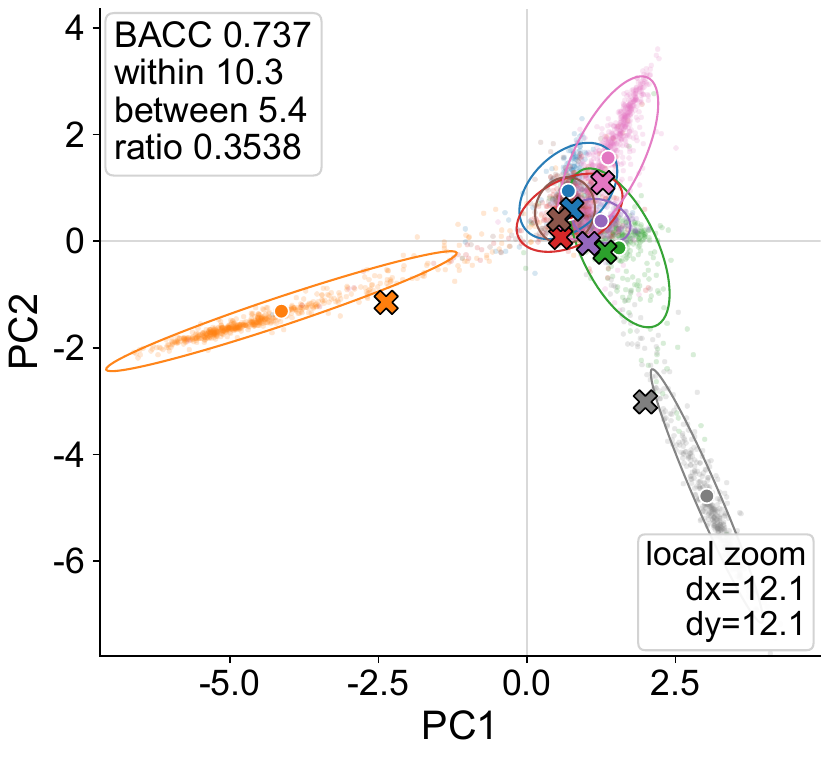} \\
\scriptsize \noreg{} & \scriptsize \stdkl{} & \scriptsize \condkl{} & \scriptsize \dualkl{}
\end{tabular}
\caption{BloodMNIST (1\% training data), PCA latent geometry. Representative panel numbers: \noreg{} (BAcc $0.618$, \textit{Within} $632.0$, \textit{Ratio} $1.9702$), \dualkl{} (BAcc $0.737$, \textit{Within} $10.3$, \textit{Ratio} $0.3538$).}
\label{fig:app_vis_bloodmnist1pct_pca}
\end{figure}

\paragraph{BloodMNIST (1\% training data): latent-statistics geometry.}
In Fig.~\ref{fig:app_vis_bloodmnist1pct_stats}, latent-statistics coordinates show a clear population-level redistribution toward a more favorable upper-left regime. This pattern is consistent with the larger low-data gains reported in Tab.~\ref{tab:app_bloodmnist_results}.

\begin{figure}[H]
\centering
\setlength{\tabcolsep}{1pt}
\begin{tabular}{cccc}
\includegraphics[width=0.24\linewidth]{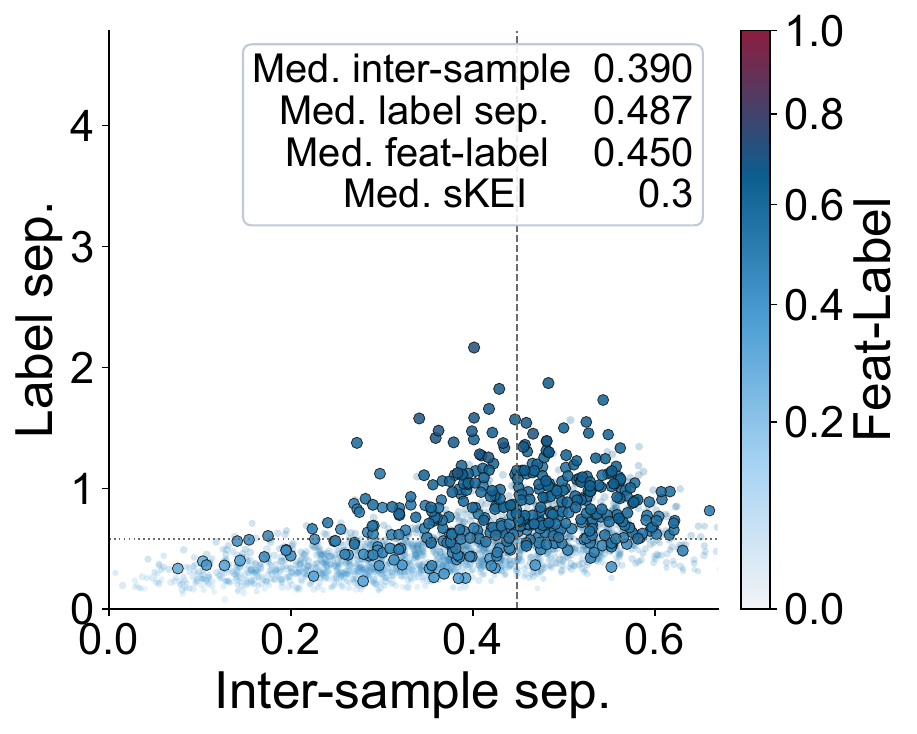} &
\includegraphics[width=0.24\linewidth]{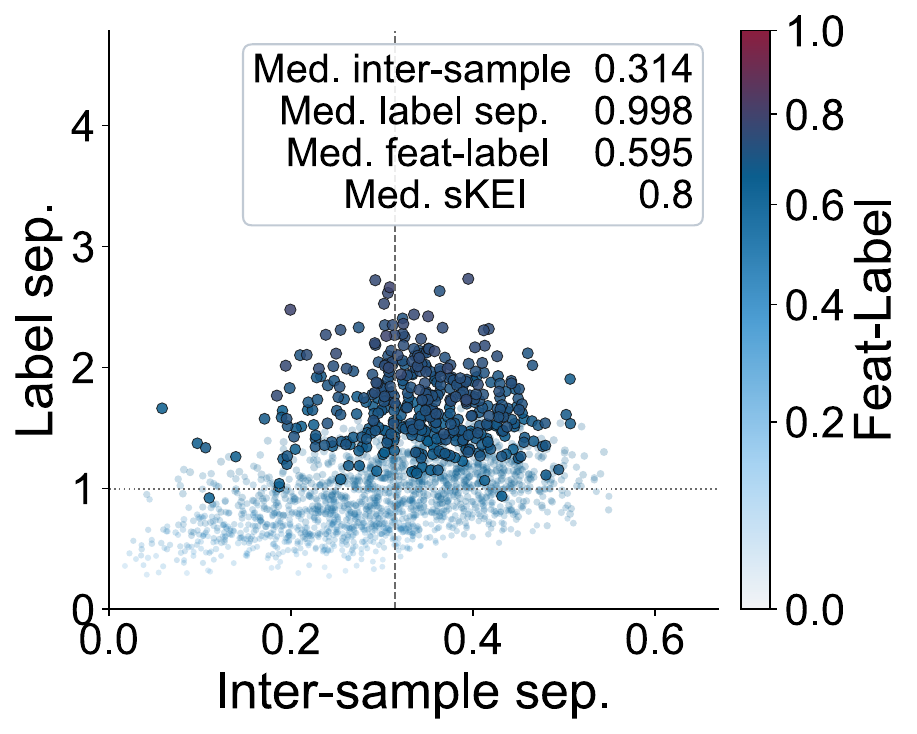} &
\includegraphics[width=0.24\linewidth]{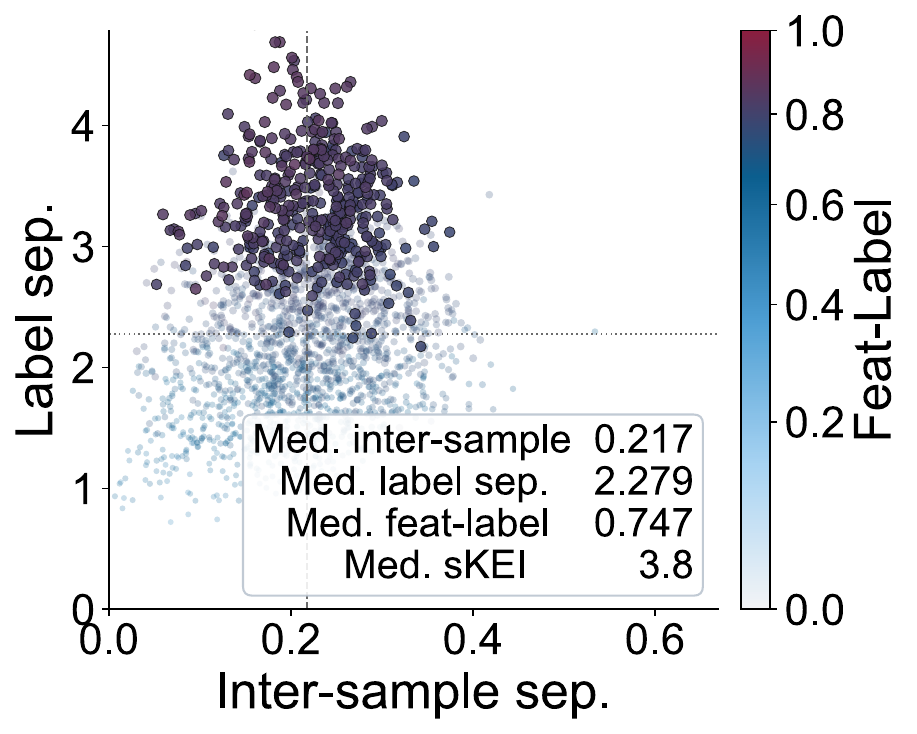} &
\includegraphics[width=0.24\linewidth]{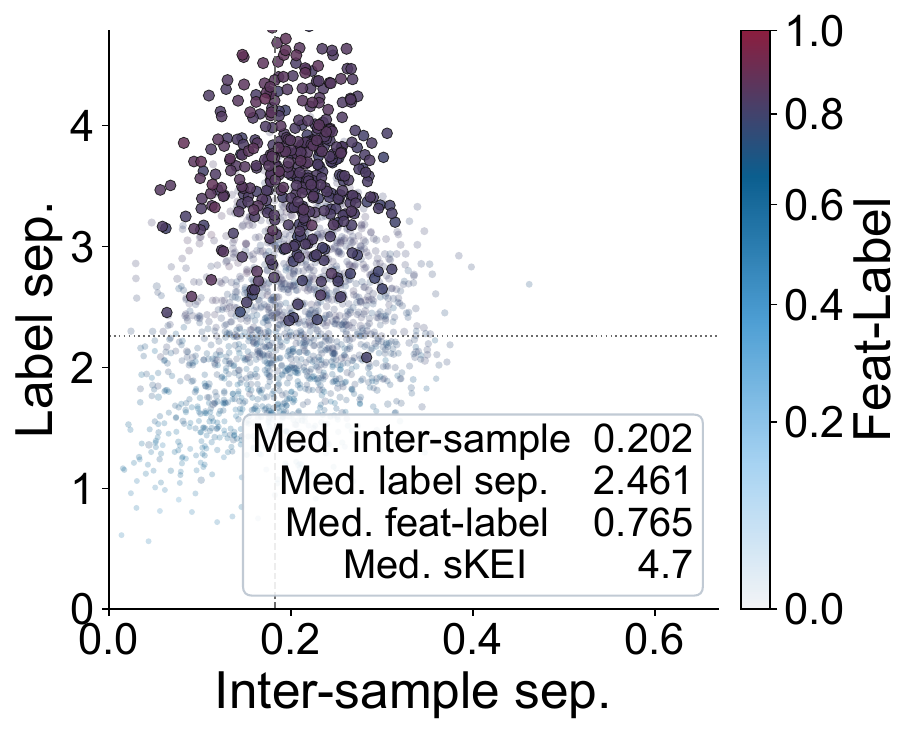} \\
\scriptsize \noreg{} & \scriptsize \stdkl{} & \scriptsize \condkl{} & \scriptsize \dualkl{}
\end{tabular}
\caption{BloodMNIST (1\% training data), latent-statistics geometry. Representative panel medians: \noreg{} ($0.390$, $0.487$, $0.450$, $0.3$) vs \dualkl{} ($0.202$, $2.461$, $0.765$, $4.7$) for (inter-sample, label-sep, feat-label, sKEI).}
\label{fig:app_vis_bloodmnist1pct_stats}
\end{figure}

\paragraph{CHASEDB1: PCA latent geometry.}
Figure~\ref{fig:app_vis_chasedb1_pca} shows the same compactness--separation direction. \textsc{DualKL} gives the strongest compactness--separation trade-off in this regime.

\begin{figure}[H]
\centering
\setlength{\tabcolsep}{1pt}
\begin{tabular}{cccc}
\includegraphics[width=0.24\linewidth]{figs/chasedb1/embedgeo/chasedb1_core_chasedb1_full_v1_none.pdf} &
\includegraphics[width=0.24\linewidth]{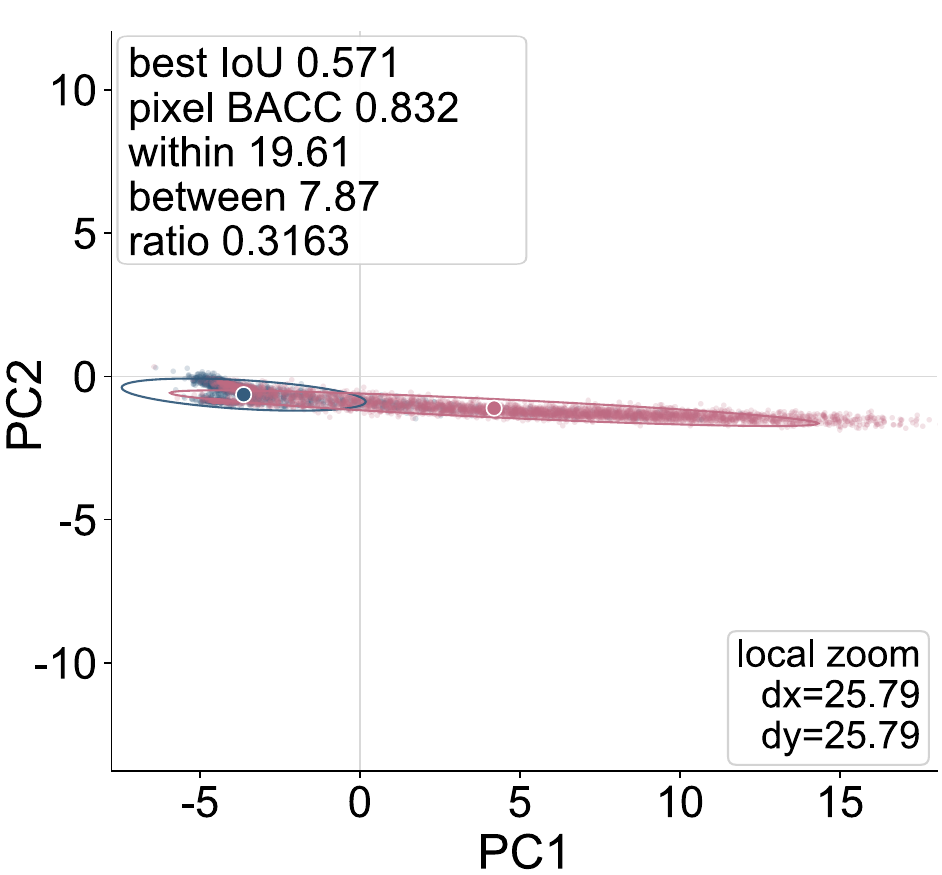} &
\includegraphics[width=0.24\linewidth]{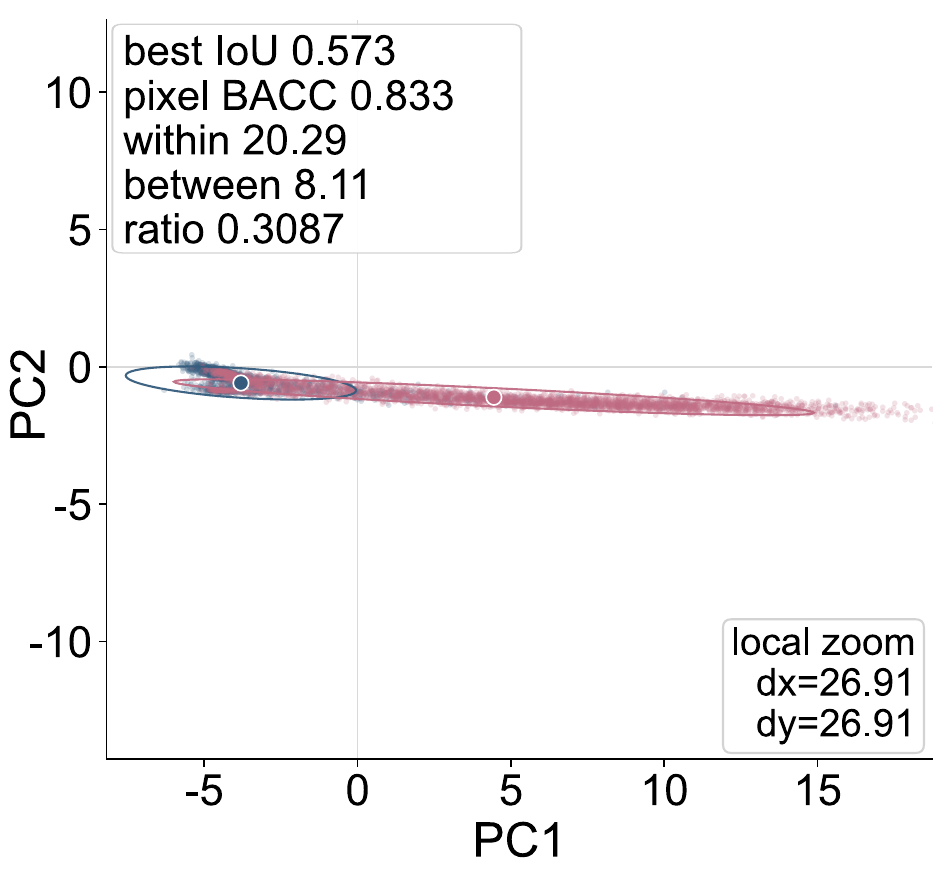} &
\includegraphics[width=0.24\linewidth]{figs/chasedb1/embedgeo/chasedb1_core_chasedb1_full_v1_stdkl_condkl.pdf} \\
\scriptsize \noreg{} & \scriptsize \stdkl{} & \scriptsize \condkl{} & \scriptsize \dualkl{}
\end{tabular}
\caption{CHASEDB1, PCA latent geometry. Representative panel numbers: \noreg{} (IoU $0.557$, \textit{Within} $19.75$, \textit{Ratio} $0.4314$), \dualkl{} (IoU $0.574$, \textit{Within} $18.38$, \textit{Ratio} $0.3082$).}
\label{fig:app_vis_chasedb1_pca}
\end{figure}

\paragraph{CHASEDB1: latent-statistics geometry.}
In Fig.~\ref{fig:app_vis_chasedb1_stats}, KL regularization again moves features away from the \noreg{} regime toward a higher-separation and higher-association region, confirming consistent population-level reorganization.

\begin{figure}[H]
\centering
\setlength{\tabcolsep}{1pt}
\begin{tabular}{cccc}
\includegraphics[width=0.24\linewidth]{figs/chasedb1/latentgeo/chasedb1_core_chasedb1_full_v1_none.pdf} &
\includegraphics[width=0.24\linewidth]{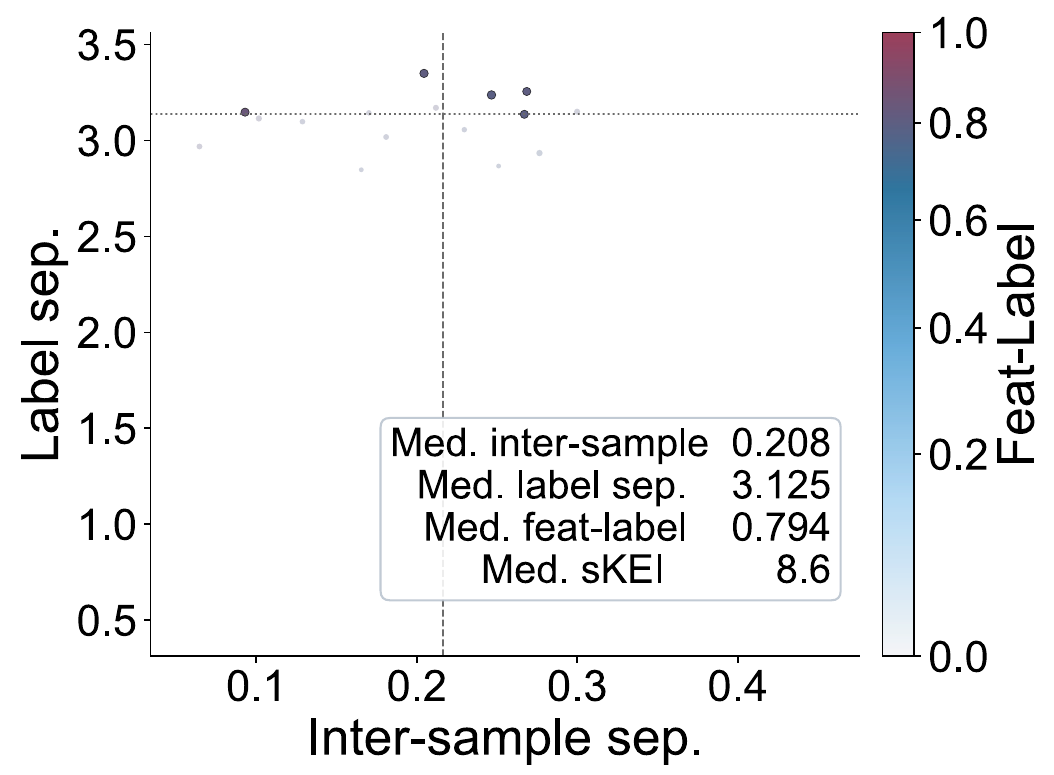} &
\includegraphics[width=0.24\linewidth]{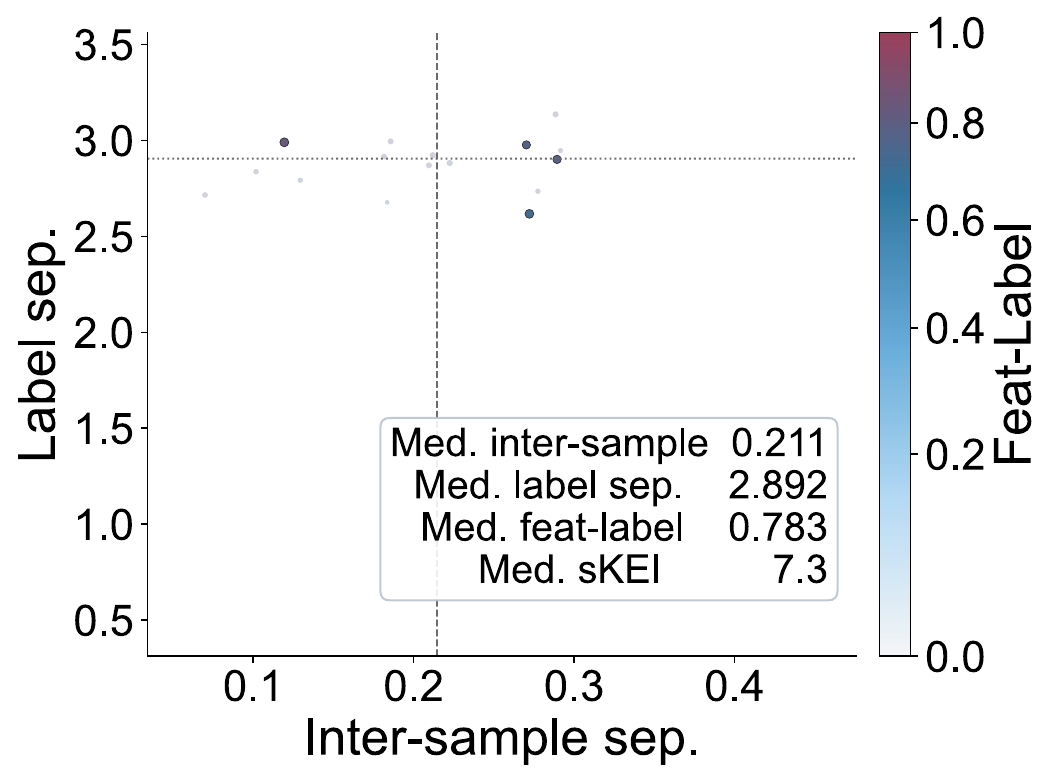} &
\includegraphics[width=0.24\linewidth]{figs/chasedb1/latentgeo/chasedb1_core_chasedb1_full_v1_stdkl_condkl.pdf} \\
\scriptsize \noreg{} & \scriptsize \stdkl{} & \scriptsize \condkl{} & \scriptsize \dualkl{}
\end{tabular}
\caption{CHASEDB1, latent-statistics geometry. Panel medians show \noreg{} ($0.343$, $1.909$, $0.623$, $2.6$) vs KL ranges: inter-sample $0.195$--$0.211$, label-sep $2.892$--$3.125$, feat-label $0.783$--$0.794$, sKEI $7.3$--$8.9$.}
\label{fig:app_vis_chasedb1_stats}
\end{figure}

\paragraph{STARE: PCA latent geometry.}
As shown in Fig.~\ref{fig:app_vis_stare_pca}, KL-regularized models reduce the elongated \noreg{} spread and improve the compactness--separation summary. This agrees with the separation-oriented PCA trend while pixel-BAcc remains in a narrow range.

\begin{figure}[H]
\centering
\setlength{\tabcolsep}{1pt}
\begin{tabular}{cccc}
\includegraphics[width=0.24\linewidth]{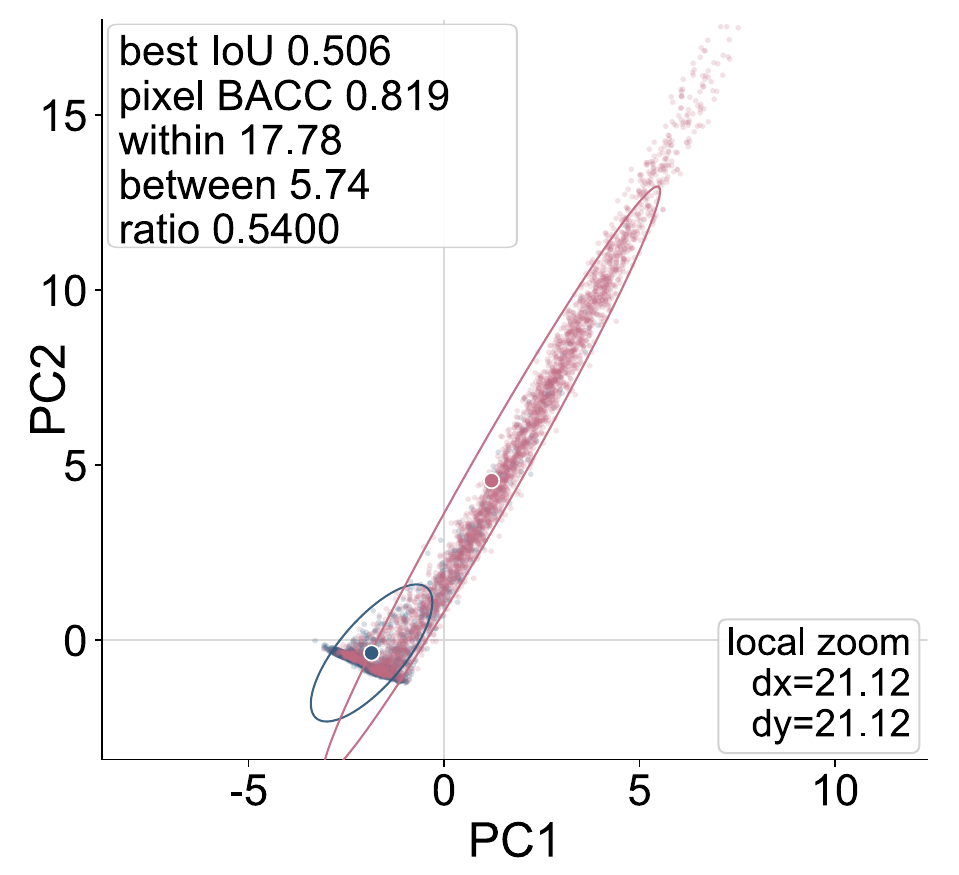} &
\includegraphics[width=0.24\linewidth]{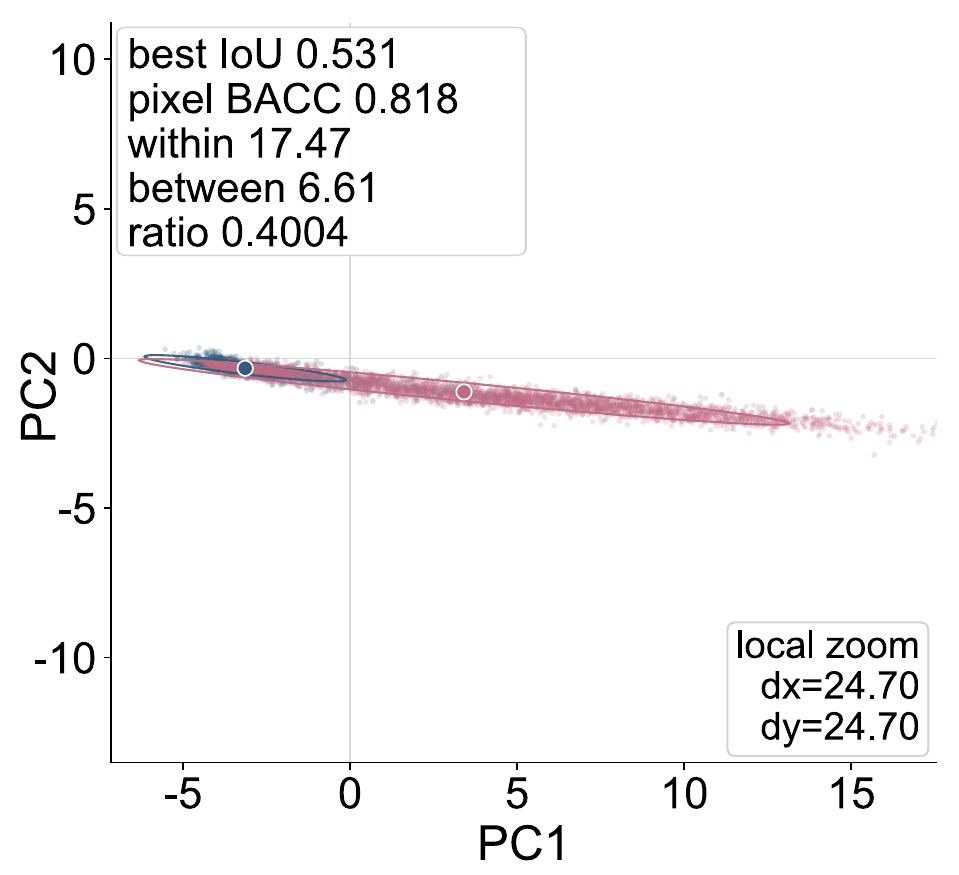} &
\includegraphics[width=0.24\linewidth]{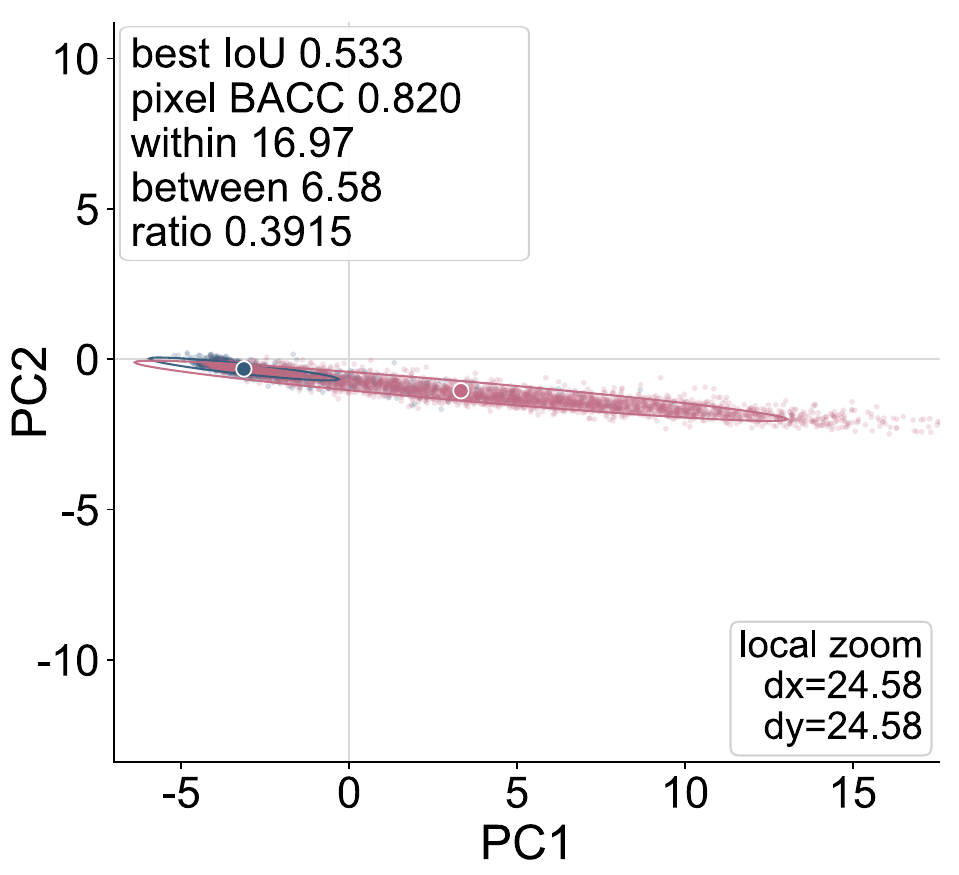} &
\includegraphics[width=0.24\linewidth]{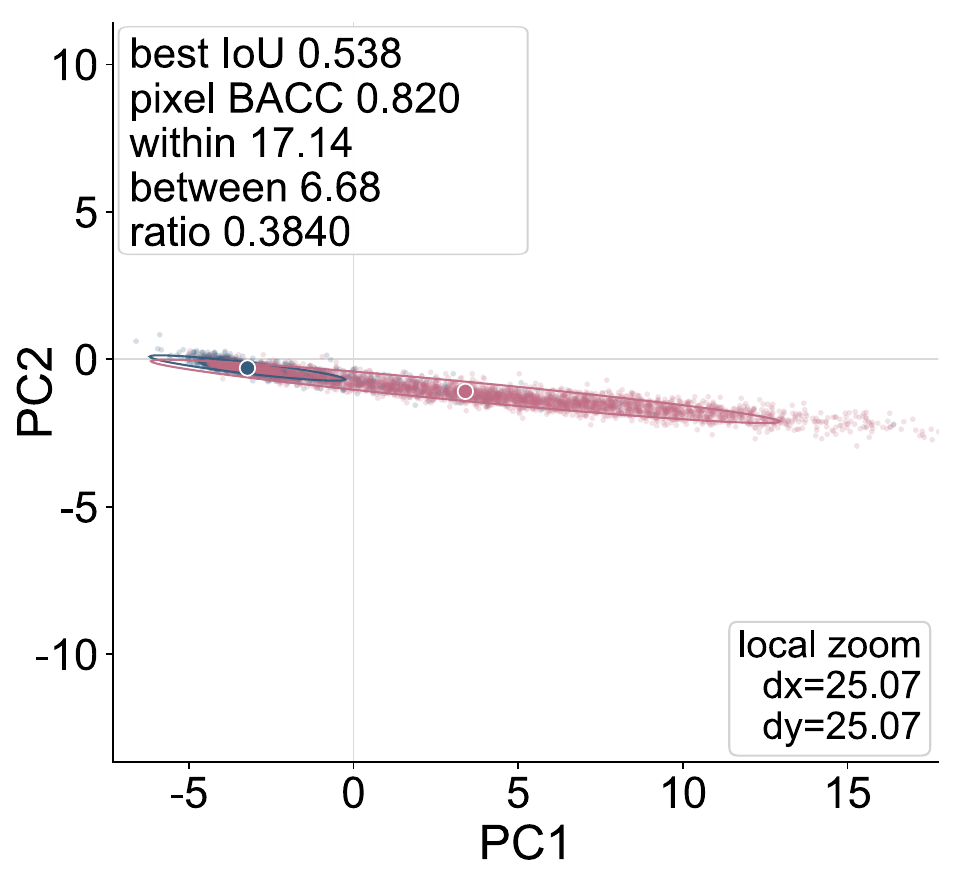} \\
\scriptsize \noreg{} & \scriptsize \stdkl{} & \scriptsize \condkl{} & \scriptsize \dualkl{}
\end{tabular}
\caption{STARE, PCA latent geometry. Representative panel numbers: \noreg{} (IoU $0.506$, \textit{Ratio} $0.5400$), \dualkl{} (IoU $0.538$, \textit{Ratio} $0.3840$).}
\label{fig:app_vis_stare_pca}
\end{figure}

\paragraph{STARE: latent-statistics geometry.}
In Fig.~\ref{fig:app_vis_stare_stats}, all KL settings move far from the \noreg{} low-separation regime. Among the KL variants, \condkl{} and \dualkl{} occupy the strongest overall feat-label/sKEI regime.

\begin{figure}[H]
\centering
\setlength{\tabcolsep}{1pt}
\begin{tabular}{cccc}
\includegraphics[width=0.24\linewidth]{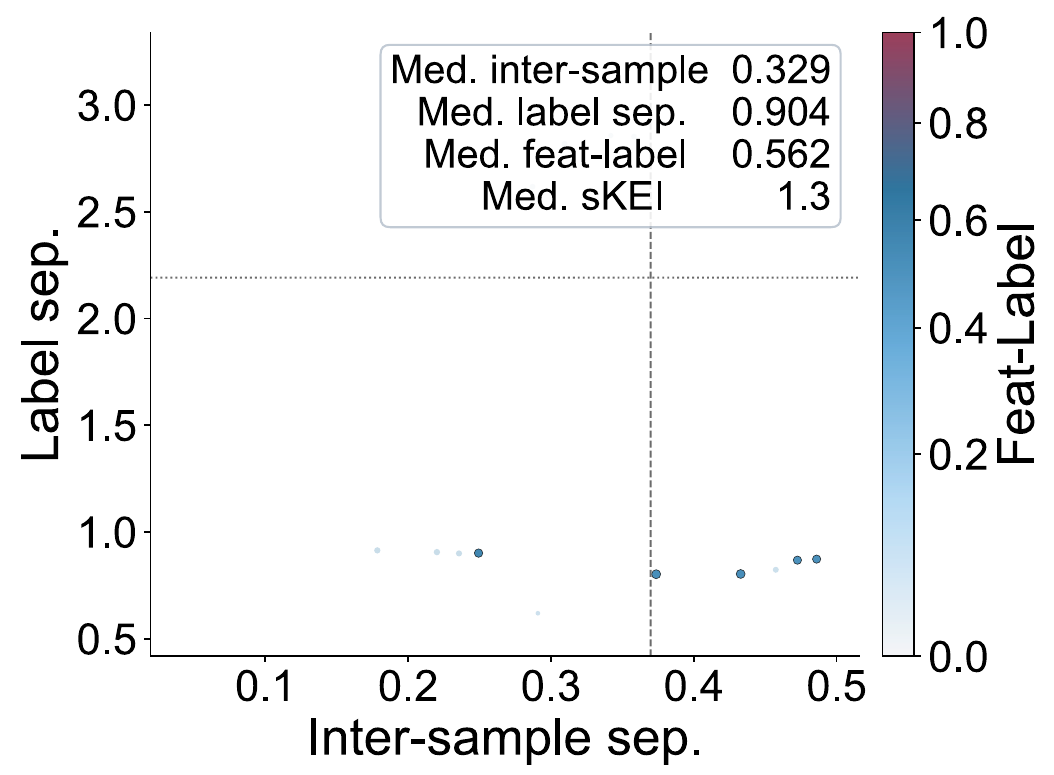} &
\includegraphics[width=0.24\linewidth]{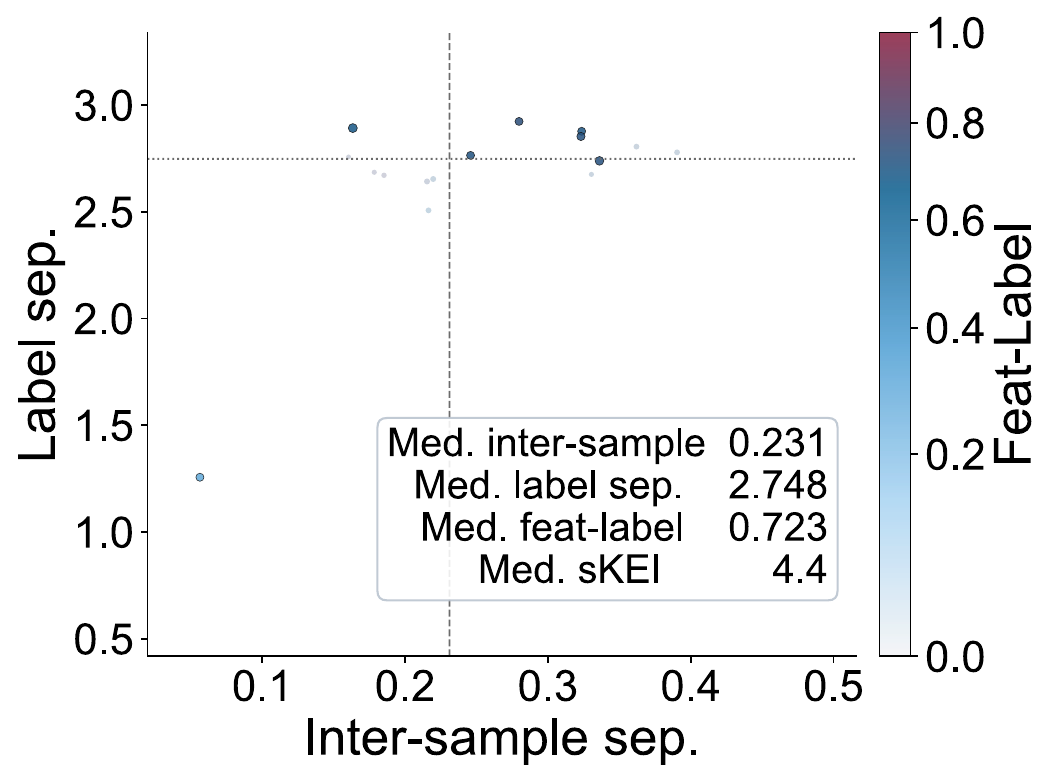} &
\includegraphics[width=0.24\linewidth]{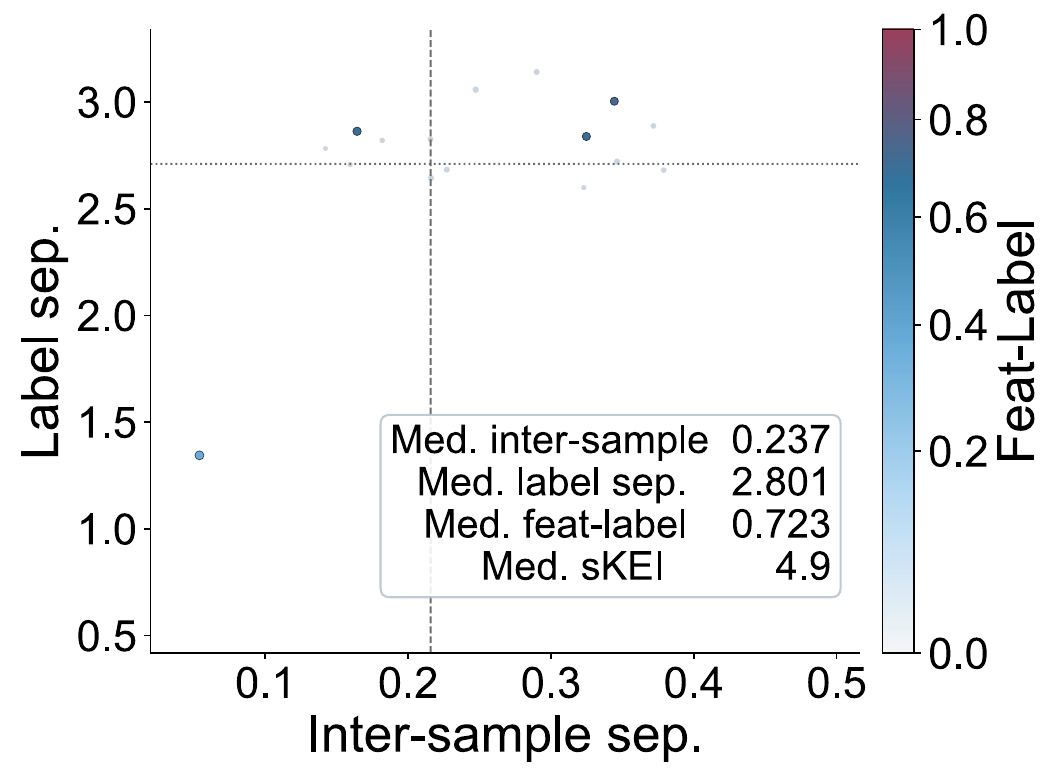} &
\includegraphics[width=0.24\linewidth]{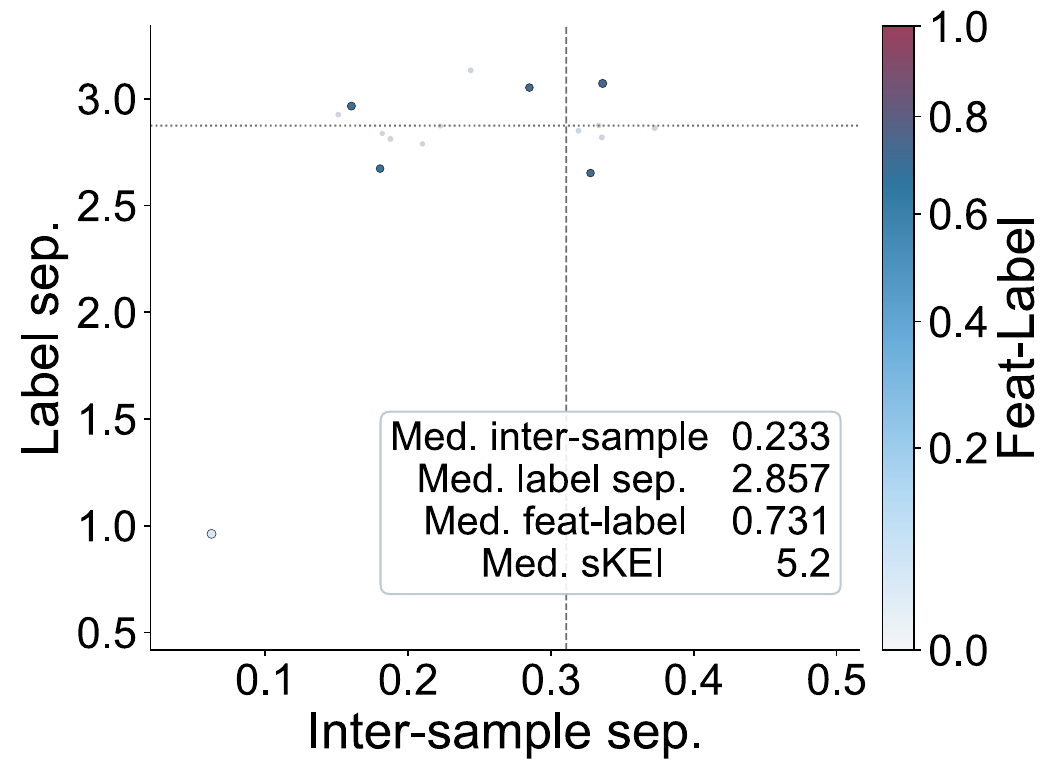} \\
\scriptsize \noreg{} & \scriptsize \stdkl{} & \scriptsize \condkl{} & \scriptsize \dualkl{}
\end{tabular}
\caption{STARE, latent-statistics geometry. Panel medians show the \noreg{} to KL-regularized transition: inter-sample $0.329\rightarrow0.231$--$0.237$, label-sep $0.904\rightarrow2.748$--$2.857$, feat-label $0.562\rightarrow0.723$--$0.731$, sKEI $1.3\rightarrow4.4$--$5.2$.}
\label{fig:app_vis_stare_stats}
\end{figure}

\subsection{Effect of the Label-Induced Partition}
\label{app:partition_choice}

\paragraph{Dimensionality for high-\(K\) partitions.}
The main segmentation experiments use the foreground-count condition by default, where
\(C_u=\sum_{v\in\mathcal N(u)}Y_v\) and \(K=10\). Thus the default segmentation setting only
requires \(d\ge 9\) for the exact simplex construction. The high-granularity pattern condition
in Tab.~\ref{tab:partition_choice_ablation} is an auxiliary partition-choice ablation rather than the default
segmentation configuration. Since the exact simplex prior requires \(d\ge K-1\), all variants in
this ablation use a matched final stochastic latent/reparameterization dimension \(d=512\), which
satisfies the requirement even for the pattern partition with \(K=512\). The backbone and
segmentation decoder are unchanged; only the final stochastic latent projection/reparameterization
layer is set to the matched dimension for this ablation. Therefore, the comparison in
Tab.~\ref{tab:partition_choice_ablation} varies the condition map \(\pi\) while keeping the non-latent
segmentation architecture matched across the three partition choices. These ablation numbers are
intended to compare feasible choices of \(\pi\), not to redefine the default segmentation protocol
used in the main objective-level comparisons.

Segmentation provides a natural setting where \(C=\pi(Y)\) can be instantiated beyond the identity label partition. In classification, \(Y\) is a
scalar class label and \(C=Y\) is the most direct choice. In segmentation,
however, \(Y\) is a dense mask: pixels with the same foreground/background label
may have different structural roles depending on whether they lie inside a
region, near a boundary, or on a thin structure.

We therefore define a pixel-wise condition from the local target neighborhood,
\[
C_u=\pi(Y_{\mathcal N(u)}),
\]
where \(u\) is a pixel location and \(\mathcal N(u)\) is its \(3\times3\)
ground-truth mask neighborhood. Table~\ref{tab:partition_definitions} summarizes
the three partitions considered in this ablation. Within this ablation, the task loss, backbone, data split,
optimizer, KL weights, and segmentation target are unchanged; only the condition
variable used in the conditional KL term is varied.

\begin{table}[h]
\centering
\caption{Segmentation condition constructions. For each pixel \(u\), the
condition \(C_u\) is induced from the ground-truth mask in a \(3\times3\)
neighborhood \(\mathcal N(u)\).}
\label{tab:partition_definitions}
\small
\setlength{\tabcolsep}{4.5pt}
\begin{tabular}{llcl}
\toprule
Partition & Definition & \(K\) & Interpretation \\
\midrule
Label &
\(C_u = Y_u\) &
2 &
Direct foreground/background pixel label. \\
Count &
\(C_u = \sum_{v\in\mathcal N(u)}Y_v\) &
10 &
Coarse local occupancy and boundary-proximity proxy. \\
Pattern &
\(C_u = \mathrm{bin}(Y_{\mathcal N(u)})\) &
512 &
Exact local shape, boundary, and thin-structure pattern. \\
\bottomrule
\end{tabular}
\end{table}

Tables~\ref{tab:coarse_partition_statistics} and
\ref{tab:pattern_partition_statistics} show that the induced conditions have
substantially different distributions. On the vessel datasets, the direct label
condition is strongly dominated by background pixels, while the count partition
reduces this concentration by separating pure-background, interior, and
boundary-neighborhood cases. The pattern partition exposes a much richer local
condition space, with hundreds of observed nontrivial patterns on CHASEDB1,
DRIVE, and STARE. This richer partition is also naturally long-tailed: many
boundary-related patterns occur sparsely. GLAS shows a different profile, with
large pure-background and pure-foreground masses, reflecting its larger interior
regions. Thus, increasing the granularity of \(C\) makes local structural
differences more explicit, but also introduces stronger sparsity and
dataset-dependent granularity effects.

\begin{table}[h]
\centering
\caption{Train-split concentration statistics for coarse segmentation
conditions. Top-1 mass is reported in percentage; \(H_{\mathrm{norm}}\) denotes
normalized entropy.}
\label{tab:coarse_partition_statistics}
\small
\setlength{\tabcolsep}{5pt}
\begin{tabular}{lcccc}
\toprule
Dataset
& \multicolumn{2}{c}{Label \(K=2\)}
& \multicolumn{2}{c}{Count \(K=10\)} \\
\cmidrule(lr){2-3}\cmidrule(lr){4-5}
& Top-1 & \(H_{\rm norm}\)
& Top-1 & \(H_{\rm norm}\) \\
\midrule
CHASEDB1 & 92.9 & 0.37 & 81.3 & 0.37 \\
DRIVE    & 91.4 & 0.42 & 75.8 & 0.44 \\
GLAS     & 50.3 & 1.00 & 46.5 & 0.45 \\
STARE    & 92.7 & 0.38 & 80.6 & 0.38 \\
\bottomrule
\end{tabular}
\end{table}

\begin{table}[h]
\centering
\caption{Train-split statistics for the exact \(3\times3\) pattern partition
(\(K=512\)). Nontrivial patterns exclude pure-background and pure-foreground
neighborhoods. Mass values are percentages.}
\label{tab:pattern_partition_statistics}
\small
\setlength{\tabcolsep}{4pt}
\begin{tabular}{lccccc}
\toprule
Dataset
& Obs. \(K\)
& Nontriv. \(K\)
& BG mass
& FG mass
& Nontriv. mass \\
\midrule
CHASEDB1 & 489 & 487 & 81.3 & 0.3  & 18.4 \\
DRIVE    & 508 & 506 & 75.8 & 0.3  & 23.9 \\
GLAS     & 269 & 267 & 46.5 & 46.4 & 7.2  \\
STARE    & 480 & 478 & 80.6 & 0.4  & 19.0 \\
\bottomrule
\end{tabular}
\end{table}

Table~\ref{tab:partition_choice_ablation} shows that the choice of \(\pi\)
affects the conditional bottleneck in practice. Non-trivial local partitions are
competitive with, and often better than, the direct pixel-label condition. The
pattern partition gives the best average IoU/Dice and performs best on GLAS and
STARE, whereas the coarser count partition is strongest on CHASEDB1 and DRIVE.
This dataset-dependent ranking is consistent with the statistics above:
fine-grained patterns expose richer local structure, but their long-tailed
distribution can make coarser partitions more stable in some vessel datasets.

These results support the use of the general notation \(C=\pi(Y)\). In dense
prediction, pixels with the same binary label can correspond to different local
structural roles, especially near boundaries or thin structures. The condition
variable is therefore not introduced only for mathematical generality; it can
encode target-induced local structure. The choice of \(\pi\) therefore serves as a practical design dimension of the proposed dual information-control framework.

\begin{table}[h!]
\centering
\caption{Partition-choice ablation for segmentation. We compare the direct
pixel label (\(K=2\)), \(3\times3\) foreground-count condition (\(K=10\)),
and exact binary local-mask pattern (\(K=512\)). Results are validation
IoU/Dice (\%).}
\label{tab:partition_choice_ablation}
\small
\setlength{\tabcolsep}{4.5pt}
\begin{tabular}{lccc}
\toprule
Dataset & Label \(K=2\) & Count \(K=10\) & Pattern \(K=512\) \\
\midrule
CHASEDB1 & 56.75 / 72.38 & \textbf{57.58 / 73.06} & 57.16 / 72.73 \\
DRIVE    & 58.27 / 73.63 & \textbf{58.29 / 73.64} & 58.13 / 73.51 \\
GLAS     & 81.83 / 89.89 & 81.93 / 89.93 & \textbf{82.25 / 90.18} \\
STARE    & 52.77 / 69.06 & 52.85 / 69.13 & \textbf{53.47 / 69.67} \\
\midrule
Average  & 62.40 / 76.24 & 62.66 / 76.44 & \textbf{62.75 / 76.52} \\
\bottomrule
\end{tabular}
\end{table}

\subsection{Sensitivity Experiment to Regularization Weights}
\label{app:beta_sensitivity}
We examine the sensitivity of the objective to the two regularization weights
$(\beta_{\rm std}, \beta_{\rm cond})$ on the full-data BloodMNIST setting.
All other training choices, including the backbone, data split, optimizer,
training schedule, simplex conditional prior, and shared fixed $\tau$, are kept
unchanged. The grid includes the unregularized point, the two single-axis edges,
and the full dual-regularization region.

Fig.~\ref{fig:beta_sensitivity_perf} reports the change in BAcc and Macro-F1
relative to the mean performance over the beta grid. The unregularized point is
below the grid mean, while several moderate regularization settings achieve
above-average performance. Very large weights do not uniformly improve the
scores, indicating that the gains are not obtained by simply increasing the
regularization strength. Overall, the sweep suggests that the reported behavior
is not driven by an isolated choice of
$(\beta_{\rm std}, \beta_{\rm cond})$.

\begin{figure}[h]
    \centering
    \vspace{-5pt}
    \begin{minipage}[t]{0.45\linewidth}
        \centering
        \includegraphics[width=\linewidth]{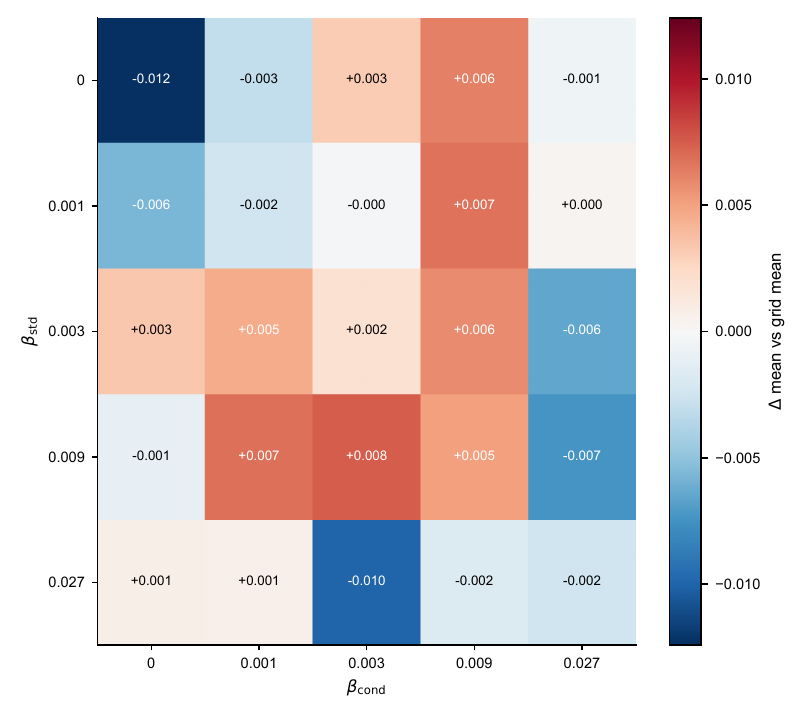}
        \vspace{-5mm}
        {\small (a) $\Delta$BAcc}
    \end{minipage}
    \hspace{0mm}
    \begin{minipage}[t]{0.45\linewidth}
        \centering
        \includegraphics[width=\linewidth]{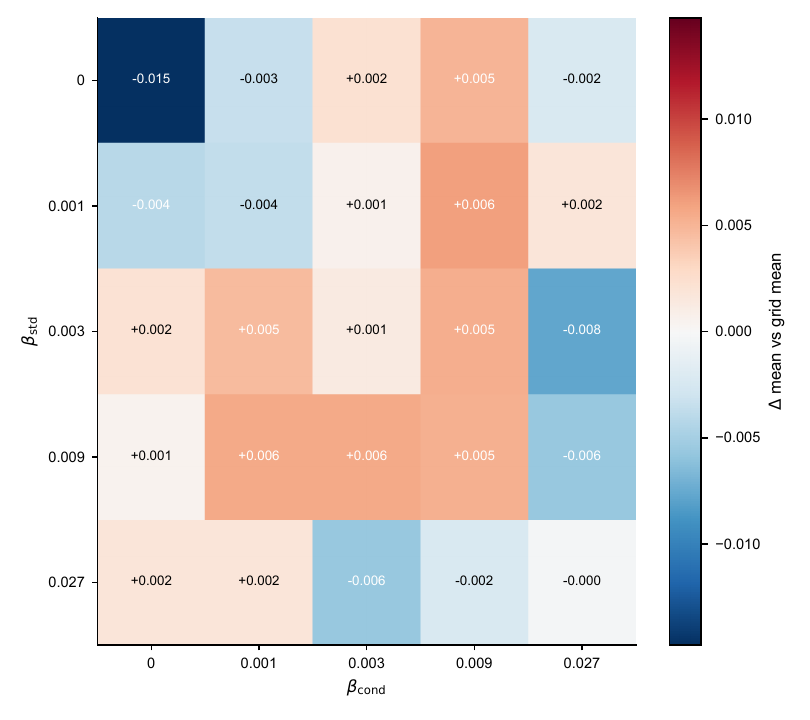}
        \vspace{-5mm}
        {\small (b) $\Delta$Macro-F1}
    \end{minipage}
    \vspace{8pt}
    \caption{
    Sensitivity to $(\beta_{\rm std}, \beta_{\rm cond})$ on full-data BloodMNIST.
    Each cell reports the change relative to the grid mean.
    }
    \label{fig:beta_sensitivity_perf}
\end{figure}
\FloatBarrier

\section{Literature review}

\paragraph{Relation to conditional bottlenecks and conditional priors.}
Our formulation is related to prior uses of variational bottlenecks,
conditional information quantities, and conditional latent priors. Deep VIB
\citep{alemi2017deep} provides a variational KL surrogate for global information
bottleneck regularization, and our standard term $J_{\rm std}$ follows this
global-capacity role. However, our motivation is different: under a
label-induced condition variable $C=\pi(Y)$, the global budget $I(Z;X)$
decomposes into condition-level structure $I(Z;C)$ and within-condition
information $I(Z;X\mid C)$, so a single global bottleneck does not specify how
the retained information is allocated.

Conditional bottleneck-style objectives have also appeared in top-venue
representation-learning work. Hwang et al.~\citep{hwang2021multiview} derive
conditional VIB-style terms from a total-correlation objective for multi-view
representation learning, where the conditions correspond to observed views
rather than supervised label-induced partitions. Lee et al.~\citep{lee2021compressive}
incorporate conditional information bottleneck-style  compression into SimCLR and
BYOL to improve self-supervised visual representations and robustness. These
works show that conditional information control is a meaningful design
principle, but their objectives target multi-view calibration or self-supervised
compression. In contrast, our condition variable is induced by supervised labels,
and the conditional KL is used to control within-condition information in a
supervised task representation.

Our use of a condition-dependent prior is also distinct from conditional
generative models such as CVAEs~\citep{sohn2015learning} and semi-supervised
deep generative models~\citep{kingma2014semi}. In those models, conditional
latent variables or priors are part of a generative likelihood or ELBO. In our
method, the conditional prior is not introduced for generative modelling; it is
a variational information-control prior inside a supervised regularizer. Thus,
the novelty is not the existence of a conditional prior alone, but the supervised
dual-bottleneck construction: $J_{\rm std}$ controls global capacity,
$J_{\rm cond}$ controls within-condition information through an exact
decomposition, and the simplex conditional prior provides a symmetric and
controllable geometric instantiation.

\begin{table}[h]
\centering
\caption{Comparison with related bottleneck and conditional-prior formulations.}
\label{tab:related_method_differences}
\small
\setlength{\tabcolsep}{4pt}
\begin{tabular}{p{3.6cm}p{8.1cm}}
\toprule
Related method & How our method differs \\
\midrule

Deep VIB \citep{alemi2017deep}
& Deep VIB provides a variational KL surrogate for a single global information
bottleneck, mainly controlling $I(Z;X)$. Our method keeps this global control
through $J_{\rm std}$, but adds a label-induced conditional control axis
$J_{\rm cond}$ motivated by
$I(Z;X)=I(Z;C)+I(Z;X\mid C)$. Thus, our objective separates global capacity
control from within-condition information control. \\
\midrule

Multi-view conditional VIB-style objectives
\citep{hwang2021multiview}
& These objectives use conditional information terms to calibrate shared
representations across multiple observed views and to handle view imbalance or
missing views. Our conditioning variable is not a view but a supervised
label-induced partition $C=\pi(Y)$, and our goal is not multi-view calibration
but supervised regularization of within-condition residual information. \\
\midrule

Self-supervised compressive visual representations~\citep{lee2021compressive}
& This line of work uses conditional information bottleneck-style objectives to compress
self-supervised representations, typically between augmented views. Our method
uses a supervised condition variable $C=\pi(Y)$ and combines conditional
information control with an explicit global KL term. Hence, the method is a
supervised dual-bottleneck regularizer rather than a self-supervised compression
objective. \\
\midrule

Conditional VAE / deep conditional generative models
\citep{sohn2015learning}
& Conditional VAEs use condition-dependent latent variables or priors inside a
generative likelihood or ELBO for structured output modelling. Our conditional
prior is not used for generation; it is a variational information-control prior
inside a supervised regularizer, where the conditional KL decomposes into
$I(Z;X\mid C)$ plus condition-wise prior mismatch. \\
\midrule

Semi-supervised deep generative models
\citep{kingma2014semi}
& Semi-supervised generative models use class variables and latent variables to
model data density and exploit unlabeled data. Our method does not model
$p(x)$ or require unlabeled data; it regularizes a supervised task representation
by jointly controlling global capacity and within-condition information. \\

\bottomrule
\end{tabular}
\end{table}

%%%%%%%%%%%%%%%%%%%%%%%%%%%%%%%%%%%%%%%%%%%%%%%%%%%%%%%%%%%%

\clearpage

%%%%%%%%%%%%%%%%%%%%%%%%%%%%%%%%%%%%%%%%%%%%%%%%%%%%%%%%%%%%

\end{document}

\maketitle

\begin{abstract}
In supervised representation learning, the Variational Information Bottleneck (VIB) is a commonly used regularization method that improves generalization by compressing the overall information bandwidth between inputs and latent representations. However, such uniform global compression does not explicitly distinguish task-relevant label-induced structure from unnecessary within-condition variation. To address this limitation, we propose a Conditional Information Bottleneck regularization framework as a complementary extension to standard VIB. Specifically, we formulate a dual-bottleneck objective consisting of a standard KL term and a conditional KL term, which jointly controls the overall representation capacity while further constraining information retained within each condition. Theoretically, we show that the conditional KL term admits an exact decomposition into a within-condition information term and a prior-mismatch term, providing an information-theoretic objective-level justification for the proposed mechanism. In addition, we introduce a structured conditional prior based on regular-simplex prototypes, which injects a symmetric and controllable separability bias into the latent space. Preliminary experiments suggest that the proposed dual-bottleneck structure yields promising regularization effects across multiple supervised vision tasks.
\end{abstract}

\section{Introduction}
\label{sec:intro}

Learning latent representations with strong generalization capabilities is a central objective in deep supervised learning. To this end, the Information Bottleneck (IB) principle and its variational form (VIB) have been widely adopted as a powerful regularization paradigm. The standard VIB limits the overall information bandwidth by minimizing $I(Z;X)$, which essentially performs a uniform and indiscriminate compression of the input information. However, in many supervised scenarios, such a ``one-size-fits-all'' compression is suboptimal. What we truly desire is to preserve the separability induced by a finite label-induced partition, while purposefully suppressing unnecessary within-condition variation.

A single global information bottleneck tends to conflate fundamentally different types of information: it captures both the structural information $I(Z;C)$ required to maintain label-induced separability and the within-condition information $I(Z;X\mid C)$, which may contain nuisance variation or memorization-heavy details. Consequently, relying solely on a single global bottleneck cannot explicitly distinguish these two types of information, and is therefore generally inadequate for independently controlling within-condition nuisance bandwidth.

To address this conceptual limitation, we define a condition variable $C=\phi(Y)$, which represents a finite partition of the input space induced by the labels. Based on this, we propose a dual-bottleneck regularization framework with a rigorous mathematical foundation. Under this framework, the standard KL divergence term is responsible for restricting the overall information bandwidth, while the conditional KL divergence term is used to more directly constrain information retention within each condition, thereby penalizing fine-grained variations that have no direct correspondence to label-induced separability. Furthermore, to stably implement this conditional regularization, we introduce a structured prior initialization scheme based on a regular simplex. Without directly relying on empirical data to construct class prototypes, this approach injects a symmetric and uniform separability bias into the latent space.

Theoretically, we first demonstrate that, for the design goal of preserving task separability while suppressing within-condition nuisance factors, a single scalar quantity $I(Z;X)$ is generally insufficient. More importantly, we show that the conditional KL divergence term is not merely a heuristic design: it can be exactly decomposed into the sum of a within-condition information term and a prior-mismatch term. Under a structural assumption in which task signals and nuisance factors are separable, this conditional information term further yields a tractable upper bound on nuisance dependence. Together, these results provide a clear objective-level theoretical justification for the proposed regularization objective.

The main contributions of this paper are summarized as follows:
\begin{itemize}
    \item \textbf{Dual-Bottleneck Framework:} We propose a dual information bottleneck regularization method based on label-induced partitions. By combining a standard KL term with a conditional KL term, the proposed framework enables decoupled yet synergistic control over global information bandwidth and within-condition redundant variation.
    \item \textbf{Theoretical Justification for Conditional Regularization:} We show that a single global IB term is generally insufficient for disentangling condition-level structure from within-condition nuisance variation. We further provide an exact decomposition theorem for the conditional KL term, and under suitable structural assumptions, establish its role as a tractable upper bound on nuisance dependence.
    \item \textbf{Structured Conditional Prior:} We introduce a prior initialization method based on regular-simplex geometry. This design injects a symmetric and controllable separability bias into the latent space, avoids the potential fragility of directly constructing priors from empirical class prototypes, and provides a more stable geometric template for optimization.
\end{itemize}

\section{Related Work}

\section{Method}

\section{Experiments}

\section{Conclusion}

\section{Submission of papers to NeurIPS 2025}

Please read the instructions below carefully and follow them faithfully.

\subsection{Style}

Papers to be submitted to NeurIPS 2025 must be prepared according to the
instructions presented here. Papers may only be up to {\bf nine} pages long,
including figures.
% Additional pages \emph{containing only acknowledgments and references} are allowed.
Additional pages \emph{containing references, checklist, and the optional technical appendices} do not count as content pages.
Papers that exceed the page limit will not be
reviewed, or in any other way considered for presentation at the conference.

The margins in 2025 are the same as those in previous years.

Authors are required to use the NeurIPS \LaTeX{} style files obtainable at the
NeurIPS website as indicated below. Please make sure you use the current files
and not previous versions. Tweaking the style files may be grounds for
rejection.

\subsection{Retrieval of style files}

The style files for NeurIPS and other conference information are available on
the website at
\begin{center}
  \url{https://neurips.cc}
\end{center}
The file \verb+neurips_2025.pdf+ contains these instructions and illustrates the
various formatting requirements your NeurIPS paper must satisfy.

The only supported style file for NeurIPS 2025 is \verb+neurips_2025.sty+,
rewritten for \LaTeXe{}.  \textbf{Previous style files for \LaTeX{} 2.09,
  Microsoft Word, and RTF are no longer supported!}

The \LaTeX{} style file contains three optional arguments: \verb+final+, which
creates a camera-ready copy, \verb+preprint+, which creates a preprint for
submission to, e.g., arXiv, and \verb+nonatbib+, which will not load the
\verb+natbib+ package for you in case of package clash.

\paragraph{Preprint option}
If you wish to post a preprint of your work online, e.g., on arXiv, using the
NeurIPS style, please use the \verb+preprint+ option. This will create a
nonanonymized version of your work with the text ``Preprint. Work in progress.''
in the footer. This version may be distributed as you see fit, as long as you do not say which conference it was submitted to. Please \textbf{do
  not} use the \verb+final+ option, which should \textbf{only} be used for
papers accepted to NeurIPS.

At submission time, please omit the \verb+final+ and \verb+preprint+
options. This will anonymize your submission and add line numbers to aid
review. Please do \emph{not} refer to these line numbers in your paper as they
will be removed during generation of camera-ready copies.

The file \verb+neurips_2025.tex+ may be used as a ``shell'' for writing your
paper. All you have to do is replace the author, title, abstract, and text of
the paper with your own.

The formatting instructions contained in these style files are summarized in
Sections \ref{gen_inst}, \ref{headings}, and \ref{others} below.

\section{General formatting instructions}
\label{gen_inst}

The text must be confined within a rectangle 5.5~inches (33~picas) wide and
9~inches (54~picas) long. The left margin is 1.5~inch (9~picas).  Use 10~point
type with a vertical spacing (leading) of 11~points.  Times New Roman is the
preferred typeface throughout, and will be selected for you by default.
Paragraphs are separated by \nicefrac{1}{2}~line space (5.5 points), with no
indentation.

The paper title should be 17~point, initial caps/lower case, bold, centered
between two horizontal rules. The top rule should be 4~points thick and the
bottom rule should be 1~point thick. Allow \nicefrac{1}{4}~inch space above and
below the title to rules. All pages should start at 1~inch (6~picas) from the
top of the page.

For the final version, authors' names are set in boldface, and each name is
centered above the corresponding address. The lead author's name is to be listed
first (left-most), and the co-authors' names (if different address) are set to
follow. If there is only one co-author, list both author and co-author side by
side.

Please pay special attention to the instructions in Section \ref{others}
regarding figures, tables, acknowledgments, and references.

\section{Headings: first level}
\label{headings}

All headings should be lower case (except for first word and proper nouns),
flush left, and bold.

First-level headings should be in 12-point type.

\subsection{Headings: second level}

Second-level headings should be in 10-point type.

\subsubsection{Headings: third level}

Third-level headings should be in 10-point type.

\paragraph{Paragraphs}

There is also a \verb+\paragraph+ command available, which sets the heading in
bold, flush left, and inline with the text, with the heading followed by 1\,em
of space.

\section{Citations, figures, tables, references}
\label{others}

These instructions apply to everyone.

\subsection{Citations within the text}

The \verb+natbib+ package will be loaded for you by default.  Citations may be
author/year or numeric, as long as you maintain internal consistency.  As to the
format of the references themselves, any style is acceptable as long as it is
used consistently.

The documentation for \verb+natbib+ may be found at
\begin{center}
  \url{http://mirrors.ctan.org/macros/latex/contrib/natbib/natnotes.pdf}
\end{center}
Of note is the command \verb+\citet+, which produces citations appropriate for
use in inline text.  For example,
\begin{verbatim}
   \citet{hasselmo} investigated\dots
\end{verbatim}
produces
\begin{quote}
  Hasselmo, et al.\ (1995) investigated\dots
\end{quote}

If you wish to load the \verb+natbib+ package with options, you may add the
following before loading the \verb+neurips_2025+ package:
\begin{verbatim}
   \PassOptionsToPackage{options}{natbib}
\end{verbatim}

If \verb+natbib+ clashes with another package you load, you can add the optional
argument \verb+nonatbib+ when loading the style file:
\begin{verbatim}
   \usepackage[nonatbib]{neurips_2025}
\end{verbatim}

As submission is double blind, refer to your own published work in the third
person. That is, use ``In the previous work of Jones et al.\ [4],'' not ``In our
previous work [4].'' If you cite your other papers that are not widely available
(e.g., a journal paper under review), use anonymous author names in the
citation, e.g., an author of the form ``A.\ Anonymous'' and include a copy of the anonymized paper in the supplementary material.

\subsection{Footnotes}

Footnotes should be used sparingly.  If you do require a footnote, indicate
footnotes with a number\footnote{Sample of the first footnote.} in the
text. Place the footnotes at the bottom of the page on which they appear.
Precede the footnote with a horizontal rule of 2~inches (12~picas).

Note that footnotes are properly typeset \emph{after} punctuation
marks.\footnote{As in this example.}

\subsection{Figures}

\begin{figure}
  \centering
  \fbox{\rule[-.5cm]{0cm}{4cm} \rule[-.5cm]{4cm}{0cm}}
  \caption{Sample figure caption.}
\end{figure}

All artwork must be neat, clean, and legible. Lines should be dark enough for
purposes of reproduction. The figure number and caption always appear after the
figure. Place one line space before the figure caption and one line space after
the figure. The figure caption should be lower case (except for first word and
proper nouns); figures are numbered consecutively.

You may use color figures.  However, it is best for the figure captions and the
paper body to be legible if the paper is printed in either black/white or in
color.

\subsection{Tables}

All tables must be centered, neat, clean and legible.  The table number and
title always appear before the table.  See Table~\ref{sample-table}.

Place one line space before the table title, one line space after the
table title, and one line space after the table. The table title must
be lower case (except for first word and proper nouns); tables are
numbered consecutively.

Note that publication-quality tables \emph{do not contain vertical rules.} We
strongly suggest the use of the \verb+booktabs+ package, which allows for
typesetting high-quality, professional tables:
\begin{center}
  \url{https://www.ctan.org/pkg/booktabs}
\end{center}
This package was used to typeset Table~\ref{sample-table}.

\begin{table}
  \caption{Sample table title}
  \label{sample-table}
  \centering
  \begin{tabular}{lll}
    \toprule
    \multicolumn{2}{c}{Part}                   \\
    \cmidrule(r){1-2}
    Name     & Description     & Size ($\mu$m) \\
    \midrule
    Dendrite & Input terminal  & $\sim$100     \\
    Axon     & Output terminal & $\sim$10      \\
    Soma     & Cell body       & up to $10^6$  \\
    \bottomrule
  \end{tabular}
\end{table}

\subsection{Math}
Note that display math in bare TeX commands will not create correct line numbers for submission. Please use LaTeX (or AMSTeX) commands for unnumbered display math. (You really shouldn't be using \$\$ anyway; see \url{https://tex.stackexchange.com/questions/503/why-is-preferable-to} and \url{https://tex.stackexchange.com/questions/40492/what-are-the-differences-between-align-equation-and-displaymath} for more information.)

\subsection{Final instructions}

Do not change any aspects of the formatting parameters in the style files.  In
particular, do not modify the width or length of the rectangle the text should
fit into, and do not change font sizes (except perhaps in the
\textbf{References} section; see below). Please note that pages should be
numbered.

\section{Preparing PDF files}

Please prepare submission files with paper size ``US Letter,'' and not, for
example, ``A4.''

Fonts were the main cause of problems in the past years. Your PDF file must only
contain Type 1 or Embedded TrueType fonts. Here are a few instructions to
achieve this.

\begin{itemize}

\item You should directly generate PDF files using \verb+pdflatex+.

\item You can check which fonts a PDF files uses.  In Acrobat Reader, select the
  menu Files$>$Document Properties$>$Fonts and select Show All Fonts. You can
  also use the program \verb+pdffonts+ which comes with \verb+xpdf+ and is
  available out-of-the-box on most Linux machines.

\item \verb+xfig+ "patterned" shapes are implemented with bitmap fonts.  Use
  "solid" shapes instead.

\item The \verb+\bbold+ package almost always uses bitmap fonts.  You should use
  the equivalent AMS Fonts:
\begin{verbatim}
   \usepackage{amsfonts}
\end{verbatim}
followed by, e.g., \verb+\mathbb{R}+, \verb+\mathbb{N}+, or \verb+\mathbb{C}+
for $\mathbb{R}$, $\mathbb{N}$ or $\mathbb{C}$.  You can also use the following
workaround for reals, natural and complex:
\begin{verbatim}
   \newcommand{\RR}{I\!\!R} %real numbers
   \newcommand{\Nat}{I\!\!N} %natural numbers
   \newcommand{\CC}{I\!\!\!\!C} %complex numbers
\end{verbatim}
Note that \verb+amsfonts+ is automatically loaded by the \verb+amssymb+ package.

\end{itemize}

If your file contains type 3 fonts or non embedded TrueType fonts, we will ask
you to fix it.

\subsection{Margins in \LaTeX{}}

Most of the margin problems come from figures positioned by hand using
\verb+\special+ or other commands. We suggest using the command
\verb+\includegraphics+ from the \verb+graphicx+ package. Always specify the
figure width as a multiple of the line width as in the example below:
\begin{verbatim}
   \usepackage[pdftex]{graphicx} ...
   \includegraphics[width=0.8\linewidth]{myfile.pdf}
\end{verbatim}
See Section 4.4 in the graphics bundle documentation
(\url{http://mirrors.ctan.org/macros/latex/required/graphics/grfguide.pdf})

A number of width problems arise when \LaTeX{} cannot properly hyphenate a
line. Please give LaTeX hyphenation hints using the \verb+\-+ command when
necessary.

\begin{ack}
Use unnumbered first level headings for the acknowledgments. All acknowledgments
go at the end of the paper before the list of references. Moreover, you are required to declare
funding (financial activities supporting the submitted work) and competing interests (related financial activities outside the submitted work).
More information about this disclosure can be found at: \url{https://neurips.cc/Conferences/2025/PaperInformation/FundingDisclosure}.

Do {\bf not} include this section in the anonymized submission, only in the final paper. You can use the \texttt{ack} environment provided in the style file to automatically hide this section in the anonymized submission.
\end{ack}

\section*{References}

References follow the acknowledgments in the camera-ready paper. Use unnumbered first-level heading for
the references. Any choice of citation style is acceptable as long as you are
consistent. It is permissible to reduce the font size to \verb+small+ (9 point)
when listing the references.
Note that the Reference section does not count towards the page limit.
\medskip

{
\small

[1] Alexander, J.A.\ \& Mozer, M.C.\ (1995) Template-based algorithms for
connectionist rule extraction. In G.\ Tesauro, D.S.\ Touretzky and T.K.\ Leen
(eds.), {\it Advances in Neural Information Processing Systems 7},
pp.\ 609--616. Cambridge, MA: MIT Press.

[2] Bower, J.M.\ \& Beeman, D.\ (1995) {\it The Book of GENESIS: Exploring
  Realistic Neural Models with the GEneral NEural SImulation System.}  New York:
TELOS/Springer--Verlag.

[3] Hasselmo, M.E., Schnell, E.\ \& Barkai, E.\ (1995) Dynamics of learning and
recall at excitatory recurrent synapses and cholinergic modulation in rat
hippocampal region CA3. {\it Journal of Neuroscience} {\bf 15}(7):5249-5262.
}

%%%%%%%%%%%%%%%%%%%%%%%%%%%%%%%%%%%%%%%%%%%%%%%%%%%%%%%%%%%%

\appendix

\section{Detailed Proofs for the Label-Induced Conditional Bottleneck}
\label{app:proofs}

This appendix provides detailed derivations for the mathematical claims used in the main method.
Our goal is not only to state the final objective, but also to make explicit:
(i) which statements are exact identities,
(ii) which statements are upper bounds,
(iii) which statements are existence/insufficiency results,
and (iv) which parts are modeling choices rather than universal optimality claims.

\subsection{Notation and proof roadmap}
\label{app:notation}

Let \(X \in \mathcal{X}\) be the input, \(Y \in \mathcal{Y}\) the supervised target, and
\(Z \in \mathbb{R}^d\) the stochastic latent variable generated by the encoder \(q_\theta(z\mid x)\).

We assume a \emph{label-induced finite condition}
\[
C = \phi(Y) = \phi(y^\star(X)),
\]
where \(y^\star : \mathcal{X}\to\mathcal{Y}\) is a deterministic target rule and
\(\phi : \mathcal{Y}\to\{1,\dots,C\}\) is a finite partition map.

The proof chain is organized as follows:
\begin{enumerate}
    \item Theorem~\ref{thm:decomp}: exact decomposition of \(I(Z;X)\) under a label-induced condition.
    \item Theorem~\ref{thm:insufficient}: a single global bottleneck is insufficient for separately controlling within-condition information.
    \item Theorem~\ref{thm:condkl}: exact decomposition of the conditional KL term.
    \item Proposition~\ref{prop:nuisance}: nuisance suppression interpretation under an optional structural assumption.
    \item Lemma~\ref{lem:stdkl}: exact decomposition of the standard KL term.
    \item Proposition~\ref{prop:simplex}: symmetric separability induced by the simplex prior.
    \item Proposition~\ref{prop:tau}: corrected minimization argument for the shared variance parameter \(\tau\).
\end{enumerate}

\subsection{Exact decomposition under a label-induced condition}
\label{app:thm1}

\begin{theorem}
\label{thm:decomp}
If \(C=\phi(y^\star(X))\), then
\[
I(Z;X)=I(Z;C)+I(Z;X\mid C).
\]
\end{theorem}

\begin{proof}
Start from the chain rule for mutual information:
\[
I(Z;X,C)=I(Z;C)+I(Z;X\mid C).
\]
On the other hand, the chain rule also gives
\[
I(Z;X,C)=I(Z;X)+I(Z;C\mid X).
\]
Since \(C\) is a deterministic function of \(X\), conditioning on \(X\) fully determines \(C\).
Therefore,
\[
I(Z;C\mid X)=0.
\]
Substituting this into the second identity yields
\[
I(Z;X,C)=I(Z;X).
\]
Combining the two expressions for \(I(Z;X,C)\), we obtain
\[
I(Z;X)=I(Z;C)+I(Z;X\mid C).
\]
This proves the claim.
\end{proof}

\paragraph{Interpretation.}
The decomposition is exact. It shows that the total information budget \(I(Z;X)\) mixes:
\[
\underbrace{I(Z;C)}_{\text{condition-level / label-induced structure}}
\quad\text{and}\quad
\underbrace{I(Z;X\mid C)}_{\text{within-condition information}}.
\]
Therefore, a single global information budget does not by itself distinguish these two roles.

\subsection{Why a single global bottleneck is insufficient}
\label{app:thm2}

\begin{theorem}
\label{thm:insufficient}
There exist two encoders \(q_1(z\mid x)\) and \(q_2(z\mid x)\) such that
\[
I_{q_1}(Z;X)=I_{q_2}(Z;X),
\qquad
I_{q_1}(Z;X\mid C)\neq I_{q_2}(Z;X\mid C).
\]
Hence, any regularizer depending only on \(I(Z;X)\) cannot, in general, separately control
within-condition information.
\end{theorem}

\begin{proof}
Let \(C\sim \mathrm{Bernoulli}(1/2)\) and \(N\sim \mathrm{Bernoulli}(1/2)\) be independent random variables.
Define the input
\[
X=(C,N).
\]
Now define two deterministic encoders:
\[
Z_1=C,\qquad Z_2=N.
\]

Since \(Z_1\) is a deterministic one-bit function of \(X\),
\[
I(Z_1;X)=H(Z_1)-H(Z_1\mid X)=H(Z_1)=1.
\]
Similarly, since \(Z_2\) is also a deterministic one-bit function of \(X\),
\[
I(Z_2;X)=H(Z_2)-H(Z_2\mid X)=H(Z_2)=1.
\]
Therefore,
\[
I(Z_1;X)=I(Z_2;X).
\]

We next compare the conditional mutual informations.

For \(Z_1=C\), once we condition on \(C\), the latent variable \(Z_1\) is fully determined.
Hence
\[
H(Z_1\mid C)=0,
\qquad
H(Z_1\mid X,C)=0,
\]
and therefore
\[
I(Z_1;X\mid C)=H(Z_1\mid C)-H(Z_1\mid X,C)=0.
\]

For \(Z_2=N\), independence of \(N\) and \(C\) gives
\[
H(Z_2\mid C)=H(N\mid C)=H(N)=1.
\]
Moreover, since \(Z_2=N\) is determined by \(X=(C,N)\),
\[
H(Z_2\mid X,C)=0.
\]
Thus
\[
I(Z_2;X\mid C)=H(Z_2\mid C)-H(Z_2\mid X,C)=1.
\]

We have shown
\[
I(Z_1;X)=I(Z_2;X)=1,
\qquad
I(Z_1;X\mid C)=0,\quad I(Z_2;X\mid C)=1.
\]
Hence two encoders can have the same global information budget but different within-condition
information. Therefore, a regularizer depending only on \(I(Z;X)\) cannot in general distinguish them.
\end{proof}

\paragraph{What this theorem proves.}
The theorem proves a necessity statement in the correct sense:
a \emph{condition-aware} term is mathematically necessary if the design goal is to preserve
label-induced separability while separately suppressing within-condition variation.

\paragraph{What this theorem does \emph{not} prove.}
It does not prove that the conditional KL term is the \emph{only} possible condition-aware choice.
That second step is addressed separately by the exact decomposition in
Theorem~\ref{thm:condkl}.

\subsection{Exact decomposition of the conditional KL term}
\label{app:thm3}

Define the condition-wise aggregated posterior
\[
q_\theta(z\mid c)
=
\int q_\theta(z\mid x)\,p(x\mid c)\,dx.
\]

\begin{theorem}
\label{thm:condkl}
For any encoder \(q_\theta(z\mid x)\) and any conditional prior \(r_\psi(z\mid c)\),
\[
\mathbb{E}_{p(x,c)}
\mathrm{KL}\!\bigl(q_\theta(z\mid x)\,\|\,r_\psi(z\mid c)\bigr)
=
I_\theta(Z;X\mid C)
+
\mathbb{E}_{p(c)}
\mathrm{KL}\!\bigl(q_\theta(z\mid c)\,\|\,r_\psi(z\mid c)\bigr).
\]
Consequently,
\[
I_\theta(Z;X\mid C)
\le
\mathbb{E}_{p(x,c)}
\mathrm{KL}\!\bigl(q_\theta(z\mid x)\,\|\,r_\psi(z\mid c)\bigr).
\]
\end{theorem}

\begin{proof}
Start from the definition of conditional KL:
\[
\mathbb{E}_{p(x,c)}
\mathrm{KL}\!\bigl(q_\theta(z\mid x)\,\|\,r_\psi(z\mid c)\bigr)
=
\mathbb{E}_{p(x,c)}\mathbb{E}_{q_\theta(z\mid x)}
\left[
\log\frac{q_\theta(z\mid x)}{r_\psi(z\mid c)}
\right].
\]
Add and subtract \(\log q_\theta(z\mid c)\) inside the logarithm:
\[
\log\frac{q_\theta(z\mid x)}{r_\psi(z\mid c)}
=
\log\frac{q_\theta(z\mid x)}{q_\theta(z\mid c)}
+
\log\frac{q_\theta(z\mid c)}{r_\psi(z\mid c)}.
\]
Therefore,
\begin{align*}
&\mathbb{E}_{p(x,c)}
\mathrm{KL}\!\bigl(q_\theta(z\mid x)\,\|\,r_\psi(z\mid c)\bigr) \\
&=
\mathbb{E}_{p(x,c)}\mathbb{E}_{q_\theta(z\mid x)}
\left[
\log\frac{q_\theta(z\mid x)}{q_\theta(z\mid c)}
\right]
+
\mathbb{E}_{p(x,c)}\mathbb{E}_{q_\theta(z\mid x)}
\left[
\log\frac{q_\theta(z\mid c)}{r_\psi(z\mid c)}
\right].
\end{align*}

The first term is exactly the conditional mutual information:
\[
\mathbb{E}_{p(x,c)}\mathbb{E}_{q_\theta(z\mid x)}
\left[
\log\frac{q_\theta(z\mid x)}{q_\theta(z\mid c)}
\right]
=
I_\theta(Z;X\mid C).
\]

For the second term, the integrand depends on \(x\) only through \(q_\theta(z\mid x)\).
Integrating out \(x\) conditional on \(c\) yields
\[
\mathbb{E}_{p(x,c)}\mathbb{E}_{q_\theta(z\mid x)}
\left[
\log\frac{q_\theta(z\mid c)}{r_\psi(z\mid c)}
\right]
=
\mathbb{E}_{p(c)}\mathbb{E}_{q_\theta(z\mid c)}
\left[
\log\frac{q_\theta(z\mid c)}{r_\psi(z\mid c)}
\right].
\]
By definition, this is
\[
\mathbb{E}_{p(c)}
\mathrm{KL}\!\bigl(q_\theta(z\mid c)\,\|\,r_\psi(z\mid c)\bigr).
\]

Combining the two terms gives
\[
\mathbb{E}_{p(x,c)}
\mathrm{KL}\!\bigl(q_\theta(z\mid x)\,\|\,r_\psi(z\mid c)\bigr)
=
I_\theta(Z;X\mid C)
+
\mathbb{E}_{p(c)}
\mathrm{KL}\!\bigl(q_\theta(z\mid c)\,\|\,r_\psi(z\mid c)\bigr).
\]
Finally, since KL divergence is nonnegative,
\[
\mathbb{E}_{p(c)}
\mathrm{KL}\!\bigl(q_\theta(z\mid c)\,\|\,r_\psi(z\mid c)\bigr)\ge 0,
\]
which implies
\[
I_\theta(Z;X\mid C)
\le
\mathbb{E}_{p(x,c)}
\mathrm{KL}\!\bigl(q_\theta(z\mid x)\,\|\,r_\psi(z\mid c)\bigr).
\]
\end{proof}

\paragraph{Interpretation.}
This is the key closure property of the method:
\[
\text{conditional KL}
=
\text{within-condition information}
+
\text{prior mismatch}.
\]
Thus the conditional KL term is not merely heuristic; it is an exact decomposition plus a tractable upper bound on \(I(Z;X\mid C)\).

\subsection{Optional nuisance interpretation}
\label{app:nuisance}

We now add an optional structural interpretation.
Assume there exist latent variables \(S\) and \(N\) such that
\[
X=T(S,N),\qquad Y=f(S),\qquad C=\phi(Y).
\]
Here \(S\) denotes task-relevant signal and \(N\) denotes within-condition nuisance variation.

\begin{proposition}
\label{prop:nuisance}
Suppose that conditioned on \(C\), the variables satisfy the Markov chain
\[
N \to X \to Z.
\]
Then
\[
I(Z;N\mid C)\le I(Z;X\mid C).
\]
Combining this with Theorem~\ref{thm:condkl} yields
\[
I(Z;N\mid C)
\le
I(Z;X\mid C)
\le
\mathbb{E}_{p(x,c)}
\mathrm{KL}\!\bigl(q_\theta(z\mid x)\,\|\,r_\psi(z\mid c)\bigr).
\]
\end{proposition}

\begin{proof}
Fix any value \(c\) of \(C\).
Under the conditional Markov chain
\[
N \to X \to Z \quad \text{given } C=c,
\]
the conditional data processing inequality gives
\[
I(Z;N\mid C=c)\le I(Z;X\mid C=c).
\]
Taking expectation over \(c\) yields
\[
I(Z;N\mid C)\le I(Z;X\mid C).
\]
The second inequality follows directly from Theorem~\ref{thm:condkl}.
\end{proof}

\paragraph{Interpretation.}
This proposition does \emph{not} claim that \(I(Z;X\mid C)\) is identical to nuisance information
in complete generality. The correct and sufficient statement is weaker:
under the structural assumption above, the conditional KL term upper-bounds nuisance dependence
through \(I(Z;X\mid C)\).

\subsection{Exact decomposition of the standard KL term}
\label{app:stdkl}

The standard VIB term also admits an exact decomposition, completely parallel to the conditional case.

\begin{lemma}
\label{lem:stdkl}
Let
\[
q_\theta(z)=\int q_\theta(z\mid x)p(x)\,dx
\]
be the aggregated posterior. Then for any unconditional prior \(r_0(z)\),
\[
\mathbb{E}_{p(x)}
\mathrm{KL}\!\bigl(q_\theta(z\mid x)\,\|\,r_0(z)\bigr)
=
I_\theta(Z;X)
+
\mathrm{KL}\!\bigl(q_\theta(z)\,\|\,r_0(z)\bigr).
\]
Consequently,
\[
I_\theta(Z;X)\le
\mathbb{E}_{p(x)}
\mathrm{KL}\!\bigl(q_\theta(z\mid x)\,\|\,r_0(z)\bigr).
\]
\end{lemma}

\begin{proof}
We expand:
\[
\mathbb{E}_{p(x)}
\mathrm{KL}\!\bigl(q_\theta(z\mid x)\,\|\,r_0(z)\bigr)
=
\mathbb{E}_{p(x)}\mathbb{E}_{q_\theta(z\mid x)}
\left[
\log\frac{q_\theta(z\mid x)}{r_0(z)}
\right].
\]
Add and subtract \(\log q_\theta(z)\):
\[
\log\frac{q_\theta(z\mid x)}{r_0(z)}
=
\log\frac{q_\theta(z\mid x)}{q_\theta(z)}
+
\log\frac{q_\theta(z)}{r_0(z)}.
\]
Therefore,
\begin{align*}
&\mathbb{E}_{p(x)}
\mathrm{KL}\!\bigl(q_\theta(z\mid x)\,\|\,r_0(z)\bigr) \\
&=
\mathbb{E}_{p(x)}\mathbb{E}_{q_\theta(z\mid x)}
\left[
\log\frac{q_\theta(z\mid x)}{q_\theta(z)}
\right]
+
\mathbb{E}_{p(x)}\mathbb{E}_{q_\theta(z\mid x)}
\left[
\log\frac{q_\theta(z)}{r_0(z)}
\right].
\end{align*}
The first term is
\[
I_\theta(Z;X).
\]
The second term becomes, after integrating out \(x\),
\[
\mathbb{E}_{q_\theta(z)}
\left[
\log\frac{q_\theta(z)}{r_0(z)}
\right]
=
\mathrm{KL}\!\bigl(q_\theta(z)\,\|\,r_0(z)\bigr).
\]
Combining the two terms proves the identity. The inequality follows from nonnegativity of KL.
\end{proof}

\paragraph{Interpretation.}
The standard KL term controls global bandwidth:
\[
\text{standard KL}
=
\text{global information}
+
\text{global prior mismatch}.
\]
This makes its role fully parallel to the conditional KL term.

\subsection{A structured conditional prior and its induced geometry}
\label{app:simplex}

We choose
\[
r_\psi(z\mid c)=\mathcal{N}(\mu_c,\tau^2 I),
\qquad
\mu_c=\rho R s_c+b,
\]
where:
\begin{itemize}
    \item \(s_1,\dots,s_C\in\mathbb{R}^{C-1}\) are regular simplex vertices,
    \item \(R^\top R=I\),
    \item \(\rho>0\), \(b\in\mathbb{R}^d\), and \(\tau>0\).
\end{itemize}
For a regular simplex,
\[
\|s_c\|=1,
\qquad
s_c^\top s_{c'}=-\frac{1}{C-1}\quad (c\neq c').
\]

\begin{proposition}
\label{prop:simplex}
For any \(c\neq c'\),
\[
\|\mu_c-\mu_{c'}\|^2
=
2\rho^2\frac{C}{C-1}.
\]
Hence all condition centers are equally separated.
\end{proposition}

\begin{proof}
By definition,
\[
\mu_c-\mu_{c'}=\rho R(s_c-s_{c'}).
\]
Therefore,
\[
\|\mu_c-\mu_{c'}\|^2
=
\|\rho R(s_c-s_{c'})\|^2
=
\rho^2\|R(s_c-s_{c'})\|^2.
\]
Since \(R^\top R=I\), the linear map \(R\) preserves Euclidean norm on the simplex subspace, so
\[
\|R(s_c-s_{c'})\|^2=\|s_c-s_{c'}\|^2.
\]
Hence
\[
\|\mu_c-\mu_{c'}\|^2=\rho^2\|s_c-s_{c'}\|^2.
\]

Now expand the squared distance:
\[
\|s_c-s_{c'}\|^2
=
\|s_c\|^2+\|s_{c'}\|^2-2s_c^\top s_{c'}.
\]
Using the simplex identities,
\[
\|s_c-s_{c'}\|^2
=
1+1-2\left(-\frac{1}{C-1}\right)
=
2+\frac{2}{C-1}
=
2\frac{C}{C-1}.
\]
Therefore,
\[
\|\mu_c-\mu_{c'}\|^2
=
\rho^2\cdot 2\frac{C}{C-1}
=
2\rho^2\frac{C}{C-1}.
\]
\end{proof}

\paragraph{Interpretation.}
This proposition establishes a precise and modest claim:
the simplex prior injects a symmetric separability bias.
It does not claim universal optimality of the simplex construction.

\subsection{A corrected minimization argument for the shared variance}
\label{app:tau}

\begin{proposition}
\label{prop:tau}
Suppose the encoder family \(q_\theta(z\mid x)=\mathcal{N}(m_x,\Sigma_x)\), the centers \(\{\mu_c\}\),
and the data distribution \(p(x,c)\) are temporarily fixed.
Define
\[
A
:=
\mathbb{E}_{p(x,c)}
\bigl[
\mathrm{tr}(\Sigma_x)+\|m_x-\mu_c\|^2
\bigr].
\]
Let
\[
J(\tau^2)
:=
\mathbb{E}_{p(x,c)}
\mathrm{KL}\!\left(
\mathcal{N}(m_x,\Sigma_x)\,\|\,\mathcal{N}(\mu_c,\tau^2 I)
\right).
\]
Then
\[
J(\tau^2)
=
\frac12\left(\frac{A}{\tau^2}+d\log\tau^2\right)+\mathrm{const},
\]
where \(\mathrm{const}\) does not depend on \(\tau\).
Moreover, the unique minimizer is
\[
\tau_\star^2=\frac{A}{d}.
\]
\end{proposition}

\begin{proof}
For Gaussian distributions,
\[
\mathrm{KL}\!\left(
\mathcal{N}(m_x,\Sigma_x)\,\|\,\mathcal{N}(\mu_c,\tau^2 I)
\right)
=
\frac12\Bigl[
\tau^{-2}\mathrm{tr}(\Sigma_x)
+\tau^{-2}\|m_x-\mu_c\|^2
-d
+d\log\tau^2
-\log\det\Sigma_x
\Bigr].
\]
Taking expectation over \(p(x,c)\), we obtain
\[
J(\tau^2)
=
\frac12\left(
\frac{A}{\tau^2}
+
d\log\tau^2
\right)
+
\frac12\,\mathbb{E}_{p(x,c)}[-d-\log\det\Sigma_x].
\]
The second term is independent of \(\tau\), so we may write
\[
J(\tau^2)
=
\frac12\left(
\frac{A}{\tau^2}
+
d\log\tau^2
\right)
+\mathrm{const}.
\]

Now let \(u=\tau^2>0\). Then
\[
J(u)=\frac12\left(\frac{A}{u}+d\log u\right)+\mathrm{const}.
\]
Differentiate with respect to \(u\):
\[
J'(u)
=
\frac12\left(-\frac{A}{u^2}+\frac{d}{u}\right)
=
\frac{du-A}{2u^2}.
\]
Since \(2u^2>0\) for all \(u>0\), the sign of \(J'(u)\) is determined entirely by \(du-A\):
\[
J'(u)<0 \quad \text{if } u<\frac{A}{d},
\qquad
J'(u)>0 \quad \text{if } u>\frac{A}{d}.
\]
Therefore \(u=A/d\) is the unique stationary point, and the derivative changes from negative
to positive at that point. Hence it is the unique global minimizer:
\[
u_\star=\frac{A}{d}.
\]
Returning to \(u=\tau^2\), we obtain
\[
\tau_\star^2=\frac{A}{d}.
\]
\end{proof}

\paragraph{Interpretation.}
This proposition is intentionally weak.
It does not claim that fixing \(\tau\) is universally optimal.
Its narrower meaning is:
if one wants to keep a shared isotropic scale across conditions, then learning \(\tau\)
only adjusts a single global average dispersion level.
Therefore, fixing \(\tau\) is a coherent modeling choice when one wishes to avoid
coupling center geometry with variance adaptation.

\subsection{How the proofs justify the final dual-bottleneck objective}
\label{app:final-logic}

The final objective is
\[
L(\theta,\omega,\psi)
=
\mathbb{E}_{p(x,y)}\mathbb{E}_{q_\theta(z\mid x)}
\bigl[\ell_{\mathrm{task}}(g_\omega(z),y)\bigr]
+
\beta_{\mathrm{std}}
\mathbb{E}_{p(x)}
\mathrm{KL}\!\bigl(q_\theta(z\mid x)\,\|\,r_0(z)\bigr)
+
\beta_{\mathrm{cond}}
\mathbb{E}_{p(x,c)}
\mathrm{KL}\!\bigl(q_\theta(z\mid x)\,\|\,r_\psi(z\mid c)\bigr),
\]
with \(c=\phi(y)\).

The mathematical role of each term is now explicit:
\begin{itemize}
    \item the task loss preserves supervised prediction;
    \item by Lemma~\ref{lem:stdkl}, the standard KL controls global information bandwidth;
    \item by Theorem~\ref{thm:condkl}, the conditional KL controls within-condition information up to a prior mismatch term;
    \item by Theorem~\ref{thm:insufficient}, both terms are needed if the design goal is to separately preserve label-induced structure and suppress within-condition variation.
\end{itemize}

Hence the dual-bottleneck construction is not an ad hoc addition of two penalties:
it is a mathematically targeted decomposition of the regularization objective.

\subsection{Scope and non-claims}
\label{app:scope}

For clarity, we explicitly state what the present theory does and does not establish.

\paragraph{Established in this note.}
\begin{itemize}
    \item Exact decomposition of \(I(Z;X)\) under a label-induced partition.
    \item Existence-level insufficiency of a single global bottleneck for the stated design goal.
    \item Exact decomposition of the conditional KL term into within-condition mutual information and prior mismatch.
    \item A symmetric structured conditional prior with explicit center separation.
    \item An objective-level justification of the dual-bottleneck training criterion.
\end{itemize}

\paragraph{Not established in this note.}
\begin{itemize}
    \item Universal superiority of the conditional KL term over all possible condition-aware regularizers.
    \item Universal task-level gains in accuracy or generalization.
    \item Optimization consistency or estimator optimality in minibatch finite-sample training.
    \item Global optimality of the simplex prior construction for all tasks.
\end{itemize}

Stating these non-claims explicitly improves the credibility of the method by aligning the theorem
strength with the actual proof strength.

%%%%%%%%%%%%%%%%%%%%%%%%%%%%%%%%%%%%%%%%%%%%%%%%%%%%%%%%%%%%

\newpage
\section*{NeurIPS Paper Checklist}

%%% BEGIN INSTRUCTIONS %%%
The checklist is designed to encourage best practices for responsible machine learning research, addressing issues of reproducibility, transparency, research ethics, and societal impact. Do not remove the checklist: {\bf The papers not including the checklist will be desk rejected.} The checklist should follow the references and follow the (optional) supplemental material.  The checklist does NOT count towards the page
limit. 

Please read the checklist guidelines carefully for information on how to answer these questions. For each question in the checklist:
\begin{itemize}
    \item You should answer \answerYes{}, \answerNo{}, or \answerNA{}.
    \item \answerNA{} means either that the question is Not Applicable for that particular paper or the relevant information is Not Available.
    \item Please provide a short (1–2 sentence) justification right after your answer (even for NA). 
   % \item {\bf The papers not including the checklist will be desk rejected.}
\end{itemize}

{\bf The checklist answers are an integral part of your paper submission.} They are visible to the reviewers, area chairs, senior area chairs, and ethics reviewers. You will be asked to also include it (after eventual revisions) with the final version of your paper, and its final version will be published with the paper.

The reviewers of your paper will be asked to use the checklist as one of the factors in their evaluation. While "\answerYes{}" is generally preferable to "\answerNo{}", it is perfectly acceptable to answer "\answerNo{}" provided a proper justification is given (e.g., "error bars are not reported because it would be too computationally expensive" or "we were unable to find the license for the dataset we used"). In general, answering "\answerNo{}" or "\answerNA{}" is not grounds for rejection. While the questions are phrased in a binary way, we acknowledge that the true answer is often more nuanced, so please just use your best judgment and write a justification to elaborate. All supporting evidence can appear either in the main paper or the supplemental material, provided in appendix. If you answer \answerYes{} to a question, in the justification please point to the section(s) where related material for the question can be found.

IMPORTANT, please:
\begin{itemize}
    \item {\bf Delete this instruction block, but keep the section heading ``NeurIPS Paper Checklist"},
    \item  {\bf Keep the checklist subsection headings, questions/answers and guidelines below.}
    \item {\bf Do not modify the questions and only use the provided macros for your answers}.
\end{itemize}

%%% END INSTRUCTIONS %%%

\begin{enumerate}

\item {\bf Claims}
    \item[] Question: Do the main claims made in the abstract and introduction accurately reflect the paper's contributions and scope?
    \item[] Answer: \answerTODO{} % Replace by \answerYes{}, \answerNo{}, or \answerNA{}.
    \item[] Justification: \justificationTODO{}

\item {\bf Limitations}
    \item[] Question: Does the paper discuss the limitations of the work performed by the authors?
    \item[] Answer: \answerTODO{} % Replace by \answerYes{}, \answerNo{}, or \answerNA{}.
    \item[] Justification: \justificationTODO{}

\item {\bf Theory assumptions and proofs}
    \item[] Question: For each theoretical result, does the paper provide the full set of assumptions and a complete (and correct) proof?
    \item[] Answer: \answerTODO{} % Replace by \answerYes{}, \answerNo{}, or \answerNA{}.
    \item[] Justification: \justificationTODO{}

    \item {\bf Experimental result reproducibility}
    \item[] Question: Does the paper fully disclose all the information needed to reproduce the main experimental results of the paper to the extent that it affects the main claims and/or conclusions of the paper (regardless of whether the code and data are provided or not)?
    \item[] Answer: \answerTODO{} % Replace by \answerYes{}, \answerNo{}, or \answerNA{}.
    \item[] Justification: \justificationTODO{}

\item {\bf Open access to data and code}
    \item[] Question: Does the paper provide open access to the data and code, with sufficient instructions to faithfully reproduce the main experimental results, as described in supplemental material?
    \item[] Answer: \answerTODO{} % Replace by \answerYes{}, \answerNo{}, or \answerNA{}.
    \item[] Justification: \justificationTODO{}

\item {\bf Experimental setting/details}
    \item[] Question: Does the paper specify all the training and test details (e.g., data splits, hyperparameters, how they were chosen, type of optimizer, etc.) necessary to understand the results?
    \item[] Answer: \answerTODO{} % Replace by \answerYes{}, \answerNo{}, or \answerNA{}.
    \item[] Justification: \justificationTODO{}

\item {\bf Experiment statistical significance}
    \item[] Question: Does the paper report error bars suitably and correctly defined or other appropriate information about the statistical significance of the experiments?
    \item[] Answer: \answerTODO{} % Replace by \answerYes{}, \answerNo{}, or \answerNA{}.
    \item[] Justification: \justificationTODO{}

\item {\bf Experiments compute resources}
    \item[] Question: For each experiment, does the paper provide sufficient information on the computer resources (type of compute workers, memory, time of execution) needed to reproduce the experiments?
    \item[] Answer: \answerTODO{} % Replace by \answerYes{}, \answerNo{}, or \answerNA{}.
    \item[] Justification: \justificationTODO{}
    
\item {\bf Code of ethics}
    \item[] Question: Does the research conducted in the paper conform, in every respect, with the NeurIPS Code of Ethics \url{https://neurips.cc/public/EthicsGuidelines}?
    \item[] Answer: \answerTODO{} % Replace by \answerYes{}, \answerNo{}, or \answerNA{}.
    \item[] Justification: \justificationTODO{}

\item {\bf Broader impacts}
    \item[] Question: Does the paper discuss both potential positive societal impacts and negative societal impacts of the work performed?
    \item[] Answer: \answerTODO{} % Replace by \answerYes{}, \answerNo{}, or \answerNA{}.
    \item[] Justification: \justificationTODO{}
    
\item {\bf Safeguards}
    \item[] Question: Does the paper describe safeguards that have been put in place for responsible release of data or models that have a high risk for misuse (e.g., pretrained language models, image generators, or scraped datasets)?
    \item[] Answer: \answerTODO{} % Replace by \answerYes{}, \answerNo{}, or \answerNA{}.
    \item[] Justification: \justificationTODO{}

\item {\bf Licenses for existing assets}
    \item[] Question: Are the creators or original owners of assets (e.g., code, data, models), used in the paper, properly credited and are the license and terms of use explicitly mentioned and properly respected?
    \item[] Answer: \answerTODO{} % Replace by \answerYes{}, \answerNo{}, or \answerNA{}.
    \item[] Justification: \justificationTODO{}

\item {\bf New assets}
    \item[] Question: Are new assets introduced in the paper well documented and is the documentation provided alongside the assets?
    \item[] Answer: \answerTODO{} % Replace by \answerYes{}, \answerNo{}, or \answerNA{}.
    \item[] Justification: \justificationTODO{}

\item {\bf Crowdsourcing and research with human subjects}
    \item[] Question: For crowdsourcing experiments and research with human subjects, does the paper include the full text of instructions given to participants and screenshots, if applicable, as well as details about compensation (if any)? 
    \item[] Answer: \answerTODO{} % Replace by \answerYes{}, \answerNo{}, or \answerNA{}.
    \item[] Justification: \justificationTODO{}

\item {\bf Institutional review board (IRB) approvals or equivalent for research with human subjects}
    \item[] Question: Does the paper describe potential risks incurred by study participants, whether such risks were disclosed to the subjects, and whether Institutional Review Board (IRB) approvals (or an equivalent approval/review based on the requirements of your country or institution) were obtained?
    \item[] Answer: \answerTODO{} % Replace by \answerYes{}, \answerNo{}, or \answerNA{}.
    \item[] Justification: \justificationTODO{}

\item {\bf Declaration of LLM usage}
    \item[] Question: Does the paper describe the usage of LLMs if it is an important, original, or non-standard component of the core methods in this research? Note that if the LLM is used only for writing, editing, or formatting purposes and does not impact the core methodology, scientific rigorousness, or originality of the research, declaration is not required.
    %this research? 
    \item[] Answer: \answerTODO{} % Replace by \answerYes{}, \answerNo{}, or \answerNA{}.
    \item[] Justification: \justificationTODO{}

\end{enumerate}